\PassOptionsToPackage{table}{xcolor}

\documentclass[nohyperref]{article}

\usepackage[accepted]{icml2023}

\usepackage{scalefnt,letltxmacro}
\LetLtxMacro{\oldtextsc}{\textsc}
\renewcommand{\textsc}[1]{\oldtextsc{\scalefont{1.10}#1}}
\usepackage[acronym,smallcaps,nowarn]{glossaries}[=v4.46]
\glsdisablehyper
\makeglossaries
\usepackage{xspace}
\usepackage{float}
\usepackage{physics}
\usepackage{lipsum}

\usepackage{enumitem}
\usepackage{amssymb}
\usepackage{mathtools}
\usepackage{amsfonts}
\usepackage{amsmath}
\usepackage{amsthm}
\usepackage{booktabs}
\usepackage[]{microtype}

\usepackage{pifont}
\newcommand{\cmark}{\textcolor{green!60!black}{\ding{51}}\xspace}
\newcommand{\xmark}{\textcolor{red!60!black}{\ding{55}}\xspace}

\usepackage[vlined,linesnumbered,ruled]{algorithm2e}

\SetCommentSty{mycommfont}
\SetKwInput{KwInput}{Input} 
\SetKwInput{KwOutput}{Output}

\usepackage{xcolor}
\definecolor{mylightgray}{gray}{0.94}

\definecolor{color_vae}{HTML}{91058A}
\definecolor{color_cvae}{HTML}{DA70D6}
\definecolor{color_gppvae}{HTML}{32CD32}
\definecolor{color_gpvae}{HTML}{039796}
\definecolor{color_gp_vae}{HTML}{00FFFF}
\definecolor{color_sgp_vae}{HTML}{99CB32}
\definecolor{color_svgp_vae}{HTML}{19198c}
\definecolor{color_bae}{HTML}{AB6345}
\definecolor{color_gpbae}{HTML}{FF9B38}
\definecolor{color_bsgpae}{HTML}{DF2B4F}
\definecolor{color_blue_trajectory}{HTML}{DF2B4F}
\definecolor{color_blue_trajectory}{HTML}{3E9EFF}
\definecolor{color_orange_trajectory}{HTML}{DF5E33}

\usepackage[colorlinks,linktoc=all]{hyperref}
\usepackage[all]{hypcap}
\definecolor{frenchblue}{rgb}{0.01171875, 0.0078125, 0.4375}
\hypersetup{citecolor=frenchblue}
\hypersetup{linkcolor=frenchblue}
\hypersetup{urlcolor =frenchblue}

\usepackage[capitalize,nameinlink]{cleveref}
\crefname{section}{\S}{\S\S}
\Crefname{section}{\S}{\S\S}
\creflabelformat{equation}{#2\textup{#1}#3}

\makeatletter
\@ifundefined{c@rownum}{%
  \let\c@rownum\rownum
}{}
\@ifundefined{therownum}{%
  \def\therownum{\@arabic\rownum}%
}{}
\makeatother

\makeatletter
\newcommand*{\addFileDependency}[1]{%
	\typeout{(#1)}
	\@addtofilelist{#1}
	\IfFileExists{#1}{}{\typeout{No file #1.}}
}
\makeatother

\usepackage[font=small,labelfont=bf,tableposition=top]{caption}
\usepackage{tikz}
\usepackage{graphicx}
\usepackage[]{subcaption}   %

\usepackage{pgfplots}
\pgfplotsset{compat=1.6}
\usepgfplotslibrary{groupplots}
\tikzstyle{every picture}+=[font=\sffamily]
\tikzstyle{optimized} = [circle,fill=white,draw=black, dashed,inner sep=1pt, minimum size=20pt, font=\fontsize{10}{10}\selectfont, node distance=1]
\pgfkeys{/pgf/number format/.cd,1000 sep={}}
\pgfplotsset{
	tick label style = {font=\sffamily},
	every axis label/.append style={font=\sffamily},
	typeset ticklabels with strut,
}
\pgfplotsset{every axis/.append style={
			every x tick label/.append style={font=\fontsize{6pt}{6pt}\sffamily, yshift=.5ex,},
			every y tick label/.append style={font=\fontsize{6pt}{6pt}\sffamily, xshift=.5ex},
			every y label/.append style={xshift=10ex, font=\sffamily},
			every x label/.append style={yshift=3ex, font=\sffamily},
			every title/.append style={font=\sffamily}
		},
}
\pgfplotsset{
  xticklabel={$\mathsf{\pgfmathprintnumber{\tick}}$},
  yticklabel={$\mathsf{\pgfmathprintnumber{\tick}}$},
}
\pgfplotsset{every axis title/.append style={yshift=-1ex}}
\newlength\figureheight
\newlength\figurewidth
\usepgfplotslibrary{external}
\usepackage[colorinlistoftodos,
	textsize=scriptsize,
	linecolor=red!30,
	bordercolor=red!30,
	backgroundcolor=red!10]{todonotes}
\renewcommand{\todo}[2][]{\tikzexternaldisable\@todo[#1]{#2}\tikzexternalenable}

\usepackage[most]{tcolorbox}
{\begin{tcolorbox}[toprule=2mm,left=4pt,right=4pt,top=0pt,bottom=1pt,boxsep=2pt, before skip = 2ex, after skip =2ex]}{\end{tcolorbox}}
\tcbset{boxsep=4pt,left=2pt,right=2pt,top=-0pt,bottom=0pt}

\usepackage{url}

\usepackage{xr}
\usepackage{subcaption}

\usepackage{tikz}
\usepackage{graphicx}
\usepackage{xcolor}
\usepackage{amsmath}
\usepackage{amssymb}
\usepackage{bm}

\usepackage{adjustbox}
\usepackage{multirow}
\usepackage{mathtools}

\usepackage{xspace}

\newacronym{MAP}{map}{maximum-a-posteriori}
\newacronym{MLE}{mle}{maximum likelihood estimation}
\newacronym{MNLL}{mnll}{mean negative loglikelihood}
\newacronym{NLL}{nll}{negative loglikelihood}
\newacronym{LL}{ll}{log-likelihood}
\newacronym{RMSE}{rmse}{root mean squared error}
\newacronym{ECE}{ece}{expected calibration error}
\newacronym{FID}{fid}{Fr\'echet Inception Distance}
\newacronym{MSE}{mse}{mean squared error}

\newacronym{PCA}{pca}{principal component analysis}
\newacronym{AE}{ae}{autoencoder}
\newacronym{WAE}{wae}{Wasserstein Autoencoder}
\newacronym{VAE}{vae}{Variational Autoencoder}
\newacronym{BAE}{bae}{Bayesian autoencoder}
\newacronym{CDF}{cdf}{cumulative density function}
\newacronym{GAN}{gan}{Generative Adversarial Network}

\newacronym{MC}{mc}{Monte Carlo}
\newacronym{MCMC}{mcmc}{Markov chain Monte Carlo}
\newacronym{HMC}{hmc}{Hamiltonian Monte Carlo}
\newacronym{MH}{mh}{Metropolis-Hastings}
\newacronym{NUTS}{nuts}{no-u-turn sampler}
\newacronym{SGHMC}{sghmc}{stochastic gradient Hamiltonian Monte Carlo}

\newacronym{DGP}{dgp}{deep Gaussian process} %
\newacronym{GPLVM}{gplvm}{Gaussian process latent variable model}
\newacronym{DPMM}{dpmm}{Dirichlet Process Mixture Model}

\newacronym{VFE}{vfe}{variational free energy}

\newacronym[firstplural=Gaussian Processes]{GP}{gp}{Gaussian Process}

\newacronym{VI}{vi}{variational inference}

\newacronym{ELBO}{elbo}{evidence lower bound}
\newacronym{NELBO}{nelbo}{negative evidence lower bound}
\newacronym{ELL}{ell}{expected log likelihood}
\newacronym{KL}{kl}{Kullback-Leibler divergence}
\newacronym{AUC}{auc}{area under the curve}

\newacronym[firstplural=Bayesian neural networks]{BNN}{bnn}{Bayesian neural network}
\newacronym[firstplural=deep neural networks]{DNN}{dnn}{deep neural network}
\newacronym[]{CNN}{cnn}{convolutional neural network}
\newacronym{MLP}{mlp}{multilayer perceptron}
\newacronym{NN}{nn}{neural network}
\newacronym{RELU}{ReLU}{rectified linear unit}

\newacronym{NF}{nf}{normalizing flow}

\newacronym{RBF}{rbf}{radial basis function}
\newacronym{ARD}{ard}{automatic relevance determination}

\newacronym{RKHS}{rkhs}{reproducing kernel Hilbert space}

\newacronym{OT}{ot}{optimal transport}
\newacronym{WD}{wd}{Wasserstein distance}
\newacronym{SWD}{swd}{sliced-Wasserstein distance}
\newacronym{DSWD}{dswd}{distributional sliced-Wasserstein distance}
\newacronym{BSGPAE}{bsgpae}{Bayesian Sparse Gaussian Process Autoencoder}
\newacronym{GPBAE}{{gp}-{bae}}{Gaussian Process Bayesian Autoencoder}
\newacronym{CVAE}{cvae}{Conditional Variational Autoencoder}
\newacronym{SGPBAE}{{sgp}-{bae}}{Sparse Gaussian Process Bayesian Autoencoder}

\newcommand{\adrflvm}{\textsl{advised}\textsc{rflvm}\xspace}



\DeclarePairedDelimiterX{\infdivx}[2]{[}{]}{%
  #1\;\delimsize\|\;#2%
}

\newcommand{\cN}{\mathcal{N}}

\def\x{{\mathbf x}}
\def\y{{\mathbf y}}

\def\u{{\mathbf u}}

\def\w{{\mathbf w}}

\def\vx{{\mathbf X}}
\def\va{{\mathbf A}}
\def\vy{{\mathbf Y}}
\def\vw{{\mathbf W}}

\def\vk{{\mathbf K}}
\def\vu{{\mathbf U}}
\def\vl{{\mathbf L}}
\def\vv{{\mathbf V}}

\def\cN{{\cal N}}

\def\bzeta{{\bm{\zeta}}}
\def\btheta{{{\bm{\theta}}}}

\newcommand{\cred}[1]{{\color{red}#1}}
\newcommand{\cblue}[1]{{\color{blue}#1}}

\theoremstyle{plain}
\newtheorem{theorem}{Theorem}[section]
\newtheorem{proposition}[theorem]{Proposition}
\newtheorem{lemma}[theorem]{Lemma}
\newtheorem{corollary}[theorem]{Corollary}
\theoremstyle{definition}
\newtheorem{definition}[theorem]{Definition}

\theoremstyle{remark}
\newtheorem{remark}[theorem]{Remark}

\icmltitlerunning{{Preventing Model Collapse in Gaussian Process Latent Variable Models}\hfill\thepage}

\begin{document}

\twocolumn[
\icmltitle{Preventing Model Collapse in Gaussian Process Latent Variable Models}

\icmlsetsymbol{equal}{*}
\begin{icmlauthorlist}
\icmlauthor{Ying Li}{equal,hku}
\icmlauthor{Zhidi Lin}{equal,cuhksz}
\icmlauthor{Feng Yin}{cuhksz}
\icmlauthor{Michael Minyi Zhang}{hku}
\end{icmlauthorlist}

\icmlaffiliation{hku}{Department of Statistics \& Actuarial Science, The University of Hong Kong, Hong Kong, China}
\icmlaffiliation{cuhksz}{School of Science \& Engineering, The Chinese University of Hong Kong, Shenzhen,  China}

\icmlcorrespondingauthor{Feng Yin}{yinfeng@cuhk.edu.cn}
\icmlcorrespondingauthor{Michael Minyi Zhang}{mzhang18@hku.hk}

\icmlkeywords{Machine Learning, ICML}

\vskip 0.3in
]

\printAffiliationsAndNotice{\icmlEqualContribution}  %

\begin{abstract} 
    Gaussian process latent variable models (GPLVMs) are a versatile family of unsupervised learning models commonly used for dimensionality reduction. 
    However, common challenges in modeling data with GPLVMs include inadequate kernel flexibility and improper selection of the projection noise, leading to a type of model collapse characterized by vague latent representations that do not reflect the underlying data structure.
    This paper addresses these issues by, first, theoretically examining the impact of projection variance on model collapse through the lens of a linear GPLVM.  
    Second, we tackle model collapse due to inadequate kernel flexibility by integrating the spectral mixture (SM) kernel and a differentiable random Fourier feature (RFF) kernel approximation, which ensures computational scalability and efficiency through off-the-shelf automatic differentiation tools for learning the kernel hyperparameters, projection variance, and latent representations within the variational inference framework.
    The proposed GPLVM, named \adrflvm, is evaluated across diverse datasets and consistently outperforms various salient competing models, including state-of-the-art variational autoencoders (VAEs) and other GPLVM variants, in terms of informative latent representations and missing data imputation.
\end{abstract}\vspace{-.02in}

\addtocontents{toc}{\protect\setcounter{tocdepth}{0}}

\section{Introduction} 
\label{sec:introduction} 
A latent variable model (LVM) 
represents each observed datum $\y_i \!\in\! \mathbb{R}^{M}$ using a low-dimensional latent variable $\x_i \!\in\! \mathbb{R}^Q$, where $ Q\! \ll \! M$. 
As a classic tool in statistical analysis, LVMs unveil hidden structures within the data, providing valuable insights into intricate systems across various domains \cite{bishop:2006:PRML}, such as signal processing \cite{zarzoso2010latent} and economics \cite{aigner1984latent}. 

One of the critical aspects of LVM is the choice of mapping function from the latent variables to the observed variables. A series of early works assumed that the mapping is linear, as seen in factor analysis \cite{kim1978factor}, principal component analysis (PCA) \cite{pearson1901liii,tipping1999probabilistic}, and canonical correlation analysis (CCA) \cite{hotelling1936relations}, among others. However, the linearity assumption limits the capacity of these models to capture complex, nonlinear patterns in the data, rendering them incapable of providing an optimal latent representation for complex data sets. To tackle this issue, more advanced methods like the variational autoencoder (VAE) \cite{kingma2019introduction,kingma2013auto} utilizes neural networks, while the Gaussian process latent variable model (GPLVM) \cite{lawrence2005probabilistic,titsias2010bayesian} employs the 
Gaussian process (GP) \cite{williams2006gaussian}, as the nonlinear mapping modules in LVM, providing enhanced capacity in capturing nonlinear relationships. 

GPLVMs benefit from the incorporation of the GP, which offers enhanced interpretability through explicit 
uncertainty calibration and the interpretable kernel functions \cite{theodoridis2020machine,cheng2022rethinking}. Additionally, the implicit regularization imposed by the GP prior prevents GPLVMs from severe overfitting \cite{lotfi2022bayesian, wilson2020bayesian}. 
Consequently, GPLVMs often achieve superior performance in practice, even with small sample sizes. Due to these favorable and unique properties, GPLVM has been applied to various applications, 
such as intrusion detection \cite{abolhasanzadeh2015gaussian}, image recognition \cite{eleftheriadis2013shared,li2017shared}, human pose estimation \cite{ek2008gaussian}, and image-text retrieval \cite{song2015similarity}.

Despite the popularity of GPLVM and the recent efforts dedicated to enhancing its learning and inference capabilities \cite{titsias2010bayesian,gundersen2021latent,ramchandran2021latent,de2021learning,lalchand2022generalised,zhang2023bayesian}, the existing work still lacks an in-depth understanding of how to optimally learn a compact and informative latent representation using the GPLVM. This ambiguity hinders our ability to overcome ``model collapse'' (see Definition~\ref{definition_model_collapse}), which is characterized by learning vague latent representations with practical implementations. This paper elucidates the two key factors that lead to model collapse--the improper selection of model projection noise and inadequate kernel flexibility. To this end, we propose a new GPLVM that is immune to model collapse. Our contributions are: \vspace{-.1in} 
\begin{itemize}
    \item  We provide a theoretical investigation of the impact that projection variance has on encouraging model collapse through the lens of linear GPLVMs. Our empirical validation further demonstrates the relevance of these analyses to general GPLVMs. These findings collectively emphasize the importance of learning the model projection variance.
    \vspace{-.05in}
    \item  We propose a novel GPLVM that integrates a spectral mixture (SM) kernel \cite{wilson2013gaussian}, capable of approximating arbitrary stationary kernels, to overcome model collapse arising from inadequate kernel flexibility. To reduce computational complexity and avoid introducing additional parameters like those in inducing point-based sparse GP methods \cite{titsias2009variational, hensman2013gaussian}, we leverage a differentiable random Fourier feature (RFF) approximation for the SM kernel \cite{jung2022efficient, lopez2014randomized}. This deliberate introduction of differentiability in the RFF approximation allows us to readily use modern off-the-shelf automatic differentiation tools \cite{paszke2019pytorch} to efficiently and scalably learn the kernel hyperparameters, projection variance, and latent representations of the proposed GPLVM within a variational inference framework \cite{bishop:2006:PRML}.
    %
    %
    \vspace{-.05in}
    \item Our proposed GPLVM is subjected to rigorous evaluation across diverse datasets, consistently outperforming various models, including the state-of-the-art (SOTA) VAEs and some representative GPLVM variants. Specifically, it excels in learning compact and informative latent representations, addressing the issues of model collapse in existing GPLVMs. 
\end{itemize}

\section{Preliminaries}
\label{sec:background}

\textbf{Gaussian Process.}
The GP is a generalization of the Gaussian distribution defined across infinite index sets \cite{williams2006gaussian}, thereby enabling the specification of distribution over functions $f: \mathbb{R}^Q \!\mapsto\! \mathbb{R}$. A GP is fully characterized by its mean function $\mu({\x})$, frequently set as zero, and its covariance function, a.k.a. kernel function, $k({\x}, {\x}^\prime; \bm{\theta}_{gp})$, where $\bm{\theta}_{gp}$ is a set of hyperparameters that needs to be tuned for model selection. 
According to the definition of GP, the function values $\mathbf{f} \!=\! \{f(\x_i)\}_{i=1}^N$ at any finite set of points $\vx \!=\! \{\x_i\}_{i=1}^N$ follow a joint Gaussian distribution, i.e.,
\begin{equation}
\setlength{\abovedisplayskip}{4.5pt}
\setlength{\belowdisplayskip}{4.5pt}
   \mathbf{f} \mid \vx = \mathcal{N}(\mathbf{f} \mid \bm{0}, \mathbf{K}), 
\end{equation}
where $\mathbf{K}$ denotes the covariance matrix evaluated on the finite input $\vx$ with $[\mathbf{K}]_{i,j} \!=\! k({{\x}}_i, {\x}_j)$. 
Given the observed function values $\mathbf{f}$ at the input $\vx$, the GP prediction distribution,  $p(f(\bm{x}_*) \vert {\x}_*, \mathbf{f}, \vx)$, at any new input ${\x}_*$, is Gaussian, fully characterized by the posterior mean $\xi$ and the posterior variance $\Xi$.  Concretely, 
\begin{subequations}
\label{eq:GP_posterior}
\setlength{\abovedisplayskip}{4.5pt}
\setlength{\belowdisplayskip}{4.5pt}
    \begin{align}
        & \xi(\mathbf{x}_*) = \mathbf{K}_{\mathbf{x}_*, \mathbf{X}} \mathbf{K}^{-1} { \mathbf{f}}, 
        \label{eq:post_mean}\\
        & \Xi(\mathbf{x}_*) = k(\mathbf{x}_*, \mathbf{x}_*)  - \mathbf{K}_{\mathbf{x}_*, \mathbf{X}} \mathbf{K}^{-1} \mathbf{K}_{\mathbf{x}_*, \mathbf{X}}^\top,
        \label{eq:post_cov}
    \end{align}
\end{subequations}
where $\mathbf{K}_{\mathbf{x}_*, \mathbf{X}}$ is the cross covariance matrix evaluated on the new input ${\x}_*$ and the observed input ${\vx}$.

\textbf{Spectral Mixture Kernel.} The behavior of a GP-distributed function is generally defined by the choice of the kernel function. 
However, subjectively selecting an appropriate kernel for complex applications is considerably challenging. By resorting to the fact that, according to Bochner's theorem, any stationary kernel and its spectral density are Fourier duals, we know that one type of popular kernel learning methods is to approximate the spectral density of the underlying stationary kernel \cite{bochner1934theorem}. In the spectral mixture (SM) kernel \cite{wilson2013gaussian}, the underlying spectral density is approximated using a Gaussian mixture:
\begin{equation}
\label{eq:SM_density}
\setlength{\abovedisplayskip}{6pt}
\setlength{\belowdisplayskip}{4.5pt}
\begin{aligned}
    &\!\!\!  s_i(\w) \!=\! \frac{\mathcal{N}(\mathbf{w} \vert \bm{\mu}_i, \operatorname{diag}(\bm{\sigma}_i^2)) \!+\! \mathcal{N}(-\mathbf{w} \vert \bm{\mu}_i, \operatorname{diag}(\bm{\sigma}_i^2))}{2},\\
    &\!\!\! p_{\mathrm{sm}}(\mathbf{w})= \sum_{i=1}^m \alpha_i s_i(\w),  
\end{aligned}
\end{equation}
where $\alpha_i$ is the mixture weight, $\bm{\mu}_i \!\in\! \mathbb{R}^Q$ and $\bm{\sigma}_i^2\!\in\! \mathbb{R}^Q$ are the mean and variance of the $i$-th Gaussian density, $m$ is the number of mixture components. Taking the inverse Fourier transform, we readily get the SM kernel, $ k_{\mathrm{sm}}(\x, \x^{\prime})=$ 
\begin{equation}
\setlength{\abovedisplayskip}{4.5pt}
\setlength{\belowdisplayskip}{4.5pt}
 \sum_{i=1}^m \alpha_i \exp \left(- 2 \pi^2 \|\boldsymbol{\sigma}_i^\top (\x-\x^{\prime}) \|^2 \right) \cos \left( 2\pi \boldsymbol \mu_i^{\top}\left(\x-\x^{\prime}\right)\right),
 \nonumber
\end{equation}
where $\bm{\theta}_{\mathrm{sm}} \!=\! \{\alpha_i, \bm{\mu}_i, \bm{\sigma^2_i}\}_{i=1}^m$ is the set of hyperparameters. Given that Gaussian mixture is dense, the SM kernel is guaranteed to be able to approximate any stationary kernel arbitrarily well  \cite{wilson2013gaussian}.

\textbf{Gaussian Process Latent Variable Models.} 
The GPLVM is a generative model where each observed datum $\y_i \!\in\! \mathbb{R}^{M}$ is generated through a noisy Gaussian process from a latent variable $\x_i \!\in\! \mathbb{R}^{Q}$ \citep{lawrence2005probabilistic}:
\begin{equation}
    \label{eq:GPLVM_generative}
    \setlength{\abovedisplayskip}{4.5pt}
    \setlength{\belowdisplayskip}{4.5pt}
\y_i = f(\x_i) + \bm{v}_i,  \ \ \bm{v}_i \sim \cN(\bm{0}, \sigma^2 \mathbf{I}_M),
\end{equation} 
where $f(\cdot)$ follows a zero-mean GP prior, and $\sigma^2$ is the projection variance, which can be interpreted as information lost in dimensionality reduction. A standard normal density is conventionally assigned as the prior to the latent variable, $\x_i \!\sim\! \cN(\bm{0}, \mathbf{I}_Q)$. 
In the case of having $N$ observations $\vy \!\in\! \mathbb{R}^{N \!\times\! M}$ from the GPLVM, the marginal likelihood after integrating out the latent GP, is expressed as:
\begin{equation}
    \label{eq:GPLVM_marginal_likelihood}
    \setlength{\abovedisplayskip}{4.5pt}
    \setlength{\belowdisplayskip}{4.5pt}
     p(\vy \mid \vx) =  \prod_{j=1}^M \cN(\y_{:,j} \mid \bm{0}, \ \mathbf{K} + \sigma^2 \mathbf{I}_N )  
\end{equation} 
where 
$\y_{:,j} \!\in\! \mathbb{R}^N$ denotes the $j$-th column of $\vy$.  
Consequently, the maximum likelihood estimate (MLE) of the latent variables $\vx$ can be obtained by solving the following optimization problem, 
\begin{equation}
\label{eq:GPLVM_MLE}
\setlength{\abovedisplayskip}{4.5pt}
\setlength{\belowdisplayskip}{4.5pt}
 \hat{\vx} =  \max_{\vx} \ L(\vx) =\max_{\vx} \ \log  p(\vy \mid \vx), 
\end{equation}
using e.g. gradient-based methods \citep{kingma2015adam}.  

In the context of GPLVM, the primary objective is to obtain a compact and informative latent representation of the observed data. Unlike the general definition of model collapse in machine learning models, which is primarily characterized by a gradual shift toward homogeneous output and increased deviations from accurate predictions \cite{bau2019seeing},  model collapse in GPLVM is closely tied to the effectiveness of latent variable inference, as outlined below:
\begin{definition}[\textbf{Model Collapse}]
\label{definition_model_collapse}
When the latent variables in GPLVMs become more homogeneous and/or their crucial feature details are sacrificed or distorted, we identify this phenomenon as model collapse.
\end{definition}  
Definition \ref{definition_model_collapse} posits that two distinct manifestations of model collapse can be identified: distortion and homogeneity. Distortion occurs when the latent manifold, representing the underlying data structure, is warped or twisted, failing to accurately describe the underlying data structures. Homogeneity, on the other hand, manifests as a reduction in diversity among latent variables, resulting in a loss of crucial data features.

\begin{figure}[t!]
    \centering
    \includegraphics[width=0.96\linewidth]{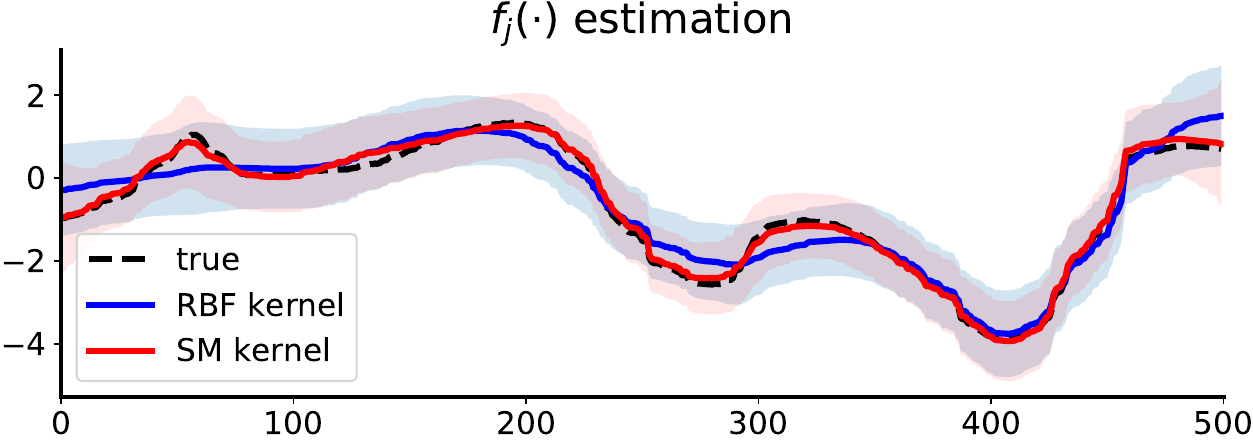} \vspace{.1in}
    
    \subfloat[Normal]{\includegraphics[width=0.32\linewidth]{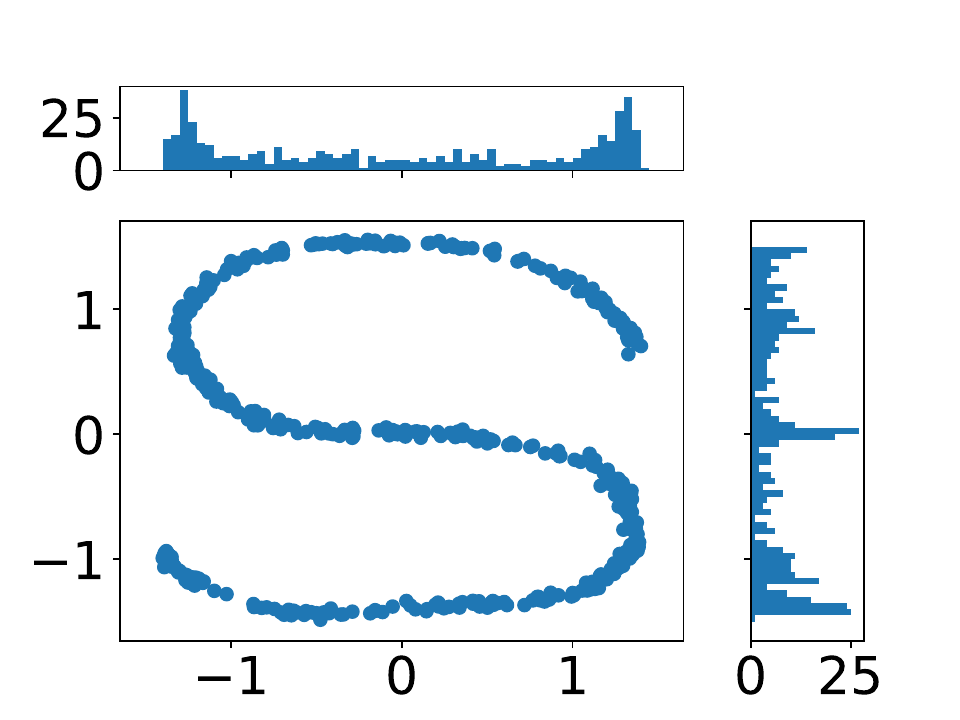}
    \label{subfig:correct_estimation}
    }
    \subfloat[Distortion]{\includegraphics[width=0.32\linewidth]{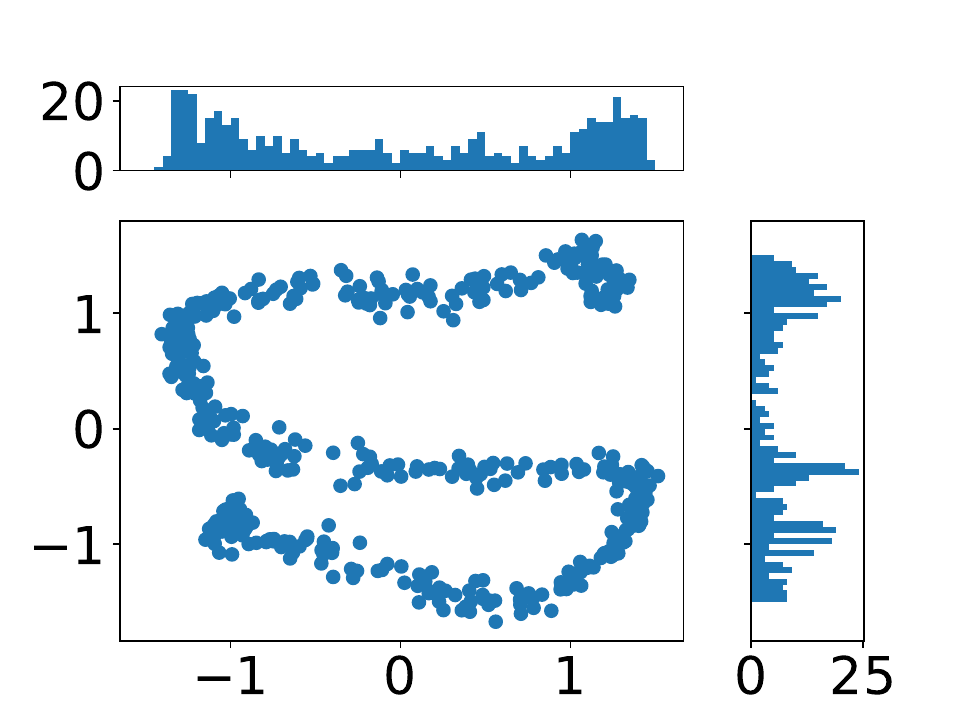}
    \label{subfig:distortion}
    } ~
    \subfloat[Homogeneity]{\includegraphics[width=0.32\linewidth]{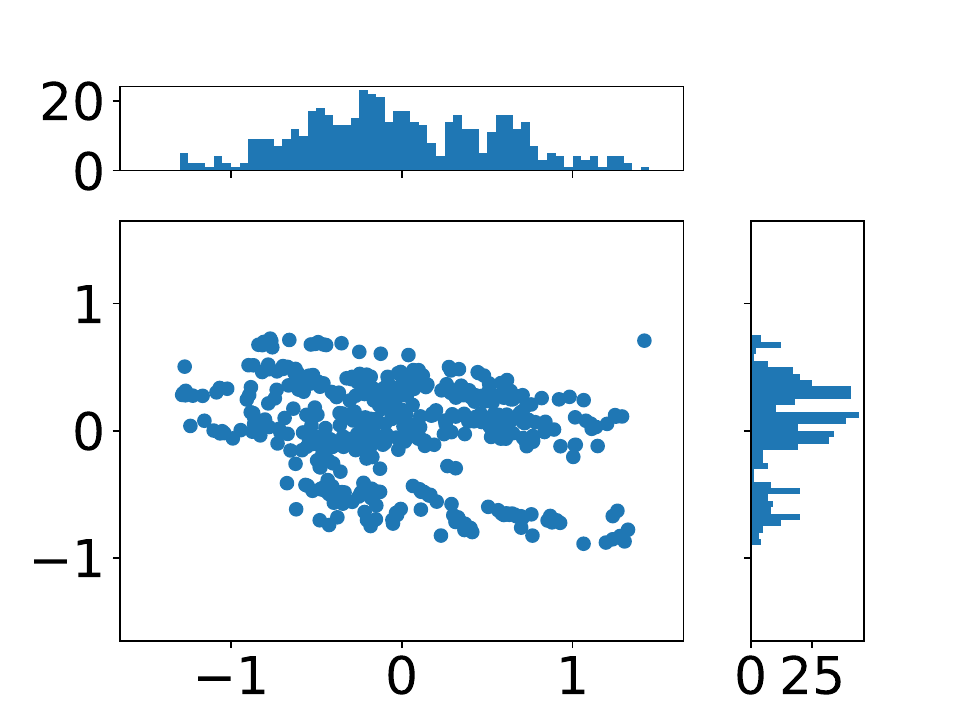}
    \label{subfig:collapse}
    } ~
    \vspace{-.07in}
    \caption{
    \textbf{Top}: Latent function estimation using GPLVM with preliminary (\cblue{\textbf{---}}) or advanced/flexible kernels (\cred{\textbf{---}}). \textbf{Bottom:} (\textbf{\subref{subfig:correct_estimation}}): 2-D S-shape latent manifold learned by the proposed \adrflvm.  (\textbf{\subref{subfig:distortion}}): 2-D S-shape latent manifold learned by using a preliminary (RBF) kernel. 
    (\textbf{\subref{subfig:collapse}}): 2-D S-shape latent manifold learned without optimizing projection variance. (\textbf{\subref{subfig:correct_estimation}})--(\textbf{\subref{subfig:collapse}}) also show histograms in different dimensions of the learned latent manifold.
    }
    \label{fig:different_kernel_illustration}
\end{figure}

\section{Causes of Model Collapse} 
In this section, we will elucidate that the distortion and homogeneity in the latent manifold are attributed to two crucial factors: improper selection of projection variance and inadequate kernel function flexibility. To further illustrate these concepts, Figs.~\ref{subfig:distortion} and \ref{subfig:collapse} depict examples where the learned latent manifolds are distorted and homogeneous, respectively.
\label{sec:SM-GPLVM}  

\subsection{Projection Variance Matters} 
\label{subsec:connection_dpcs} 
This subsection investigates the impact of projection variance on encouraging model collapse. To achieve this, we scrutinize the stationary points with respect to the latent variables $\vx$, and establish their connection to the projection variance. However, the computation of the stationary points is intractable due to the non-convex and nonlinear nature of GPLVMs in general. In light of this, we alternatively seek the lens of the linear GPLVM by assuming that the kernel function used in the GPLVM is the inner product kernel, i.e., $k(\x, \x^\prime) \!=\! \x^\top\x^\prime$. This simplified GPLVM is also known as the dual probabilistic principal component analysis (DPPCA) model \cite{lawrence2005probabilistic}. See more details in App.~\ref{app:relation_dpca_gplvm}. The main analyses are outlined below.
\begin{theorem}
    \label{theo:dpca_stationary_main_text}
    Given the maximization problem in Eq.~\eqref{eq:GPLVM_MLE}, the stationary points, $\hat{\vx}$, in the case of the linear GPLVM is:
        \begin{align}
            \label{eq:ppca_stationary_X}
            \setlength{\abovedisplayskip}{4pt}
            \setlength{\belowdisplayskip}{4pt}
            \hat{\vx} &= \vu_{Q} \left( \boldsymbol{\Lambda}_{Q} - {\sigma}^2 \mathbf{I}_{Q} \right)^{1/2} \mathbf{R}, 
        \end{align}
    where $\vu_{Q} \!\triangleq\! \left[ \u_1, \ldots, \u_Q \right] \!\in\! \mathbb{R}^{N \times Q} $ represents arbitrary eigenvectors of $\frac{1}{M} \vy \vy^{\top}$, $\mathbf{R} \in \mathbb{R}^{Q \times Q}$ is an arbitrary orthogonal matrix, and $\boldsymbol{\Lambda}_{Q} \!\in\! \mathbb{R}^{Q \times Q}$ is a diagonal matrix with:
    \begin{equation}
        \setlength{\abovedisplayskip}{4.5pt}
        \setlength{\belowdisplayskip}{4.5pt}
       \!\!\! [\boldsymbol{\Lambda}_{Q}]_{i, i} \!=\!\! \left\{
            \begin{aligned}
                & \lambda_i,        \operatorname {~the~corresponding~eigenvalue~to~} \mathbf{u}_i, \textbf{~or}\\
                & {\sigma}^2.
            \end{aligned}
        \right.
        \nonumber
    \end{equation}
\end{theorem}  
\vspace{-.2in}
\begin{proof}
 See App.~\ref{app:dpca_stationary_x} or App.~A in  \cite{lawrence2005probabilistic}.
 \vspace{-.08in}
\end{proof}

Theorem~\ref{theo:dpca_stationary_main_text} reveals that the stationary point, $\hat{\vx}$, depends on the projection variance $\sigma^2$ and eigenvalues of $\frac{1}{M}\vy\vy^\top$. However, it remains unclear which specific value of $\sigma^2$ may trigger the model collapse. Our findings, succinctly summarized in the following propositions, provide additional insight into the impact of the $\sigma^2$ on the type of the stationary point and the cause of the model collapse.




\begin{proposition}
    \label{theo:sigma}
    In the case that $\sigma^2$ equals to its MLE estimator, $\hat{\sigma}^2$: 
    \begin{equation}
        \setlength{\abovedisplayskip}{4pt}
        \setlength{\belowdisplayskip}{4pt}
        \sigma^2 = \hat{\sigma}^2 =  \frac{1}{N-Q^{\prime}} \sum_{j=Q^{\prime}+1}^{N} \lambda_j,
    \end{equation}
    where $Q^\prime$ is the number of eigenvalues retained in $\bm{\Lambda}_Q$ from $\frac{1}{M}\vy\vy^\top$, then the only stable maximum\footnote{In this case, the stationary points comprise only saddle points and the global optimum; no local optimum exists.} is the global optimum. 
\end{proposition}
\begin{proof}
See App. \ref{app:sigma}. 
\vspace{-.08in}
\end{proof}

Proposition~\ref{theo:sigma} suggests that adhering to the principle $\sigma^2 \!=\! \hat{\sigma}^2$ during the optimization of the log marginal likelihood (see Eq.~\eqref{eq:GPLVM_MLE} or Eq.~\eqref{eq:dppca_log_marg}), it is expected to yield the global optimum, thereby mitigating the risk of the model collapse.

\begin{proposition}
\label{lemma:sigma_big}
    If $\sigma^2 \!\in\!  (\lambda^o_{Q\!-\!q\!+\!1}, \lambda^o_{Q\!-\!q}), q\!=\!1,\ldots, Q\!-\!1$, where $\lambda^o_q$ denotes the $q$-th largest eigenvalues of $\frac{1}{M}\vy \vy^{\top}$, then the only stable maximum is the local optimum, with the maximizer $\hat{\vx}$ having $q$ zero columns. 
    In addition, when $q\!=\!Q$, $\sigma^2 > \lambda_{1}^{o}$, the only stable maximum occurs when $\hat{\vx} = \mathbf{0}$ (i.e., homogeneity). 

    If $\sigma^2 \!<\! \lambda^o_N$, the stationary points comprise a cluster of local minimum points, accompanied by the emergence of zero columns in the $\hat{\vx}$.
\end{proposition}
\vspace{-.15in}
\begin{proof}
See App. \ref{app:sigma_big}. 
\vspace{-.05in}
\end{proof}

Proposition~\ref{lemma:sigma_big} implies that an improper choice of $\sigma^2$ can hinder the optimization process, preventing it from reaching the optimum and leading to a loss of information (homogeneity) in $\hat{\vx}$, i.e., the undesirable model collapse.

The aforementioned findings in the linear GPLVM underscore the importance of learning the projection variance $\sigma^2$ and demonstrate how this learning can help mitigate the risk of model collapse. While it is challenging to generalize these results to the broader GPLVM framework due to the model's non-convexity and nonlinearity, they still offer valuable insights into the role of projection variance in preventing model collapse within general GPLVMs (see \S~\ref{subsec:dpca_connection_experiments}). 

\subsection{Kernel Function Flexibility Matters}  
\label{subsec:GPLVM_SM_RFF} 

The occurrence of model collapse is closely linked to the choice of kernel function as well, as the kernel plays a key role in learning the underlying mapping $f(\x)$ in GPLVMs. In particular, if the learned mapping function characterized by the GP posterior diverges from the underlying one, there is a significant possibility that the estimated latent manifold will become distorted or lose crucial feature details, resulting in the model collapse.
 
This phenomenon is depicted in Fig.~\ref{fig:different_kernel_illustration}, where it is evident that the limited flexibility of the preliminary kernels prevents them from adequately exploring the corresponding reproducing kernel Hilbert space (RKHS) to capture the structure of the underlying function $f(\x)$ \cite{theodoridis2020machine}. Consequently, using the preliminary (RBF) kernel can only roughly fit the underlying function, leading to learning a distorted latent manifold--refer to the top of Fig.~\ref{fig:different_kernel_illustration} and the associated latent manifold estimation in Fig.~\ref{subfig:distortion}, where we can see the struggle to fit the model that exhibits 
short-term irregularities.

Conversely, employing a flexible kernel capable of approximating arbitrary kernels allows for thorough exploration of the kernel space, enabling the automatic discovery of the most suitable kernel to capture hidden and possibly complex data patterns and structures, such as periodicity and long tails \cite{wilson2013gaussian, duvenaud2014automatic}. This enhances the capacity to effectively learn the underlying mapping functions and estimate an accurate latent manifold, as evidenced by the learned function using a flexible (SM) kernel in Fig.~\ref{fig:different_kernel_illustration} (top sub-figure) and the latent manifold estimate in Fig.~\ref{subfig:correct_estimation}.

In summary, Fig.~\ref{fig:different_kernel_illustration} demonstrates the importance of kernel flexibility in GPLVMs for mitigating model collapse (distortion) in practice. In this paper, we will employ a kernel capable of approximating arbitrary stationary kernels, namely the SM kernel \cite{wilson2013gaussian}. In the next section, we detail our proposed GPLVM that incorporates the SM kernel while learning projection variance to prevent model collapse.




\section{Preventing Model Collapse}
\label{sec:VI_algorithm}  

Integrating general GPLVM with the SM kernel poses two distinct challenges: 1) high computational costs and 2) intractable model learning \cite{de2021learning,jung2022efficient,chang2023memory}.  Specifically, the computational complexity of training the GPLVM with the SM kernel scales as $\mathcal{O}(N^3)$ with $N$ data points \cite{williams2006gaussian}, rendering it prohibitive in the context of big data.  To tackle the scalability issue of GPLVM, one representative variational method presented by \citet{titsias2010bayesian} involves utilizing sparse GPs based on inducing points \cite{titsias2009variational}. However, this variational method is computationally tractable only for limited preliminary kernel functions, such as the RBF kernel. Recent work has tried to enhance the scalability and flexibility of the GPLVM by using the stochastic variational inference approach proposed by \citet{hensman2013gaussian} \citep{lalchand2022generalised,de2021learning,ramchandran2021latent} . Despite these endeavors, the need to optimize additional inducing points still leads to increased computational burden and the risk of getting stuck in suboptimal solutions. 
Thus, despite the enhanced model capability, these models often face challenges in achieving their theoretical potential to address model collapse (see \S~\ref{sec:experiments}).  

To address the aforementioned issues, we resort to the variational inference technique \cite{jordan1999introduction} and a random Fourier features (RFF) approximation \cite{jung2022efficient, rahimi2007random}, which will enable us to efficiently and scalably learn the SM kernel-embedded GPLVM without introducing extra parameters (inducing points) as required in sparse GP-based methods \cite{titsias2010bayesian, lalchand2022generalised}.
The vanilla RFF approximates any stationary kernel $k(\x,\x^\prime)$ using Monte Carlo integration \cite{rahimi2007random}, i.e.,
\begin{align}
\setlength{\abovedisplayskip}{4.5pt}
\setlength{\belowdisplayskip}{4.5pt}
    & k(\x,\x^\prime) \approx \varphi(\mathbf{x})^\top \varphi(\mathbf{x}^\prime), \   \varphi(\x) \triangleq \sqrt{\frac{2}{L}} 
    \left[ \sin(2 \pi \w_1^{\top} \x),   \right.\\
    &   
    \left.   \cos(2\pi \w_1^{\top} \x), \ldots, \sin(2\pi\w_{\frac{L}{2}}^{\top} \x), \cos(2\pi\w_{\frac{L}{2}}^{\top} \x) \right]  
    \nonumber
\end{align}
where $\{\mathbf{w}_l\}_{l=1}^{L/2}$ are ${L}/{2}$  
i.i.d. spectral points drawn from the density function $p(\mathbf{w})$ of the associated kernel function $k(\x, \x^\prime)$, where $L$ is an positive, even, integer. 

Leveraging the RFF approximation, we can obtain the following SM kernel-embedded GPLVM:
\begin{subequations}
\label{eq:GPLVM_RFF_SM}
\setlength{\abovedisplayskip}{5pt}
\setlength{\belowdisplayskip}{5pt}
    \begin{align}
        & \!\!\! \y_{:,j} \sim \cN(\bm{0}, \ \varphi(\vx)\varphi(\vx)^\top + \sigma^2 \mathbf{I}_N), \  j\!=\!1,\ldots, M,\\
        & \!\!\! \w_l \sim p_{\mathrm{sm}}(\w), \quad l=1,\ldots, {L}/{2}, \label{subeq:sampling_from_Gaussian_mixture}\\
        & \!\!\! \mathbf{x}_i \sim \mathcal{N}(\mathbf{0}, \mathbf{I}_Q), \ \ i=1,\ldots, N,
    \end{align}
\end{subequations}
ensuring both computational scalability and modeling flexibility\footnote{Similar to \citet{gundersen2021latent}, we consider $\vw$ as part of the data-generating process. We then constrain its prior $p(\vw)$ to be Gaussian mixtures, thereby defining SM kernels. For a detailed interpretation of $\vw$, see App.~\ref{appendix:interpretation_eq10_11_12}.}.  The following subsections will further detail our proposed variational inference algorithm to manage the learning tractability and efficacy in addressing the model collapse.

\vspace{-.03in}
\subsection{Approximate Bayesian Inference} 
\label{subsec:scalable_VI} \vspace{-.02in}
Given the SM kernel-embedded GPLVM defined in Eq.~\eqref{eq:GPLVM_RFF_SM},  we utilize the variational inference technique \cite{theodoridis2020machine} to learn the model hyperparameters $\boldsymbol{\theta} \!=\! [\boldsymbol{\theta}_{\text{sm}}, \sigma^2]$, aiming to mitigate the risk of model collapse. 
Specifically, we can immediately obtain the joint distribution of the GPLVM in Eq.~\eqref{eq:GPLVM_RFF_SM} as 
\begin{equation}
\setlength{\abovedisplayskip}{4.5pt}
\setlength{\belowdisplayskip}{4.5pt}
    \begin{aligned}
        p(\vy, \vx, \mathbf{W}) & =  p(\vx) p(\mathbf{W}) p(\vy \vert \vx, \mathbf{W}) \\
        & = p(\mathbf{W})   \prod_{i=1}^{N} p(\x_i) \prod_{j=1}^M  p(\y_{:,j} \vert \vx, \mathbf{W}),
    \end{aligned}
    \label{eq:model_joint}
\end{equation}
where $p(\vw) \!=\! \prod_{l=1}^{{L}/{2}} p_{\mathrm{sm}}(\w)$ is the joint distribution of the spectral points. 
The variational inference method involves constructing a variational lower bound $\mathcal{L}$ of the log marginal likelihood that has the Kullback–Leibler (KL) divergence from approximating the underlying posterior as its slack: $\log p(\vy) \!-\! \mathcal{L} \!=\! \operatorname{KL}[q(\vx, \vw)\| p(\vx, \vw|\vy)]$. By maximizing $\mathcal{L}$ w.r.t. $q(\cdot)$, we improve the quality of the approximation \cite{cao2023variational,cheng2022rethinking}.

For this purpose, we introduce the following variational distribution to approximate the posterior over all the latent variables, $\{ \mathbf{W}, \vx\}$:
\begin{equation}
\setlength{\abovedisplayskip}{4pt}
\setlength{\belowdisplayskip}{4pt}
    q(\vx, \mathbf{W}) \triangleq  p(\mathbf{W}) q(\vx) = p(\mathbf{W}) \prod_{i = 1}^N q(\x_i), 
    \label{eq:variational_dis}
\end{equation}
where $q(\vx) = \prod_{i = 1}^N \mathcal{N}(\x_i \vert \bm{\mu}_i, \bm{S}_i)$, and $\bm{\mu}_i \in \mathbb{R}^Q, \bm{S}_i \in \mathbb{R}^{Q\times Q}$ are the associated free variational parameters.  
The variational distribution $q(\vw)$ is constrained to be the prior distribution, which is essentially equivalent to explicitly assuming that $q(\vw)$ is Gaussian mixtures. See App.~\ref{appendix:interpretation_eq10_11_12} for detailed discussions on this equivalence and other more complex variational distributions of $\vw$. 
Consequently, the variational lower bound for simultaneous learning and inference is ready to be derived and summarized in the following theorem.
\begin{theorem} \label{thm:ELBO}
With the model joint distribution in Eq.~\eqref{eq:model_joint} and the assumed variational distribution in Eq.~\eqref{eq:variational_dis}, the evidence lower bound (ELBO), $\mathcal{L} = \mathbb{E}_{q(\vx, \mathbf{W})} \left[\log {p(\vy, \vx, \vw)} - \log {q(\vx, \vw)} \right]$, for the joint learning and inference is 
\begin{align}
\setlength{\abovedisplayskip}{4.5pt}
\setlength{\belowdisplayskip}{4.5pt}
    \mathcal{L} &\!=\! \mathbb{E}_{q(\vx, \mathbf{W})} \left[\log \frac{p(\mathbf{W}) \prod_{i = 1}^N p(\x_i) \prod_{j=1}^{M} p(\y_{:,j} \vert \vx, \mathbf{W})}{p(\mathbf{W}) \prod_{i = 1}^N q(\x_i) } \right] \nonumber \\
    & \!\!=\! \underbrace{\sum_{j=1}^M \mathbb{E}_{q(\vx, \mathbf{W})} \left[ \log p(\y_{:,j} \vert \vx, \mathbf{W}) \right] }_{\text{Term 1: data reconstruction}} \!-\! \underbrace{\sum_{i=1}^N\operatorname{KL}(q(\x_i) \| p(\x_i))}_{\text{Term 2: regularization}}. 
    \nonumber
\end{align}
\end{theorem}
\vspace{-.1in}
Here, the first term corresponds to the data reconstruction error, which encourages any latent variables $\vx$ and $\vw$ sampled from the variational distribution, $q(\vx, \vw)$, to accurately reconstruct the observations/likelihood. The second term represents a regularization for $q(\vx)$, which discourages significant deviations of $q(\vx)$ from the prior $p(\vx)$.

For the evaluation of $\mathcal{L}$, the second term can be evaluated analytically due to the Gaussian nature of the distributions.
The first term needs to be handled numerically with Monte Carlo estimation, i.e.,
\begin{subequations} \label{eq:evaluation_term_1}
\setlength{\abovedisplayskip}{4.5pt}
\setlength{\belowdisplayskip}{4.4pt}
\begin{align}
   \text{Term 1} & = \sum_{j=1}^M \mathbb{E}_{q(\vx, \mathbf{W})} \left[ \log p(\y_{:,j} \vert \vx, \mathbf{W}) \right] \\
    & \approx  \sum_{j=1}^M  \frac{1}{{I}}\sum_{i=1}^{I} \log \mathcal{N}(\y_{:,j} \vert \bm{0}, \hat{{\vk}}_{\mathrm{sm}}^{(i)} + \sigma^2 \mathbf{I}_N), 
\end{align}
\end{subequations}
where ${I}$ denotes the number of Monte Carlo samples drawn from $q(\vx)$ and $p(\vw)$, and $\hat{{\vk}}_{\mathrm{sm}}$ is the SM kernel matrix approximation constructed by the feature map $\varphi(\cdot)$. See App.~\ref{appendix:ELBO_deriviations} for more computational details.

Note that in Eq.~\eqref{subeq:sampling_from_Gaussian_mixture}, we need to sample $\w_l$ from a Gaussian mixture, involving that first generates an index $i$ from the discrete probability distribution, $P(i) = {\alpha_i}/{\sum_{j=1}^m \alpha_j}, i = 1, \ldots, m$, and then draws sample $\mathbf{w}_l$ from $s_i(\w)$. However, due to the difficulty of reparameterizing the discrete distribution over mixture weights \cite{graves2016stochastic}, maximizing the ELBO w.r.t. the weights $\alpha_i$ using modern off-the-shelf automatic differentiation tools (e.g., PyTorch \cite{paszke2019pytorch}) becomes challenging. To this end, we similarly leverage a differentiable RFF feature map construction approach developed for GP regression models by \citet{jung2022efficient} to ensure inherent differentiability w.r.t. the mixture weights.

\subsection{Differentiable RFF Approximation for SM Kernel} \label{subsec:GPLVM_SM_RFF_autodifferentiable} 
Rather than directly sampling from the Gaussian mixture, we first apply the vanilla RFF to get the corresponding feature map $\varphi_i(\x) \!\triangleq\!  \sqrt{\alpha_i} \cdot \varphi(\x; \{\mathbf{w}_l^{(i)}\}_{l=1}^{{L}/{2}})$, $i \!=\! 1, \ldots, m$, for each mixture component, where the reparametrization trick \cite{kingma2019introduction} is employed to sample $\mathbf{w}_l^{(i)}$ from $s_i(\w)$. Subsequently, the stacking of $m$ feature maps yields the ultimate new RFF approximation for the SM kernel, denoted as $\phi(\x)$,  i.e.,
\begin{equation}
\label{eq:SM_approximations_RFF_new}
\setlength{\abovedisplayskip}{4.5pt}
\setlength{\belowdisplayskip}{4.5pt}
    \!\! \! \phi\left( \x \right) \!=\! \left[ \varphi_1(\x)^\top, \varphi_2(\x)^\top, \ldots, \varphi_m(\x)^\top \right]^\top \!\in\! \mathbb{R}^{mL \times 1}.
\end{equation}
It can be shown that $\phi\left( \x \right)^\top \! \phi\left( \x \right)$ is an unbiased estimator of the SM kernel characterized by the hyperparameters $\bm{\theta}_{\mathrm{sm}} \!=\! \{\alpha_i, \bm{\mu}_i, \bm{\sigma^2_i}\}_{i=1}^m$. The result is succinctly encapsulated in the following proposition \citep{lopez2014randomized}. 
%
\begin{proposition}
\label{prop:SM_RFF_approx}
Let $\vw = \{\mathbf{w}_{1}^{(i)}, \mathbf{w}_{2}^{(i)}, \ldots, \mathbf{w}_{{L}/{2}}^{(i)}\}_{i = 1}^m$ be the spectral points sampled from the distribution $p(\vw) = \prod_{i=1}^m \prod_{l=1}^{{L}/{2}} s_i(\w)$ using the reparameterization trick \cite{kingma2019introduction}. With the RFF feature map constructed in Eq.~\eqref{eq:SM_approximations_RFF_new}, given any inputs $\x$ and $\x^\prime$, $\phi\left( \x \right)^\top \! \phi\left( \x^\prime \right)$ is an unbiased estimator of $k_{\mathrm{sm}}\left(\x, \x^\prime\right)$ with the hyperparameters $\bm{\theta}_{\mathrm{sm}}$,
i.e., 
\begin{equation}
\setlength{\abovedisplayskip}{4.5pt}
\setlength{\belowdisplayskip}{4.5pt}
    \mathbb{E}_{p\left( \vw \right)} \left[\phi\left( \x \right)^\top \phi\left( \x^\prime \right)  \right] = k_{\operatorname{sm}}(\x, \x^\prime; \btheta_{\mathrm{sm}})
    \vspace{-.12in}
\end{equation}
\end{proposition}
\begin{proof}
   See App. \ref{appendix:prop_SM_RFF_approx}
   \vspace{-.1in}
\end{proof}
In fact, given inputs $\vx$ and the new feature map defined in Eq.~\eqref{eq:SM_approximations_RFF_new}, we can further characterize the approximation error bound for the constructed SM kernel matrix approximation, $\hat{\mathbf{K}}_{\mathrm{sm}} \!=\! \Phi_{\mathrm{sm}}(\vx) \Phi_{\mathrm{sm}}(\vx)^\top$, where the random feature matrix $\Phi_{\mathrm{sm}}(\vx)\!=\!\left[\phi\left(\x_1\right),  \ldots , \phi\left(\x_N\right)\right]^\top \!\in \! \mathbb{R}^{N \times mL}$ \cite{jung2022efficient,lopez2014randomized}. 
\begin{theorem}
\label{thm:SM_RFF_approx}
For all small $\epsilon>0$, the approximation error between the underlying SM kernel matrix $\mathbf{K}_{\mathrm{sm}}$ and its RFF approximation $\hat{\mathbf{K}}_{\mathrm{sm}}$ is characterized by 
\begin{equation}
\setlength{\abovedisplayskip}{4.5pt}
\setlength{\belowdisplayskip}{4.5pt}
\begin{aligned}
    & P\left(\left\|\hat{\mathbf{K}}_{\mathrm{sm}} - \mathbf{K}_{\mathrm{sm}} \right\|_2 \geq \epsilon\right) \leq \\
    & N \exp \left(\frac{-3 \epsilon^2 L }{2N a \left(6\left\| \mathbf{K}_{\mathrm{sm}} \right\|_2 + 3 N a \sqrt{m}+8 \epsilon\right)}\right),
\end{aligned}
\end{equation}
where $a = \sqrt{\sum_{i=1}^m \alpha_i^2}$ and $\|\cdot\|_2$ denotes the matrix spectral norm.
\vspace{-.15in}
\end{theorem} 
\begin{proof}
See App. \ref{appendix:thm_SM_RFF_approx}
\vspace{-.1in}
\end{proof}

{
\begin{algorithm}[t!] 
    \caption{\underline{\textsl{advised}\textsc{rflvm}}: Auto-Differentiable Variational Inference for SM-Embedded RFLVMs
    }
    \label{alg:GPLVM_SM_RFF}
    \KwInput{Dataset $\vy$; Initialized model hyperparameters $\btheta$ and variational parameters $\bzeta$.}
    \While{iterations not terminated}{
        Sample $\mathbf{X}$ from $q(\vx) = \prod_{i = 1}^N \mathcal{N}(\x_i \vert \bm{\mu}_i, \bm{S}_i)$  using the reparameterization trick \\ 
        Sample $\mathbf{W}$ from $p(\vw)=\prod_{i=1}^m \prod_{l=1}^{{L}/{2}} s_i(\w)$ using the reparameterization trick \\
        Construct ${\Phi}_{\mathrm{sm}}(\vx)$ using the sampled  $\mathbf{X}$ and $\mathbf{W}$  \\
        Evaluate Term 1 of $\mathcal{L}$ through Eq.~\eqref{eq:evaluation_term_1}\\
        Evaluate Term 2 of $\mathcal{L}$ analytically\\
        Maximize $\mathcal{L}$ and update $\btheta$, $\bzeta$ using Adam\\ 
    }
    \KwOutput{$\btheta$, $\bzeta$.}
\end{algorithm}
}

Beyond the theoretical guarantees of the approximation, the new feature map in Eq.~\eqref{eq:SM_approximations_RFF_new} offers a crucial advantage--it renders the variational lower bound $\mathcal{L}$ differentiable w.r.t. mixture weights $\alpha_i$, leading to the straightforward applicability of the automatic differentiation tools for hyperparameter optimization.  
Leveraging the new feature map, we can apply gradient-based methods (e.g., Adam \cite{kingma2015adam}) to maximize $\mathcal{L}$ w.r.t. model hyperparamters $\btheta$ and the variational parameters $\bm{\zeta}\!=\!\{\bm{\mu}_i, \bm{S}_i\}_{i=1}^N$. The pseudocode summarized in Algorithm \ref{alg:GPLVM_SM_RFF} outlines the implementation of the proposed method, called \texttt{a}uto-\texttt{d}ifferentiable \texttt{v}ariational \texttt{i}nference for \textsc{sm}-\texttt{e}mbedde\texttt{d} \textsc{rff}-\textsc{lvm}, abbreviated as \adrflvm. It is noteworthy that for scenarios where $N\gg mL$, the computational complexity per iteration of \adrflvm scales as $\mathcal{O}(N(mL)^2)$, as elaborated in App. \ref{appendix:ELBO_deriviations}. Notably, this computational complexity aligns with that of the inducing point-based sparse GP method \cite{titsias2010bayesian}. However, \adrflvm enhances the capacity of the GPLVM and mitigates the need for optimizing the inducing points, 
resulting in a lightweight optimization problem while alleviating the model collapse.

\vspace{-.03in}
\section{Related Work} \label{sec:related_work}
We have already described the main differences between our method and inducing points-based methods throughout the paper, e.g., in \S~\ref{sec:VI_algorithm}. Below we briefly introduce other related work on latent variable modeling and refer the reader to App.~\ref{app:related} for more details.

\vspace{-.1in}
\paragraph{VAEs.}
Variational autoencoders (VAEs) \citep{kingma2013auto} skillfully integrate LVMs typically modeled by neural networks with variational inference \cite{bishop:2006:PRML}, empowering the model to generate novel data. Unfortunately, despite the considerable success demonstrated by VAEs in generative tasks \citep{kingma2013auto, zhao2020variational, Nakagawa2023, pmlr-v202-tran23a}, they struggle to capture the underlying compact and informative latent representations of the observed data, resulting in the well-known posterior collapse issue \citep{menon2022forget,wang2022posterior, lucas2019don, razavi2019preventing}, a facet of model collapse (see App.~\ref{app:related}). This phenomenon is partially attributed to the overfitting, stemming from optimizing a large number of parameters in the encoder of VAE, leading to homogeneous latent spaces  \citep{bowman2015generating, sonderby2016train,zhu2023markovian}.

\vspace{-.1in}
\paragraph{RFLVMs.} 
In addition to inducing points-based GPLVMs, the random feature latent variable model (RFLVM) adopts the RFF approximation of the kernel function as a variant of the GPLVM and leverages a Dirichlet process (DP) mixture of Gaussians to learn the associated spectral density of the kernel function \cite{rahimi2007random,oliva2016bayesian,gundersen2021latent,zhang2023bayesian}. Despite the capacity to approximate arbitrary stationary kernels, the effectiveness in addressing model collapse in the RFLVM might be compromised by the ``rich-get-richer" property inherent in the DP mixture prior \cite{gundersen2021latent}, which places a strong assumption regarding the data generation process \citep{poux2023powered}. A comprehensive comparison between our \adrflvm and the SOTA models can be found in Table~\ref{table:comparison}, App.~\ref{app:related}.

\vspace{-.03in}
\section{Experiments} \label{sec:experiments}

We showcase the impact of the projection variance and kernel flexibility on model collapse in \S~\ref{subsec:dpca_connection_experiments} and \S~\ref{subsec:S_shape_demo}. In \S~\ref{subsec:real_data_knn_acc} and \S~\ref{subsec:missing_data_reconstruction}, we further corroborate the superior performance of \adrflvm in latent representation learning on various real-world datasets. More experimental details can be found in App.~\ref{Appendix: Experiment_details}, and the code is publicly available at \href{https://github.com/zhidilin/advisedGPLVM}{https://github.com/zhidilin/advisedGPLVM}.

\vspace{-.03in}
\subsection{Projection Variance Matters}  
\label{subsec:dpca_connection_experiments} 
\vspace{-.03in}
To evaluate the impact of the projection variance in general GPLVM,  we apply the \adrflvm on the \textsc{mnist} dataset \cite{lecun1998mnist}. We quantify the degree of model collapse under two configurations of $\sigma^2$: learned and fixed. The degree of the model collapse is evaluated by counting the number of zero-columns in the learned latent variable $\hat{\vx}$ and measuring its \textsc{k}-nearest neighbors (\textsc{knn}) classification accuracy. Detailed results are depicted in Fig.~\ref{fig:sigma}. 

On the left-hand side of Fig.~\ref{fig:sigma}, it is observed that, when $\sigma^2$ is fixed, the latent variable learned by the \adrflvm rapidly collapses to zero as the value of $\sigma^2$ increases. This observation aligns with the findings in the linear GPLVM (see Proposition \ref{lemma:sigma_big}). Additionally, the inferior performance of the \textsc{knn} accuracy depicted on the right-hand side of Fig.~\ref{fig:sigma} illustrates that, without learning $\sigma^2$, the proposed \adrflvm tends to recover a vague and uninformative latent representation. In stark contrast, 
\adrflvm with a learned $\sigma^2$ effectively mitigates the risk of model collapse, irrespective of the initialization value of $\sigma^2$ or the metric employed. This supports our hypothesis regarding the importance of learning $\sigma^2$ to prevent model collapse in general GPLVMs. 

\begin{figure}[t!]
    \centering
    \includegraphics[width=.49\linewidth]{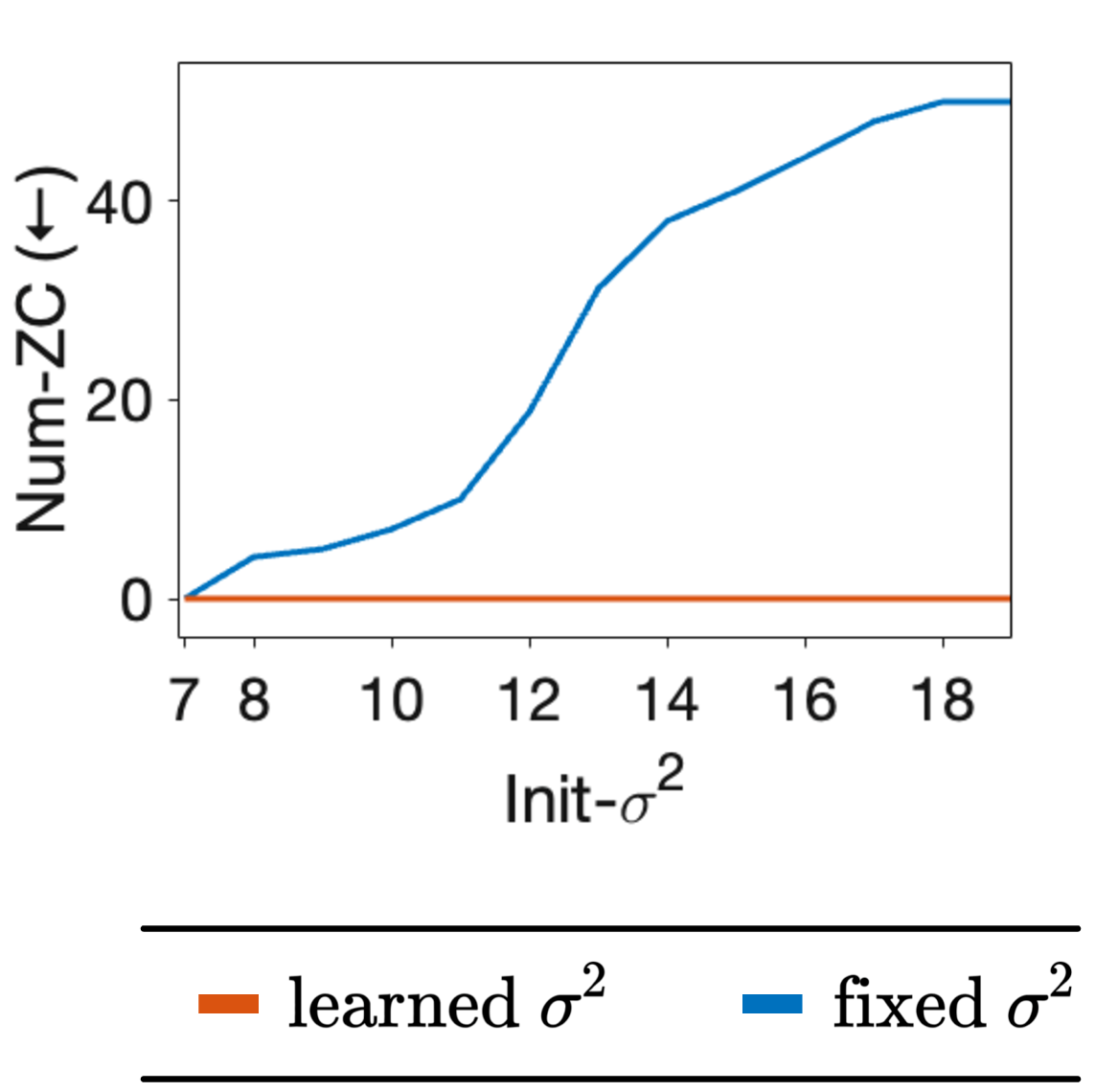} 
    \includegraphics[width=.49\linewidth]{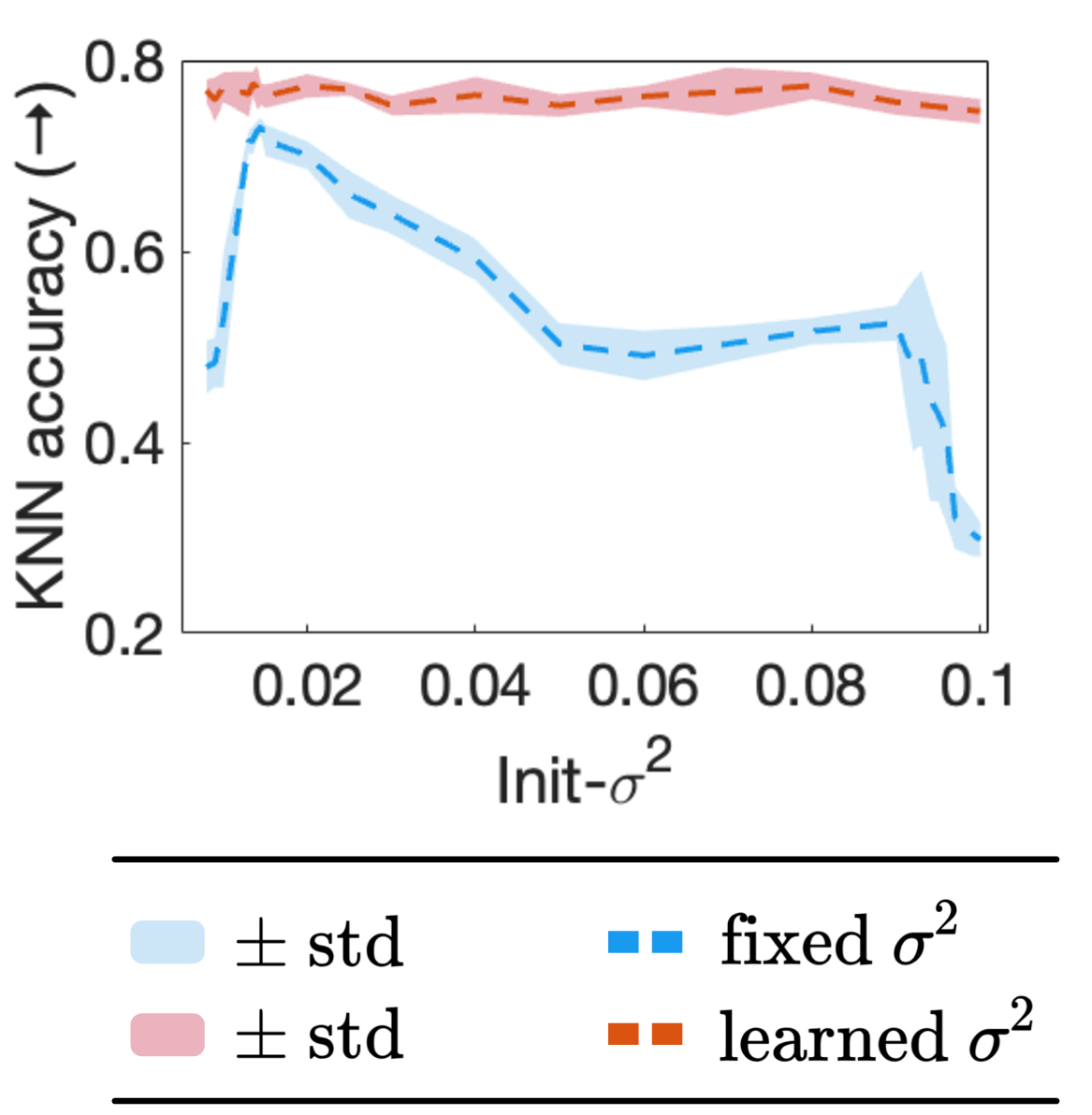} 
    \caption{ \textbf{Left}: The number of zero-columns (short as \textsc{n}um-\textsc{zc}) in the latent variable $\vx$ versus the initialization value of $\sigma^2$ (defined as \textsc{i}nit-$\sigma^2$). \textbf{Right}: \textsc{knn} classification accuracy against \textsc{i}nit-$\sigma^2$. Standard deviation is calculated over five experiments.
    \vspace{-.2in}
    }  
    \label{fig:sigma}
\end{figure}
%

%
\begin{figure*}[t!]
    \centering
    \includegraphics[width=.395\linewidth]{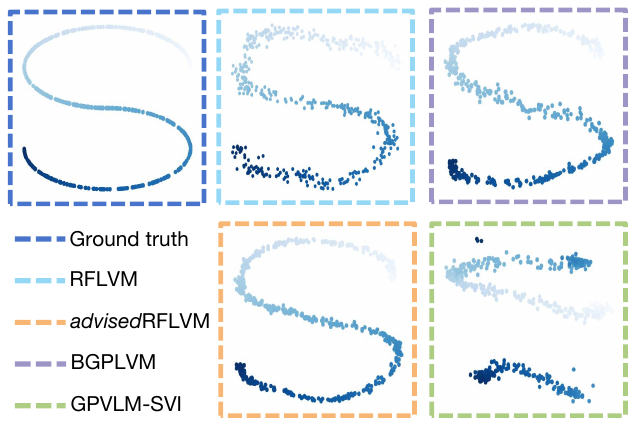}    \hspace{.3in}
    \includegraphics[width=.48\linewidth]{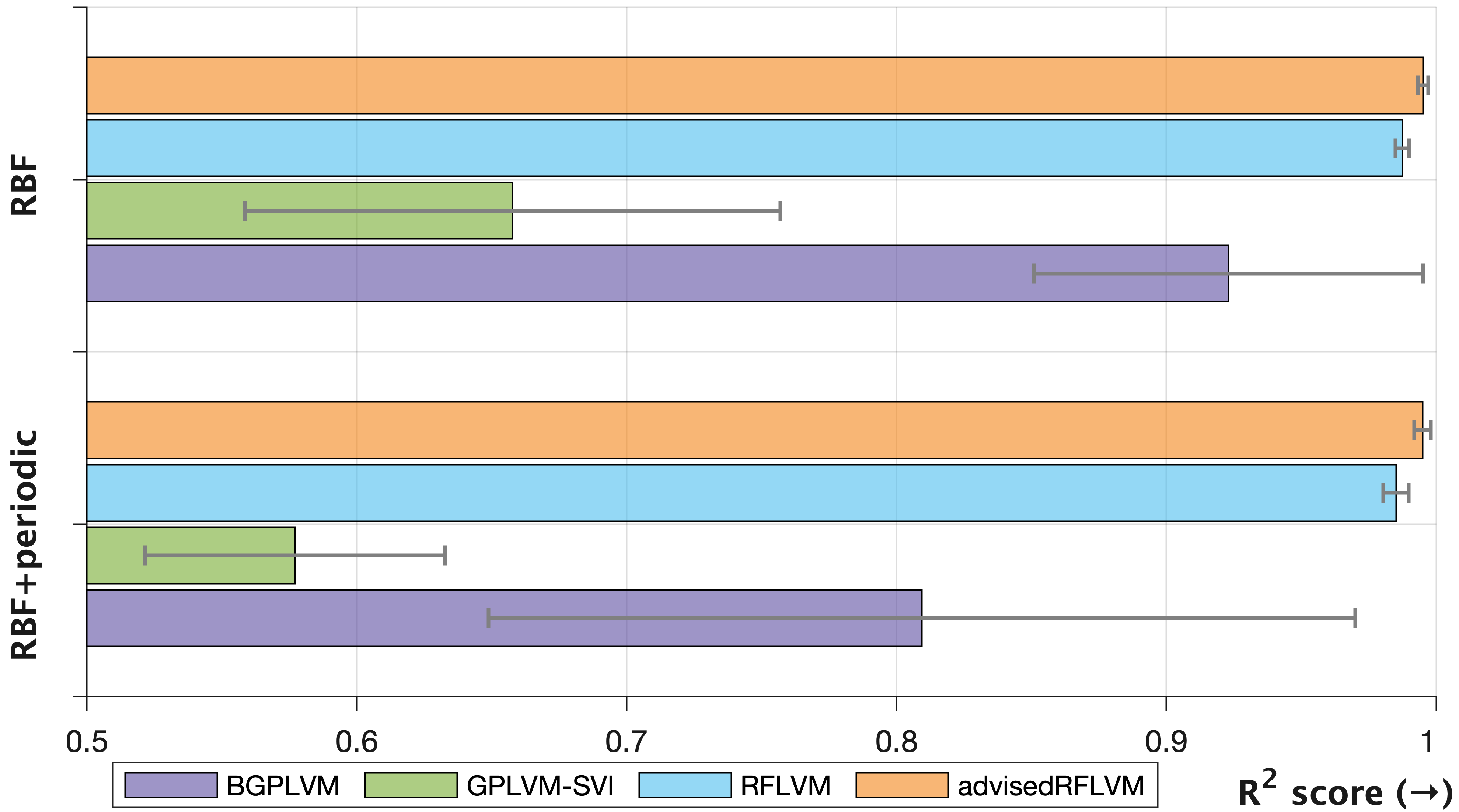}
    \caption{\textbf{Left}: Learned latent manifold in ``RBF+periodic'' dataset. \textbf{Right}: $\mathrm{R}^2$ score performance over different models in two datasets. 
    }
    \label{fig:S-shape-R2-score}
\end{figure*}

\begin{table*}[t]
\caption{Classification accuracy evaluated by fitting a \textsc{knn} classifier $(k=1)$ with five-fold cross-validation. Mean and standard deviation are computed over five experiments, and the top performance is in bold. \vspace{-.1in}}
\label{table:KNN}
\centering
\setlength{\tabcolsep}{4.0mm}
{
\scalebox{0.84}{
\begin{tabular}{c cccccc}
\toprule
\textsc{dataset} & PCA          & LDA               & Isomap            & HPF               & BGPLVM               & GPLVM-SVI          \\ \midrule \midrule
\rowcolor[HTML]{f2f2f2} 
\textsc{Bridges}    & 0.841 $\pm$ 0.007 & 0.668 $\pm$ 0.053 & 0.797 $\pm$ 0.025 & 0.544 $\pm$ 0.109 & 0.818 $\pm$ 0.037 & 0.796 $\pm$ 0.019  \\
\textsc{Cifar-10}   & 0.267 $\pm$ 0.002 & 0.227 $\pm$ 0.006 & 0.272 $\pm$ 0.006 & 0.208 $\pm$ 0.006 & 0.271 $\pm$ 0.014 & 0.251 $\pm$ 0.012  \\
\rowcolor[HTML]{f2f2f2} 
\textsc{Mnist}      & 0.365 $\pm$ 0.012 & 0.233 $\pm$ 0.026 & 0.444 $\pm$ 0.021 & 0.314 $\pm$ 0.040 & 0.567 $\pm$ 0.033 & 0.344 $\pm$ 0.054  \\
\textsc{Montreal}   & 0.678 $\pm$ 0.013 & 0.602 $\pm$ 0.028 & 0.709 $\pm$ 0.005 & 0.618 $\pm$ 0.001 & 0.725 $\pm$ 0.012 & 0.676 $\pm$ 0.010   \\
\rowcolor[HTML]{f2f2f2} 
\textsc{Newsgroups} & 0.392 $\pm$ 0.005 & 0.391 $\pm$ 0.018 & 0.397 $\pm$ 0.010 & 0.334 $\pm$ 0.019 & 0.385 $\pm$ 0.010 & 0.378  $\pm$ 0.018 \\
\textsc{Yale}       & 0.543 $\pm$ 0.008 & 0.338 $\pm$ 0.023 & 0.588 $\pm$ 0.017 & 0.511 $\pm$ 0.019 & 0.553 $\pm$ 0.036 & 0.521 $\pm$ 0.015  \\ \midrule 
\textsc{dataset}            & VAE               & NBVAE     &  DCA         &  CVQ-VAE            & RFLVM            & \textsl{advised}RFLVM           \\ \midrule \midrule
\rowcolor[HTML]{f2f2f2} 
\textsc{Bridges}    & 0.751 $\pm$ 0.016 & 0.758 $\pm$ 0.038 &   0.702 $\pm$ 0.036   &       0.688  $\pm$ 0.013
      & 0.846 $\pm$ 0.039 & \textbf{0.846 $\pm$ 0.015}   \\
\textsc{Cifar-10}  & 0.266 $\pm$ 0.002 & 0.259 $\pm$ 0.005 &   0.255 $\pm$ 0.019    &      0.224 $\pm$ 0.012
      & 0.284 $\pm$ 0.103 & \textbf{0.290 $\pm$ 0.006}    \\
\rowcolor[HTML]{f2f2f2} 
\textsc{Mnist}       & 0.643 $\pm$ 0.021 & 0.281 $\pm$ 0.012 &    0.171 $\pm$ 0.075     &     0.128 $\pm$ 0.005
     & 0.602 $\pm$ 0.055 & \textbf{0.795 $\pm$ 0.015}    \\
\textsc{Montreal}   & 0.668 $\pm$ 0.012 & 0.716 $\pm$ 0.009 &    0.685 $\pm$ 0.716    &       0.646 $\pm$ 0.003
    & 0.769 $\pm$ 0.010 & \textbf{0.789 $\pm$ 0.013}    \\
\rowcolor[HTML]{f2f2f2} 
\textsc{Newsgroups}& 0.385 $\pm$ 0.002 & 0.398 $\pm$ 0.010 &    0.399 $\pm$ 0.034     &      0.356 $\pm$ 0.019
    & 0.413 $\pm$ 0.009 & \textbf{0.418 $\pm$ 0.007}   \\
\textsc{Yale}     & 0.611 $\pm$ 0.020 & 0.456 $\pm$ 0.046 &    0.284 $\pm$ 0.054   &        0.338  $\pm$ 0.002
  & 0.653 $\pm$ 0.067 & \textbf{0.765 $\pm$ 0.010}  \\ \bottomrule
\end{tabular}
}
}\vspace{-.06in}
\end{table*} 

\vspace{-.03in}
\subsection{S-shaped Latent Manifold Learning}  
\label{subsec:S_shape_demo} \vspace{-.02in}

Next, we demonstrate the importance of kernel flexibility in preventing model collapse, utilizing two synthetic datasets, each consisting of $N\!\!=\!\!500$ observations with $M\!\!=\!\!100$ dimensions. Both datasets are generated from a GPLVM with a two-dimensional ($2$-D) latent $S$-shaped manifold, but employing distinct kernel configurations. One employs a basic RBF kernel, while the other utilizes a more complex combination of an RBF kernel and a periodic kernel \cite{williams2006gaussian}. We compare our \adrflvm with three GPLVM variants: BGPLVM \cite{titsias2010bayesian}, GPLVM-SVI \cite{lalchand2022generalised}, and RFLVM \cite{zhang2023bayesian,gundersen2021latent}. In the case of BGPLVM and GPLVM-SVI, the default setting (see App.~\ref{Appendix: Experiment_details}) is used except that the number of inducing points is selected from the set $\{6, 10, 20, 30, 60, 120\}$, which yields the best inference performance. 

Figure~\ref{fig:S-shape-R2-score} reports the results for the $S$-shaped manifold learning, where the coefficient of determination ($\mathrm{R}^2$ score) \cite{chicco2021coefficient} is used to quantify the ``closeness'' between the inferred manifold (after post-affine transformation) and the ground truth manifold. The results indicate that \adrflvm and RFLVM consistently outperform BGPLVM and GPLVM-SVI in both synthetic datasets. It is obvious that GPLVM-SVI exhibits the worst performance, and BGPLVM shows fluctuated performance, although, in some realizations, they can reasonably estimate the shape of $\vx$ (see the left illustration in Fig.~\ref{fig:S-shape-R2-score}). The fluctuated performance of BGPLVM and GPLVM-SVI suggests that optimizing the additional inducing points (variational parameters) can complicate the learning process and incur such instability. 

The performance gain of the \adrflvm and RFLVM can be attributed to the kernel flexibility, which is particularly evident when the dataset is generated from the underlying GPLVM with a hybrid of RBF kernel and periodic kernel. This validates the crucial role of kernel function flexibility in preventing model collapse. Nevertheless, \adrflvm consistently outperforms the RFLVM, although RFLVM theoretically is capable of approximating arbitrary stationary kernels as well. This discrepancy may stem from the biased assumption of DP priors for the spectral densities in RFLVM \cite{zhang2023bayesian}. Such bias can lead to unfair exposure for the density weights, resulting in only a few effective densities and a degenerated approximation capacity \cite{gundersen2021latent}. Moreover, the RFLVM is based on MCMC sampling which may be inferior in this setting to the \adrflvm, which optimizes the ELBO in terms of inference efficiency. 

%
\begin{table*}[t]
\caption{Missing data imputation on the \textsc{mnist} and \textsc{brendan} datasets.
\label{tab:missing_data}}
\vspace{-2ex}
\small
\setlength{\tabcolsep}{1.05mm}{
    \scalebox{.92}{
        \begin{tabular}{cr|| cccc| cccc| cccc| cccc}
        \toprule
        \multirow{2}{*}{\textsc{dataset}} 
        & \multirow{2}{*}{\textsc{metric}}   
        & \multicolumn{4}{c|}{\textsc{vae}}  
        & \multicolumn{4}{c|}{\textsc{bgplvm}}      
        & \multicolumn{4}{c|}{\textsc{rflvm}}  
        & \multicolumn{4}{c}{\adrflvm}   
        \\  
        \cmidrule(lr){3-6}
        \cmidrule(lr){7-10}
        \cmidrule(lr){11-14}
        \cmidrule(lr){15-18}
        & &0\% &10\% &30\% &60\%   &0\% &10\% &30\% &60\%   &0\% &10\% &30\% &60\%   &0\% &10\% &30\% &60\%
        \\ \midrule  \midrule
        \multirow{2}{*}{\textsc{mnist}}  
        & \textsc{knn acc} ($\uparrow$)    & 0.715 & 0.689 &  0.660 & 0.585  &0.603 &0.598 &0.541 &0.476  &0.602 &0.391 &0.345 &0.273  &\textbf{0.806} &\textbf{0.802} &\textbf{0.777} &\textbf{0.636} \\
        & \textsc{test mse} ($\downarrow$) &0.035& 0.038 &   0.045 &  \textbf{0.068} &0.048 &0.040 & 0.057 &0.098 &0.066 &0.067 &0.070 & 0.120 &\textbf{0.025} & \textbf{0.028} &\textbf{0.039}& {0.068}
        \\ \midrule \midrule
        \multirow{1}{*}{\textsc{brendan}}  
        & \textsc{test mse} ($\downarrow$) &0.005 & 0.009 &  \textbf{0.043} & \textbf{0.150} & 0.006 &0.041 &0.087 &0.197 & 0.010 & 0.015 & 0.049 & 0.153 & \textbf{0.003} & \textbf{0.009} & {0.045} & {0.152} 
        \\ \bottomrule
        \end{tabular}
    }
}
\vspace{-1ex}
\end{table*}

\subsection{Real Dataset Evaluation}  
\label{subsec:real_data_knn_acc} 

 This subsection further demonstrates the ability of \adrflvm to capture the latent space on multiple real-world datasets (see Table~\ref{table:KNN}), where the dataset sizes of \textsc{mnist} and \textsc{cifar} are reduced for accommodating the high complexity of RFLVM (see App.~\ref{app:ds_preprocessing} for further details). For each dataset, we hold the labels and employ them to evaluate the estimated latent space using \textsc{knn} classifier with five-fold cross-validation. 
 In addition to the GPLVM variants used in \S~\ref{subsec:S_shape_demo}, we also encompass various recent VAEs \citep{kingma2019introduction,zhao2020variational,eraslan2019single,zheng2023online} and classic dimensionality reduction methods. The \textsc{knn} classification accuracy results for all the competing methods are presented in Table~\ref{table:KNN}. 

The results demonstrate that \adrflvm consistently achieves the highest \textsc{knn} accuracy across all datasets. This suggests that the latent variables estimated by \adrflvm are more informative compared to the other methods.  The four classic methods, PCA \citep{wold1987principal,pearson1901liii}, hierarchical Poisson factorization (HPF) \cite{gopalan2015scalable}, latent Dirichlet allocation (LDA) \cite{blei2003latent}, and Isomap \citep{balasubramanian2002isomap} showing inferior performance, are primary attributed to their limited model flexibility.

For the VAE models, despite their impressive approximation capabilities through neural network-based decoders and encoders \cite{kingma2019introduction}, they often fall short in their latent space learning performance. This is because optimizing numerous neural network parameters can result in overfitting, rendering these deterministic neural networks directed toward wrong latent spaces. In contrast, GPLVM variants prevent the need for neural network parameter optimization. More importantly, the inherent regularization imposed by the GP prior mitigates overfitting and thus enhances the generalization capability for latent space learning \cite{wilson2020bayesian}. These lead to GPLVM-based models being expected to attain higher \textsc{knn} accuracy. Nevertheless, the results in Table~\ref{table:KNN} show that BGPLVM and GPLVM-SVI can only attain comparable performance compared to the PCA. This mainly attributed to the inherently inadequate kernel flexibility and the additional optimization burden of the variational parameters. 
RFLVM  consistently exhibits a slightly inferior performance compared to \adrflvm, primarily due to the unfair exposure of density weights and the inefficient and unscalable MCMC inference algorithm mentioned in \S~\ref{subsec:S_shape_demo} and \S~\ref{sec:related_work}. We also conducted additional simulations on larger datasets. The results, presented in Appendix \ref{app:larger_datasets}, emphasize the superiority of \adrflvm over state-of-the-art variants regardless of the dataset size.

\subsection{Missing Data Imputation}  
\label{subsec:missing_data_reconstruction}

This subsection further evaluates the performance of \adrflvm in the task of imputing missing data on two image datasets, namely \textsc{mnist} and \textsc{brendan}  \cite{roweis2000nonlinear}. Specifically, we randomly hold out a certain proportion (0\%, 10\%, 30\%, and 60\%) of the elements in the observed data matrix, $\vy$, and subsequently we utilize \adrflvm to estimate latent variables $\mathbf{X}$ from the incomplete dataset (denoted as $\mathbf{Y}_{obs}$). We then impute the missing values $\mathbf{Y}_{miss}$ by their posterior mean $\hat{\vy}_{miss} \!=\! \mathbb{E}[ \mathbf{Y}_{miss} \mid  \vx, \mathbf{Y}_{obs}, -]$. The imputation performance is evaluated through the mean square error (MSE) between $\hat{\vy}_{miss}$ and the ground-truth ${\vy}_{miss}$.  Additionally, \textsc{knn} classification accuracy is reported for the \textsc{mnist} dataset to illustrate the latent representation learning results. Table~\ref{tab:missing_data} presents the performance of the \adrflvm against competing methods. The results indicate that \adrflvm outperforms most competitors in reconstructing observations and recovering latent representations, regardless of the proportion of missing data. 
Despite VAE exhibiting reconstruction capabilities comparable to \adrflvm, it still lags behind in recovering informative latent variables due to its potential overfiting and inherent posterior collapse issues \citep{wang2022posterior}. More details about the reconstruction performance of \adrflvm  are provided in App.~\ref{app:missing_data}, showing its superior ability to restore missing pixels.

\vspace{-0.4ex}
\section{Conclusions}
\vspace{-0.4ex}
We have introduced our novel \adrflvm to address model collapse due to inadequate kernel flexibility and inappropriate projection variance selection in GPLVMs. By integrating the SM kernel and the differentiable RFF approximation, our \adrflvm not only enhances model flexibility but also enables the use of modern automatic differentiation tools for optimizing essential parameters, including the projection variance within the variational inference framework. Empirical results across diverse datasets corroborate the superiority of our \adrflvm in learning compact and informative latent representations, highlighting the importance of learning projection variance and kernel flexibility in mitigating model collapse. Furthermore, our model outperforms various state-of-the-art latent variable models, including VAEs and other GPLVM variants. In future work, we are focusing on how to further enhance the variational inference algorithm presented in this paper. We hope that, through our endeavors, we may scale up our LVM for scenarios with massive data sets as an efficient alternative to the resource-intensive deep learning models.

\section*{Acknowledgements}
The authors would like to thank the anonymous referees for their valuable comments that improved the quality of the paper.
The work of Feng Yin was supported by the NSFC under Grant No. 62271433, and in part by the Shenzhen Science and Technology Program under Grant No. JCYJ20220530143806016.
The work of Michael Minyi Zhang was supported by the HKU-URC Seed Fund for Basic Research for New Staff. 

\section*{Impact Statement}
This work introduces a novel probabilistic latent variable model tailored to effectively capture the underlying structures of the observed data, which allows us to provide informative but concise foundational knowledge for analyzing highly complex tasks, such as the analysis of social issues, research on human behavior, and exploration of cognitive mechanisms. Technically, this work, conducting theoretical analyses on the impact of the projection variance on model collapse, will strengthen the understanding of broader researchers and engineers on the ``default'' learning of the projection variance. We also carefully examine the impact of kernel flexibility, and all these rigorous examinations of the potential reasons for model collapse enhance model interpretability, which is crucial for safety-critical systems such as autonomous driving and intelligent healthcare. 
\vspace{-.1in}
\paragraph{Limitations and future works.}
Our model faces limitations in handling out-of-distribution data, which requires explicitly learning an encoding function from observable data points into a latent representation. One potential solution to address this is to assume a $\vy$-dependent parametric variational distribution of latent variables, $q(\vx \vert \vy)$, where the parameters of the distribution are modeled by an encoder network that takes the observation $\vy$ as input. Consequently, upon completion of the training process, the encoder network can be employed to infer the latent variables of the out-of-distribution data. Another limitation is that despite the reduction in the complexities (linear with $N$), the practical training time of our method may not be endurable for massive datasets.

\bibliography{main}
\bibliographystyle{icml2024}

\newpage

\appendix
\onecolumn




\glsresetall
\renewcommand{\contentsname}{\centering\textsc{Appendices}}
\tableofcontents
\addtocontents{toc}{\protect\vspace{10pt}}            
\addtocontents{toc}{\protect\setcounter{tocdepth}{3}} 
\vspace{.3in}
\section{Model Collapse Mechanism Revelation}



In \S~\ref{app:relation_dpca_gplvm}, we provide a detailed introduction to dual probabilistic principal analysis (DPPCA) \citep{lawrence2005probabilistic} and establish its connection with the linear GPLVM. Building upon this connection, a detailed derivation of Theorem \ref{theo:dpca_stationary_main_text} is provided, delineating the forms of stationary points. Through further exploration of the optimization landscapes around stationary points, we provide detailed proofs for Proposition \ref{theo:sigma} and Proposition \ref{lemma:sigma_big}, located in \S~\ref{app:sigma} and \S.~\ref{app:sigma_big}, respectively.

\subsection{Special Case of GPLVM: Dual Probabilistic Principal Analysis (DPPCA)}
\label{app:relation_dpca_gplvm}

In DPPCA \citep{lawrence2005probabilistic}, each observed data point $\mathbf{y}_i \in \mathbb{R}^{M} $ is generated from a latent variable $\mathbf{x}_i \in \mathbb{R}^{Q}$ through a linear transformation $\mathbf{A} \in \mathbb{R}^{M \times Q}$, i.e., 
\begin{subequations}
    \begin{align}
         & \mathbf{y}_{i} \sim \mathcal{N}\left(  \va \mathbf{x}_i, \sigma^{2} \mathbf{I}_{M} \right),     \\ 
         & p(\va) \sim \prod^{M} \mathcal{N} \left( \mathbf{0}, \mathbf{I}_{Q} \right), 
    \end{align}
\end{subequations}
where $\sigma^2$ represents the projection variance, representing the uncertainty. For $N$ observed data points in DPPCA, denoted as $\vy \in \mathbb{R}^{N \times M}$, the marginal likelihood, obtained by marginalizing the transformation matrix $\mathbf{A}$, can be represented as follows: 
\begin{align}
    \setlength{\abovedisplayskip}{4.5pt}
    \setlength{\belowdisplayskip}{4.5pt}
    \mathbf{y}_{:,j} \vert \vx \sim \mathcal{N}\left( \mathbf{0}, \mathbf{X} \mathbf{X}^{\top} + \sigma^{2} \mathbf{I}_N \right), \quad j = 1, \ldots, M,  
    \label{eq:ppca}
\end{align}
where $\mathbf{y}_{:,j}$ denotes $j$-th column in the observed data $\vy$. Consequently, the maximum likelihood estimate (MLE) for the latent variable, denoted as $ \hat{\vx}_{\text{DPPCA}} $, can be derived by maximizing the logarithm of Eq.~\eqref{eq:ppca} through, e.g., gradient-based methods, i.e.,  
\begin{equation}
\label{eq:DPPCA_MLE}
 \hat{\vx}_{\text{DPPCA}} =  \max_{\vx} \ \log  \prod_{j=1}^M \mathcal{N}\left(\mathbf{y}_{:,j} \mid \bm{0}, \vx \vx^{\top} + \sigma^2 \mathbf{I}_N \right). 
\end{equation}
Building upon Eq.~\eqref{eq:DPPCA_MLE} and the optimization problem given in Eq.~\eqref{eq:GPLVM_MLE}, a connection between GPLVM and DPPCA can be established \citep{lawrence2005probabilistic}, encapsulated in the following corollary: 
\begin{corollary}
\label{theo:relation_dpca_gplvm}
    Assuming the kernel function in GPLVM is defined as the inner product kernel with $k(\mathbf{x}, \mathbf{x}^{\prime}) = \mathbf{x}^{\top} \mathbf{x}^{\prime}$, the stationary points for the linear GPLVM, as expressed in Eq.~\eqref{eq:GPLVM_MLE}, are identical to the stationary points of DPPCA, $\hat{\vx}_{\text{DPPCA}}$.  
    \vspace{-.12in} 
\end{corollary} 
\begin{proof}
    If the kernel function is the inner product kernel, i.e., $k(\x, \x^{\prime}) = \x^{\top} \x^{\prime}$, the marginal likelihood of the linear GPLVM can be reformulated as, 
    \begin{equation}
        \begin{aligned}
        p(\vy \vert \vx) = \prod_{j=1}^M \mathcal{N}\left(\mathbf{y}_{:,j} \vert \bm{0}, \vx \vx^{\top} + \sigma^2 \mathbf{I}_N \right).
        \end{aligned}
    \end{equation}
    Then, the stationary points of the linear GPLVM, $ \hat{\vx}$, is given by
    \begin{align}
        \hat{\vx}
        &= \max_{\vx} \log p(\vy \vert \vx) = \max_{\vx}  M \left\{  -\frac{N}{2} \log 2 \pi - \frac{1}{2} \log \left|  \vx \vx^{\top} + \sigma^2 \mathbf{I}_N \right|  \right\} - \frac{1}{2} \operatorname{tr}\left( \left(  \vx \vx^{\top} + \sigma^2 \mathbf{I}_N\right)^{-1} \vy \vy^{\top} \right). 
    \end{align}
   The stationary points of DPPCA, $\hat{\vx}_{\text{DPPCA}}$, given in Eq.~\eqref{eq:DPPCA_MLE}, can be reformulated as 
    \begin{align}
        \hat{\vx}_{\text{DPPCA}} &= \max_{\vx}  M \left\{  -\frac{N}{2} \log 2 \pi - \frac{1}{2} \log \left|  \vx \vx^{\top} + \sigma^2 \mathbf{I}_N \right| \right\}  - \frac{1}{2} \operatorname{tr}\left( \left(  \vx \vx^{\top} + \sigma^2 \mathbf{I}_N\right)^{-1} \vy \vy^{\top} \right).  
        \label{eq:dppca_log_marg}
    \end{align} 
    It is evident that the stationary points of the linear GPLVM is identical to the stationary points of DPPCA.   
\end{proof}

\subsection{Proof of Theorem \ref{theo:dpca_stationary_main_text}}
\label{app:dpca_stationary_x}
This subsection conducts a comprehensive derivation, elucidating the stationary points of the linear GPLVM. Our derivation generally adheres to the one in  \cite{lawrence2005probabilistic}, albeit with subtle distinctions.
\begin{proof}
    Recall that, the log marginal likelihood can be expressed as 
      \begin{align}
        L \triangleq M \left\{  -\frac{N}{2} \log 2 \pi - \frac{1}{2} \log |  \vx \vx^{\top} + \sigma^2 \mathbf{I}_N |  \right\} - \frac{1}{2} \operatorname{tr}\left( \left(  \vx \vx^{\top} + \sigma^2 \mathbf{I}_N\right)^{-1} \vy \vy^{\top} \right).  
        \label{eq: ppca_ll_1}
   \end{align}
   Define $\vk \triangleq \vx \vx^{\top} + \sigma^2 \mathbf{I}_N$, Eq.~\eqref{eq: ppca_ll_1} could be reformulated as 
   \begin{align}
        L = M \left\{  -\frac{N}{2} \log 2 \pi - \frac{1}{2} \log | \mathbf{K} | \right\}  - \frac{1}{2} \operatorname{tr}( \mathbf{K}^{-1} \vy \vy^{\top} ).   
        \label{eq: ppca_ll}
   \end{align} 
   Taking the gradient of Eq.~\eqref{eq: ppca_ll} with respect to $\vx$, we have 
   \begin{align}
       \frac{ \partial L}{ \partial \vx } =  \mathbf{K}^{-1} \vy \vy^{\top} \mathbf{K}^{-1} \vx - M \mathbf{K}^{-1} \vx.  
   \end{align}
   Setting this gradient to zero, the stationary points of Eq.~\eqref{eq: ppca_ll} should satisfy
    \begin{align}
         \frac{1}{M} \vy \vy^{\top} \mathbf{K}^{-1} \vx = \vx. 
         \label{eq:ppca_fixed_point}
    \end{align}
    According to Lemma \ref{Lemma:Woodbury_matrix_identity}, we have 
    \begin{align}
        \mathbf{K}^{-1} \vx = \left[ \vx \vx^{\top} + \sigma^2 \mathbf{I}_N \right]^{-1} \vx = \vx    \left[ \vx^{\top} \vx + \sigma^2 \mathbf{I}_Q \right]^{-1}. 
        \label{eq:push_through}
    \end{align}
    We conduct singular value decomposition (SVD) to $\vx$, and get $\vx = \vu \vl \vv^{\top}$,  where $\vu \in \mathbb{R}^{N \times Q}$, $\vl = \operatorname{diag}(l_1, l_2, \ldots, l_Q)\in \mathbb{R}^{Q \times Q}$ is a diagonal matrix, and $\vv\in \mathbb{R}^{Q\times Q}$.   
    Together with Eq.~\eqref{eq:push_through} and  Eq.~\eqref{eq:ppca_fixed_point}, we have    
    \begin{subequations}
    \label{eq:x_solution}
        \begin{align}
        & \frac{1}{M} \vy \vy^{\top} \vu \vl \left[ \vl^2 + \sigma^2 \mathbf{I}_Q \right]^{-1} \vv^{\top} = \vu \vl \vv^{\top}, \\ 
        \Rightarrow \qquad & \frac{1}{M} \vy \vy^{\top} \vu \vl \left[ \vl^2 + \sigma^2 \mathbf{I}_Q \right]^{-1}  = \vu \vl , \\ 
        \Rightarrow \qquad & \frac{1}{M} \vy \vy^{\top} \vu \vl = \vu ( \sigma^2 \mathbf{I}_Q + \vl^2) \vl. 
        \end{align}
    \end{subequations}
    Then, we have:
    \begin{itemize}
        \item  If $l_i \neq 0$, it indicates that $\frac{1}{M} \vy \vy^{\top} \u_i = \u_i (\sigma^2 + l_i^{2})$, implying $\u_i$ is an eigenvector of $\frac{1}{M} \vy \vy^{\top}$ corresponding to the eigenvalue $\lambda_i = \sigma^2 + l_i^{2}$. 
        \item  If $l_i = 0$, the vector $\mathbf{u}_i$ is arbitrary. We can set it to be an eigenvector of $\frac{1}{M} \vy \vy^{\top}$ for consistency. 
    \end{itemize}
     Consequently, all potential stationary solutions for $\vx$ can be written as 
    \begin{align}
        \hat{\vx} = \vu_{Q} \left( \boldsymbol{\Lambda}_{Q} - \sigma^2 \mathbf{I}_Q \right)^{1/2} \mathbf{R}, 
        \label{eq:opt_x_ppca}
    \end{align}
    where $\vu_{Q} \in \mathbb{R}^{N \times Q}$ is a matrix whose columns are eigenvectors of $\frac{1}{M} \vy \vy^{\top}$, $\mathbf{R} \in \mathbb{R}^{Q \times Q} $ is an arbitrary orthogonal matrix and $\boldsymbol{\Lambda}_{Q} \in \mathbb{R}^{Q \times Q} $ is a
    diagonal matrix with: 
    \begin{align}
        [\boldsymbol{\Lambda}_{Q}]_{i, i}=
        \left\{\begin{aligned} 
        & \lambda_i, \text { the corresponding eigenvalue to } \mathbf{u}_i, \text { or }, \\
        & \sigma^2. 
        \end{aligned}\right.
        \label{eq:dppca_lambda}
    \end{align} 
\end{proof}

\subsection{Proof of Proposition \ref{theo:sigma}}
\label{app:sigma}


\subsubsection{Auxiliary Theorem}

Before delving into the proof, we first proceed to characterize the stationary point of $\sigma^2$ in the linear GPLVM, which is summarized in the following theorem. 
\begin{theorem}
    \label{theo:dpca_stationary}
    Given $\hat{\vx}$, stationary points of the projection variance, denoted as $\hat{\sigma}^2$, could be obtained by solving the following optimization problem  
    \begin{align}
        \max_{\sigma^2} \log P(\vy \mid \hat{\vx}). 
    \end{align} 
    It turns out that $\hat{\sigma}^2$ takes the following form:
    \begin{align}
        \hat{\sigma}^2 &= \frac{1}{N-Q^{\prime}} \sum_{j=Q^{\prime}+1}^{N} \lambda_j,  
        \label{eq:ppca_middle_sigma}
    \end{align}
    where $Q^\prime$ is the number of eigenvalues retained in $\bm{\Lambda}_Q$. 
    \label{eq: ppca_stationary_sigma}
\end{theorem}  \vspace{-.1in}
\begin{proof}
    To obtain the stationary point for $\sigma^2$, we substitute the stationary point for $\vx$, defined in Eq.~\eqref{eq:ppca_stationary_X}, into the log marginal likelihood function Eq.~\eqref{eq: ppca_ll} to give
    \begin{align}
        L = - \frac{M}{2} \left\{ 
        N \log 2 \pi + \sum_{j=1}^{Q^{\prime}} \log( \lambda_j )  + ( N - Q^{\prime}) \ln \sigma^{2} + \frac{1}{ \sigma^2 } \sum_{j=Q^{\prime}+1}^{N} \lambda_j + Q^{\prime} 
        \right\}, 
        \label{eq:ll_sigma_ppca}
    \end{align}
    where $Q^{\prime}$ represents the number of $[\boldsymbol{\Lambda}_{Q}]_{i, i}, i \in 1, ..., Q$ that are not equal to $\sigma^2$, see Eq.~\eqref{eq:dppca_lambda}. Consequently, $\lambda_1, ..., \lambda_{Q^{\prime}}$ denote the eigenvalues associated with the eigenvectors ``retained" in $\vx$, while $\lambda_{Q^{\prime}+1}, ..., \lambda_{N}$ refer to the eigenvalues that are ``discarded". 
    
    By taking the gradient of Eq.~\eqref{eq:ll_sigma_ppca} with respect to $\sigma^2$ and setting it to zero, we obtain:
    $$ 
       \hat{\sigma}^2 = \frac{1}{N-Q^{\prime}} \sum_{ j=Q^{\prime}+1 }^{N} \lambda_j. 
    $$
\end{proof}

\begin{remark}
    Note that the eigenvalues $\{\lambda_{Q^{\prime}+1}, \ldots \lambda_{N}\}$ can be interpreted as the discarded/lost information in the inverse projection process ($\vy \rightarrow \vx$), and the corresponding eigenvectors are treated as \textbf{discarded vectors}. 
\end{remark}

In addition, with Theorem \ref{theo:dpca_stationary}, we can immediately get the following corollary.
 \begin{corollary} \label{corollary:them_stationary_point_dppca}
If $\boldsymbol{\Lambda}_{Q}$ contains the first $Q$ principal eigenvalues of $\frac{1}{M} \vy \vy^{\top}$, then the corresponding stationary point becomes the global maximum, which could be represented as:
    \begin{subequations}
        \begin{align}
            \sigma^{2\star} & =\frac{1}{N-Q} \sum_{j=Q+1}^N \lambda_j^o,  \label{eq:optimum_sigma}\\
            \mathbf{X}^{\star} & =\mathbf{U}_Q^{\star} \left(\boldsymbol{\Lambda}_Q^{\star} -(\sigma^2)^{\star} \mathbf{I}_Q \right)^{1 / 2} \mathbf{R}, \label{eq:optimum_X}
        \end{align}
    \label{eq: opt_ppca}
    \end{subequations}
    where $ \left[ \lambda_1^o, ..., \lambda_N^o \right] $ representing the eigenvalues of $\frac{1}{M} \vy \vy^{\top}$ with $\lambda_1^o \geq \lambda_2^o, ..., \geq \lambda_N^o$. Additionally,  $\mathbf{U}_Q^{\star} \in \mathbb{R}^{N \times Q}$ are the first $Q$ principal eigenvectors of $\frac{1}{M} \vy \vy^{\top}$, with the associated eigenvalues $\boldsymbol \Lambda_Q^{\star} = \operatorname{diag}(\lambda_1^o, \lambda_2^o, \ldots, \lambda_Q^o)$. The optimal projection variance, $\sigma^{2\star}$, represents the average variance lost in the projection process. 
\end{corollary}  \vspace{-.1in}
\begin{proof}
    With the stationary point of $\sigma^2$ given in Eq.\eqref{eq:ppca_middle_sigma}, the log marginal likelihood, given in Eq.~\eqref{eq: ppca_ll}, becomes 
    \begin{align}
        L = -\frac{M}{2} \left\{ \sum_{j=1}^{Q^{\prime}} \log(\lambda_j) + (N- Q^{\prime}) \log \left( \frac{1}{N-Q^{\prime}} \sum_{ j=Q^{\prime}+1 }^{N} \lambda_j \right) + N \log (2 \pi) + N \right\}. 
        \label{eq:ppca_ll_opt_sigma}
    \end{align}
    Because of the constancy of the sum of all eigenvalues $\lambda_j$ (given the data $\vy$), maximizing Eq.~\eqref{eq:ppca_ll_opt_sigma} is equivalently to minimize the following quantity
    \begin{align}
        E = \log \left( \frac{1}{N-Q^{\prime}} \sum_{ i=Q^{\prime}+1 }^{N} \lambda_i \right) - \frac{1}{N-Q^{\prime}} \sum_{ i=Q^{\prime}+1 }^{N} \log(\lambda_i), 
        \label{eq:ppca_E}
    \end{align}
   which solely relies on the discarded eigenvalues and remains non-negative (indeed due to Jensen's inequality). Remarkably, the minimization of $E$ necessitates only that the discarded $\lambda_j$ values are contiguous within the spectrum of the ordered eigenvalues of matrix $\frac{1}{M} \vy \vy^{\top} $. However, in addition to this, Eq.~\eqref{eq:opt_x_ppca} imposes the condition that $\lambda_j > \sigma^2$ for all $i$ in the set $\{1, 2, \ldots, Q^{\prime}\}$. Consequently, based on Eq.~\eqref{eq:ppca_middle_sigma}, it can be inferred that the smallest eigenvalue must be among the discarded ones. This deduction is sufficient to establish that $E$ is minimized when $\lambda_{Q^{\prime}+1}, \ldots, \lambda_N$ represent the smallest $N-Q^{\prime}$ eigenvalues. As a result, the likelihood $L$ is maximized when $\lambda_1, \ldots, \lambda_{Q^{\prime}}$ are the largest eigenvalues of matrix $\frac{1}{M} \vy \vy^{\top}$.
    It is worth noting that the maximization of $L$ concerning $Q^{\prime}$ is achieved when there are the fewest terms in the sums outlined in Eq.~\eqref{eq:ppca_E}. This occurs when $Q^{\prime}=Q$, ensuring that none of the $l_i, i \in 1, ..., Q$ terms are zero. 
\end{proof}



\subsubsection{Proof of Proposition \ref{theo:sigma}}

1) \textbf{Outline of Proof.}

Without loss of generality, we assume the rotation matrix $\mathbf{R} = \mathbf{I}_Q$ in Eq.~\eqref{eq:ppca_stationary_X}, Theorem \ref{theo:dpca_stationary_main_text}, resulting in the stationary points of the latent variable 
\begin{equation}
    \label{eq:x_hat_without_R}
    \hat{\mathbf{X}} = \mathbf{U}_Q \left(\boldsymbol{\Lambda}_Q - \hat{\sigma}^2 \mathbf{I}_{Q}\right)^{1/2}.
\end{equation}

Based upon this form, we seek to explore the structure of the optimization landscape around $\hat{\vx}$ by examining the variation trend of the log marginal likelihood $L$ at $\hat{\vx}$ in the spanned space of discarded vectors, denoted as $\operatorname{Span}(\vu_{D})$, where $$\vu_{D} \triangleq \left[ \u_{Q^{\prime}+1}, ..., \u_{N} \right].$$  

Intuitively, if the evaluation of $L$ at a stationary point consistently decreases for all axes in $\operatorname{Span}(\vu_{D})$, then we can consider the corresponding stationary point as a \textit{\textbf{local optimum}} or \textit{\textbf{global optimum}}, and vice versa; If the evaluation of $L$ at a stationary point consistently increases along any axis and decreases along any others within $\operatorname{Span}(\vu_{D})$, the corresponding stationary point can be recognized as a \textit{\textbf{saddle point}}.  


2) \textbf{Quantitative Analysis.}

To quantitatively analyze the variation in $L$ at $\hat{\vx}$ within $\operatorname{Span}(\vu_{D})$, we introduce a small perturbation to the $i$-th column of $\hat{\vx}$ in the form of $\epsilon \mathbf{u}_j$, resulting in the perturbed stationary point $\hat{\vx}^{\epsilon}$ with 
\begin{align}
    [ \hat{\vx}^{\epsilon} ]_{:, i} = \hat{\x}_i + \epsilon \mathbf{u}_j,   \quad i = 1,2,\ldots, Q,
\end{align} 
where $\hat{\x}_i$ denotes the $i$-th column of $\hat{\vx}^{\epsilon}$, with a bit abuse of notation, and \(\epsilon\) is an arbitrarily small positive constant and \(\mathbf{u}_j, j \in Q^{\prime}+1,..., N \) represents a principal axis in $\operatorname{Span}(\vu_{D})$. The variation trends from $L(\hat{\vx})$ to $L(\hat{\vx}^{\epsilon})$ can be determined by examining the sign of the dot product of the perturbation $\u_j$ with the gradient at $\hat{\x}_i+\epsilon \mathbf{u}_j$. More precisely, when the sign is positive, the evaluation of $L$ at $\hat{\vx}$ will ascend as $\hat{\x}_i$ shifts towards the direction of $\mathbf{u}_j$, and vice versa.  For clarity, let us denote the sign of the dot product as $\operatorname{sgn}(D_{ij})$, where $D_{ij}$ denotes the dot product and is expressed as  
\begin{align}
    D_{ij} &=  \mathbf{u}_j^{\top}  \left\{ \vk^{-1} \vy \vy^{\top} \vk^{-1} \left( \hat{\x}_i + \epsilon \u_j \right)  - M \vk^{-1} \left( \hat{\x}_i + \epsilon \u_j \right) \right\},
    \label{eq:dot_pro_1}
\end{align}
with $\vk = \hat{\vx}^{\epsilon} \hat{\vx}^{\epsilon^\top} + \hat{\sigma}^2 \mathbf{I}_N$.

According to Lemma \ref{Lemma:Woodbury_matrix_identity} and Eq.~\eqref{eq:x_hat_without_R}, we have 
\begin{align}
        \begin{aligned}
            \mathbf{K}^{-1} \hat{\vx}^{\epsilon} &= \hat{\vx}^{\epsilon}  \left[ (\hat{\vx}^{\epsilon})^{\top} \hat{\vx}^{\epsilon} + \hat{\sigma}^2 \mathbf{I}_{Q} \right]^{-1}, \\ 
            &= \hat{\vx}^{\epsilon} \left[ \mathbf \Lambda^{\epsilon}_Q \right]^{-1}, 
            \label{eq:push_through_2}
        \end{aligned}
\end{align}
where $\mathbf \Lambda^{\epsilon}_Q$ is a diagonal matrix with:
\begin{align}
        [\boldsymbol{\Lambda}_{Q}^{\epsilon}]_{k, k}=
        \left\{\begin{aligned} 
        &  [\boldsymbol{\Lambda}_{Q}]_{k, k}, & \forall k \neq i, \\
        &  [\boldsymbol{\Lambda}_{Q}]_{i, i} + \epsilon^2, & \text{ otherwise}.
        \end{aligned}\right.
\end{align}  

Checking the $i$-th column of the matrices on both sides of  Eq.~\eqref{eq:push_through_2}, we find that 
\begin{align}
    \vk^{-1} \left( \hat{\x}_i + \epsilon \u_j \right) = \frac{  \hat{\x}_i + \epsilon \u_j   }{ [\boldsymbol{\Lambda}_{Q}]_{i, i} + \epsilon^2}. 
\end{align}
Substituting $\vk^{-1} \left( \hat{\x}_i + \epsilon \u_j \right) = \frac{\left( \hat{\x}_i + \epsilon \u_j \right)  }{ [\boldsymbol{\Lambda}_{Q}]_{i, i} + \epsilon^2}$ into the first term of Eq.~\eqref{eq:dot_pro_1} yields
\begin{align}
    D_{ij} &=  M \mathbf{u}_j^{\top} \vk^{-1} \left\{ \frac{1}{M} \vy \vy^{\top} \frac{\hat{\x}_i+\epsilon \u_j}{ [\boldsymbol{\Lambda}_{Q}]_{i, i} + \epsilon^2 } - (\hat{\x}_i + \epsilon \u_j)  \right\} , \nonumber \\
    &=  M \mathbf{u}_j^{\top} \vk^{-1} \left\{ \frac{1}{M} \vy \vy^{\top} \frac{1}{ [\boldsymbol{\Lambda}_{Q}]_{i, i} + \epsilon^2 } - \mathbf{I}_N  \right\} \hat{\x}_i +  M \mathbf{u}_j^{\top} \vk^{-1} \left\{ \frac{1}{M} \vy \vy^{\top} \frac{1}{ [\boldsymbol{\Lambda}_{Q}]_{i, i} + \epsilon^2 } - \mathbf{I}_N  \right\} \epsilon \u_j, 
    \label{eq:D_ij_middle}\\
    & \approx M \mathbf{u}_j^{\top} \vk^{-1} \left\{ \frac{1}{M} \vy \vy^{\top} \frac{1}{ [\boldsymbol{\Lambda}_{Q}]_{i, i} } - \mathbf{I}_N  \right\} \hat{\x}_i +  \epsilon  M \mathbf{u}_j^{\top} \vk^{-1} \left\{ \frac{1}{M} \vy \vy^{\top} \frac{1}{ [\boldsymbol{\Lambda}_{Q}]_{i, i} } - \mathbf{I}_N  \right\} \u_j. \label{eq:D_ij_middle_2}
\end{align}
According to Eq.~\eqref{eq:ppca_fixed_point}, we have 
\begin{align}
    \label{eq:D_ij_middle_3}
    \frac{1}{M} \vy \vy^{\top} \frac{\hat{\x}_i}{ [\boldsymbol{\Lambda}_{Q}]_{i, i}} = \hat{\x}_i. 
\end{align}
Therefore, Eq.~\eqref{eq:D_ij_middle_2} can be rewritten as 
\begin{align}
\label{eq:final_D_ij}
    D_{ij} = \epsilon M\left( \frac{\lambda_j}{[\boldsymbol{\Lambda}_{Q}]_{i, i}}-1\right) \mathbf{u}_j^\top \mathbf{K}^{-1}  \mathbf{u}_j, 
\end{align} 
where $\lambda_j$ represents the eigenvalues corresponding to $\u_j$.

Due to the positive definite property of $\mathbf{K}^{-1}$, $\operatorname{sgn}(D_{ij})$, see Eq.~\eqref{eq:final_D_ij}, relies solely on  
\begin{equation}
    \operatorname{sgn} \left( \frac{\lambda_j}{[\boldsymbol{\Lambda}_{Q}]_{i, i}}-1 \right),
\end{equation}
implying that the type of stationary points is dictated by the discarded and retained eigenvalues. Specifically, 
\begin{itemize}
    \item[(i)] For \(\hat{\mathbf{x}}_i, \forall i = 1, ..., Q \), if $[\boldsymbol{\Lambda}_{Q}]_{i, i} > \lambda_j, \forall j \in Q^{\prime}+1, ..., N$, then the corresponding stationary point should be recognized as a \textit{\textbf{local or global optimum point}};
    \item[(ii)] For \(\hat{\mathbf{x}}_i, \forall i = 1, ..., Q \), if $[\boldsymbol{\Lambda}_{Q}]_{i, i} < \lambda_j, \forall j \in Q^{\prime}+1, ..., N $, then the corresponding stationary point should be recognized as a \textit{\textbf{local minimum point}};
    \item[(iii)] For \(\hat{\mathbf{x}}_i, i = 1, ..., Q \), if $[\boldsymbol{\Lambda}_{Q}]_{i, i} > \lambda_j, [\boldsymbol{\Lambda}_{Q}]_{i, i} < \lambda_k, \  \exists j, k \in Q^{\prime}+1, ..., N,$ then the corresponding stationary point should be identified as a \textit{\textbf{saddle point}}. 
\end{itemize}

3) \textbf{Final Results.} 


If $[\boldsymbol{\Lambda}_{Q}]_{i, i}=\lambda_i, \forall i \in 1, ..., Q$, the stationary point represents a global optimum when $$\lambda_i>\lambda_j, \forall i \in 1, ..., Q, \text{ and } \forall j \in  Q^{\prime}+1, ..., N.$$ However, if there exists a $\lambda_i < \lambda_j$, these stationary points correspond to saddle points. 
Additionally, when $\exists i \in 1, ..., Q, [\boldsymbol{\Lambda}_{Q}]_{i, i}=\hat{\sigma}^2$,  the associated stationary points are deemed saddle points due to the existence of cases where $$\hat{\sigma}^2 < \lambda_j, j \in Q^{\prime}+1, ..., N,$$ considering that $\hat{\sigma}^2$ is the average of the discarded eigenvalues.  Because the saddle points could be escaped efficiently, they are generally regarded as unstable stationary points \citep{jin2017escape}. Therefore, during the optimization process, when we set $\sigma^2 = \hat{\sigma}^2$, the only stable maximum point is the global optimum point. 

\begin{remark}
    The analysis does not account for the equality of eigenvalues. This is because: (1) Equality among the first $Q$ principal eigenvalues does not influence the presented analysis; (2) The equality of all discarded eigenvalues is trivial. 
\end{remark}

\subsection{Proof of Proposition \ref{lemma:sigma_big}}
\label{app:sigma_big}

\begin{proof}
    Suppose the projection variance $\sigma^2$ takes a value within the range $ ( \lambda_{Q}^{o}, \lambda_{Q-1}^{o})$, where $\lambda_Q^{o}$ and $\lambda_{Q-1}^{o}$ represent the $Q$-th and $(Q\!-\!1)$-th principal eigenvalues of $\frac{1}{M}\vy \vy^{\top}$, respectively. In this scenario, the eigenvectors with associated eigenvalues less than $\lambda_Q^{o}$ are unambiguously discarded. Furthermore, in such a case, the only stable local optimum point\footnote{Other stationary points, manifested as saddle points, are unstable as discussed in App.~\ref{app:sigma}.} comprises the following $\boldsymbol{\Lambda}_{Q}$, 
\begin{equation}
    [\boldsymbol{\Lambda}_{Q}]_{i, i}=
    \left\{
        \begin{aligned}
            & \lambda_i^{o}, \text{ for } i \in 1, ..., Q-1, \text{ or }, \\
            & \sigma^2.  
        \end{aligned}
    \right.
    \nonumber
\end{equation}
It is evident that, for either $[\boldsymbol{\Lambda}_{Q}]_{i, i} = \sigma^2$ or $\lambda_i^{o}$, $[\boldsymbol{\Lambda}_{Q}]_{i, i} > \lambda_j^{o}$ for all $i \in 1, ..., Q$ and for all $j \in Q-1, ..., N$, leading to the corresponding stationary points being the local optimum point, with one zero-column in $\hat{\vx}$.

If the projection variance falls within the range $ (\lambda_{Q-1}^{o}, \lambda_{Q-2}^{o})$, the only stable local optimum point comprises the following $\boldsymbol{\Lambda}_{Q}$, 
\begin{equation}
    [\boldsymbol{\Lambda}_{Q}]_{i, i}=
    \left\{
        \begin{aligned}
            & \lambda_i^{o}, \text{ for } i \in 1, ..., Q-2, \text{ or }, \\
            & \sigma^2, 
        \end{aligned}
    \right.
    \nonumber
\end{equation} with two zero-columns in $\vx$. By deduction, when $\sigma^2 > \lambda_{1}^{o}$, the only stable local optimum point comprises the $\boldsymbol{\Lambda}_{Q}$ with $[\boldsymbol{\Lambda}_{Q}]_{i, i} = \sigma^2$ for all $i \in 1, ..., Q $  with $\vx = \mathbf{0}$.
    
It is worth noting that we deliberately avoid considering the equality of any of the $Q$ principal eigenvalues to streamline the quantitative analysis, as introducing such equality might exacerbate the complexity and hasten the occurrence of model collapse. For instance, when the projection variance falls within the range $ ( \lambda_{Q}^{o}, \lambda_{Q-1}^{o})$ and there exist two eigenvectors with eigenvalues equal to $\lambda_Q^{o}$, the stable local optimum point entails two zero-columns. 

Suppose $\sigma^2 < \lambda_N$, then, there exist a set of local minima point characterized by the following $\boldsymbol \Lambda_{Q}$, 
    \begin{equation}
        [\boldsymbol{\Lambda}_{Q}]_{i, i}=
        \left\{
            \begin{aligned}
                & \lambda_i^{o}, \text{ for } i \in N-Q+k, ..., N, \text{ or }, \\
                & \sigma^2,   
            \end{aligned}
        \right.
        \nonumber
    \end{equation}
where $k$ represents the last $k$ principal eigenvalues that are selected. It is also noteworthy that these local minima point\footnote{Equality among the last $Q$ principal eigenvalues does not impact the analysis presented.} will feature $k$ zero-columns in $\vx$.  
    
\end{proof}


\section{Modeling and Variational Approximation}
\label{appendix:modeling_and_variational_approximation}

\subsection{ELBO Derivation and Evaluation}
\label{appendix:ELBO_deriviations}
\begin{equation}
    \begin{aligned}
            \mathcal{L} & = \mathbb{E}_{q(\vx, \mathbf{W})} \left[ \frac{p(\vy, \vx, \mathbf{W})}{q(\vx, \mathbf{W})}\right] \\
            & = \mathbb{E}_{q(\vx, \mathbf{W})} \left[\log \frac{p(\mathbf{W}) \prod_{i = 1}^N p(\x_i) \prod_{j=1}^{M} p(\y_{:, j} \vert \vx, \mathbf{W})}{p(\mathbf{W}) \prod_{i = 1}^N q(\x_i) } \right] \\
            & = \underbrace{\sum_{j=1}^M \mathbb{E}_{q(\vx, \mathbf{W})} \left[ \log p(\y_{:, j} \vert \vx, \mathbf{W}) \right] }_{\text{Term 1: data reconstruction}} \underbrace{- \sum_{i=1}^N\operatorname{KL}(q(\x_i) \| p(\x_i))}_{\text{Term 2: regularization}} \\
            & \approx \sum_{j=1}^M  \frac{1}{{I}}\sum_{i=1}^{I} \log \mathcal{N}(\y_{:, j} \vert \bm{0}, \hat{\vk}_{\mathrm{sm}}^{(i)} + \sigma^2 \mathbf{I}_N) -  \frac{1}{2} \sum_{i=1}^{N} \Big[ \operatorname{tr}(\mathbf{S}_{i}) + \boldsymbol{\mu}_i^{\top} \boldsymbol{\mu}_i -\log |\mathbf{S}_{i}| - Q  \Big] \\
            & \approx \sum_{j=1}^M  \frac{1}{{I}}\sum_{i=1}^{I}  \left\{ -\frac{N}{2} \log 2 \pi - \frac{1}{2} \log \left|    \hat{\vk}_{\text{sm}}^{(i)} + \sigma^2 \mathbf{I}_N  \right| - \frac{1}{2} \y_{:, j}^{\top}  \left(  \hat{\vk}_{\text{sm}}^{(i)} + \sigma^2 \mathbf{I}_N \right)^{-1} \y_{:, j} \right\}  \! \\ 
            & ~~~ - \!  \frac{1}{2} \sum_{i=1}^{N} \Big[ \operatorname{tr}(\mathbf{S}_{i}) + \boldsymbol{\mu}_i^{\top} \boldsymbol{\mu}_i -\log |\mathbf{S}_{i}| - Q  \Big]
    \end{aligned}
    \nonumber
\end{equation}
where $\mathbf{S}_i$ is typically assumed to be a diagonal matrix. 
Note that $\hat{\vk}_{\text{sm}}^{(i)} = \Phi_{\text{sm}}(\vx^{(i)}; \vw) \Phi_{\text{sm}}(\vx^{(i)}; \vw)^\top$, where $\Phi_{\text{sm}}(\vx^{(i)}; \vw)\in \mathbb{R}^{N \times mL}$.

\vspace{.2in}

\begin{lemma}
Suppose $\mathbf{A}$ is an invertible $n$-by-$n$ matrix and $\mathbf{U}, \mathbf{V}$ are $n$-by-$m$ matrices. Then the following determinant equality holds.
$$
\left|\mathbf{A}+\mathbf{U} \mathbf{V}^{\top}\right|=\left|\mathbf{I}_{\mathrm{m}}+\mathbf{V}^{\top} \mathbf{A}^{-1} \mathbf{U}\right| \left|\mathbf{A}\right|
$$
\end{lemma}
\begin{lemma}[Woodbury matrix identity]
\label{Lemma:Woodbury_matrix_identity}
    Suppose $\mathbf{A}$ is an invertible $n$-by-$n$ matrix and $\mathbf{U}, \mathbf{V}$ are $n$-by-$m$ matrices. Then
$$
\left(\mathbf{A}+\mathbf{U} \mathbf{V}^{\top}\right)^{-1}=\mathbf{A}^{-1} - \mathbf{A}^{-1}\mathbf{U} (\mathbf{I}_{\mathrm{m}} + \mathbf{V}^\top \mathbf{U})^{-1} \mathbf{V}^\top
$$
\end{lemma}
According to the above two lemmas \cite{williams2006gaussian},  in the case that $N \gg mL$, we can compute the determinant and inversion of $\hat{\vk}_{\text{sm}}^{(i)} + \sigma^2 \mathbf{I}_N$, reducing the computational complexity of the ELBO evaluation from the original $\mathcal{O}(N^3)$ to $\mathcal{O}(N(mL)^2)$.
\begin{align}
    & \left| \hat{\vk}_{\text{sm}}^{(i)} + \sigma^2 \mathbf{I}_N \right| = \left|\mathbf{I}_{mL} +  \frac{1}{\sigma^2} \Phi_{\text{sm}}^\top\Phi_{\text{sm}}  \right| \left| \sigma^2 \mathbf{I}_N \right|  = \sigma^{2N} \left|\mathbf{I}_{mL} +  \frac{1}{\sigma^2} \Phi_{\text{sm}}^\top\Phi_{\text{sm}}  \right|, \\
    & \left( \hat{\vk}_{\text{sm}}^{(i)} + \sigma^2 \mathbf{I}_N \right)^{-1} = \frac{1}{\sigma^2} \left[ \mathbf{I}_N - \Phi_{\text{sm}}(\mathbf{I}_{mL} + \Phi_{\text{sm}}^\top \Phi_{\text{sm}})^{-1} \Phi_{\text{sm}}^\top \right].
\end{align}

\subsection{Interpretation of Modeling and Variational Distribution}
\label{appendix:interpretation_eq10_11_12}
In Eqs.~\eqref{eq:GPLVM_RFF_SM} and \eqref{eq:model_joint}, we have modeled the spectral points $\vw$ as a part of the data generation process. However, this might cause some confusion, which is clarified as follows.
\begin{itemize}
    \item If we have selected the kernel function, the probability model for the data, i.e., Eqs.~\eqref{eq:GPLVM_RFF_SM} and \eqref{eq:model_joint}, can be interpreted as independent of $p(\vw)$, as it is inherent to the kernel function.
    \item In this paper, we provide another interpretation perspective: Following the setting from RFLVM by \citet{gundersen2021latent},  we consider the data-generating process for observations $\vy$ as outlined in Eq.~\eqref{eq:GPLVM_RFF_SM} or \eqref{eq:model_joint}, which is dependent on $\vw$. Subsequently, we constrain its prior $p(\vw)$ to be Gaussian mixtures, defining the prior SM kernels functions. This alternative perspective is explained as follows:
    \begin{itemize}
        \item  Let us explicitly assume a parametric variational distribution $q_{\bm{\eta}}(\vw)$, assuming it to be another Gaussian mixture (thus still defines an SM kernel) with parameters denoted as $\bm{\eta}$ to approximate $p(\vw|\vy)$. In this case, Eq.~\eqref{eq:variational_dis} becomes: $$q(\vx,\vw) = q_{\bm{\eta}}(\vw)q(\vx).$$ Combining the joint distribution in Eq.~ \eqref{eq:model_joint}, we derive the following ELBO: 
        \begin{equation}
            \begin{aligned}
                \mathcal L & = {E}_{q(\vx, \vw)} \left[ \log \frac{p(\vx)p_{\btheta}(\vw)p_{\sigma}(\vy\vert \vx,\vw)}{q(\vx)q_{\bm{\eta}}(\vw)}\right]\\
                & = {E}_{q(\vx, \vw)} \left[ \log p_{\sigma}(\vy\vert \vx, \vw)\right] - KL(q(\vx) \| p(\vx)) - KL(q_{\bm{\eta}}(\vw) \| p_{\btheta}(\vw)).
            \end{aligned}
        \end{equation}
        In this ELBO, prior distribution $p_{\btheta}(\vw)$ is only related to the last KL divergence term. When maximizing the ELBO, we will obtain that $\btheta=\bm{\eta}$, ensuring that the last KL divergence term becomes $0$. Ultimately, this aligns with the optimization objective in our paper. 
    \end{itemize}
    \item  \textbf{More complicated $p(\vw \vert \vy)$ approximations}.  It is possible to consider assuming the variational distribution of spectral points is $\vy$-dependent $q(\vw \vert \vy)$, such as a parametric Gaussian mixture and other distributions. 
    \begin{itemize}
        \item \textbf{Gaussian mixture}: Suppose we use a parametric variational distribution $q_{\bm{\eta}}(\vw \vert \vy)$, in the form of 
        \begin{equation}
            q_{\bm{\eta}}(\vw \vert \vy) =\prod_{l=1}^{L/2} \sum_{i=1}^{m} \alpha_i \mathcal{N}_{\bm{\eta}_i}( \mu_i , \sigma_i^{2} ),
        \end{equation}
         where  $(\alpha_i, \mu_i, \sigma_i^2)$ in each mixture component is modeled by an encoder parametrized by $\bm{\eta}_i$ with $\vy$ as input.  Similarly, we can get the ELBO: 
         \begin{equation}
             \begin{aligned}
                 \mathcal L &= {E}_{q(\vx, \vw | \vy)} \left[ \log \frac{p(\vx)p_{\btheta}(\vw)p_{\sigma}(\vy\vert \vx, \vw)}{q(\vx)q_{\bm{\eta}}(\vw \vert \vy)} \right]\\
                 & = {E}_{q(\vx, \vw | \vy)} \left[ \log p_{\sigma}(\vy\vert \vx, \vw)\right] - KL(q(\vx) \| p(\vx)) - KL(q_{\bm{\eta}}(\vw \vert \vy) \| p_{\btheta}(\vw)). 
             \end{aligned}
         \end{equation}
    When maximizing the ELBO, the last KL divergence term will also be $0$. The difference between the remaining terms and our objective function lies in the first term, where $\vw$ includes information learnt from $\vy$. This potentially enhances the kernel selection process and contributes to preventing model collapse, though coming at the cost of increased computational complexity from the encoder evaluation. 
      \item \textbf{Other distribution forms}: In this case, the variational inference algorithm heavily depends on the specific form of $q(\vw \vert \vy)$. While this variational distribution can be more general, such an assumption generally introduces greater intractability, making the evaluation of the ELBO more challenging. Employing Monte Carlo sampling to approximate the ELBO in such scenarios could result in larger approximation variances compared to the case where $q(\vw) = p(\vw)$, thus potentially leading to less robust model performance.
    \end{itemize}
\end{itemize}  

\vspace{.2in}

\section{Auto-differentiable SM Kernel using RFF Approximation}

\subsection{Proof of Proposition \ref{prop:SM_RFF_approx}}
\label{appendix:prop_SM_RFF_approx}

\begin{proof}
With the RFF feature map defined in Eq.~\eqref{eq:SM_approximations_RFF_new}, we can write down the inner product of the feature maps
    \begin{equation}
        \begin{aligned}
                \phi\left( \x; \vw \right)^\top \phi \left( \x^\prime; \vw \right) =  \sum_{i=1}^m \alpha_i \sum_{l = 1}^{\frac{L}{2}} \frac{2}{L} \cos(2 \pi \mathbf{w}_{l}^{(i)\top } (\x - \x^\prime))
        \end{aligned}
    \end{equation}
    where $\vw \triangleq \{\mathbf{w}_{1}^{(i)}, \mathbf{w}_{2}^{(i)}, \ldots, \mathbf{w}_{{L}/{2}}^{(i)}\}_{i = 1}^m$, and each $\w_l^{(i)}$ are i.i.d. sampled from the symmetric distribution 
    $$
        s_i(\w) = \frac{\mathcal{N}(\mathbf{w} \vert \bm{\mu}_i, \operatorname{diag}(\bm{\sigma}_i^2)) + \mathcal{N}(-\mathbf{w} \vert \bm{\mu}_i, \operatorname{diag}(\bm{\sigma}_i^2))}{2}
    $$
    using reparameterization trick \cite{kingma2019introduction}. Taking the expectation w.r.t. $p\left(\vw\right) = \prod_{i=1}^m \prod_{l=1}^{L/2} s_i(\w)$,  we can get
    \begin{subequations}
        \label{eq:proof_RFF_SM}
        \begin{align}
                &\mathbb{E}_{p\left(\vw\right)}  \left[\phi\left( \x; \vw\right)^\top \phi\left( \x^\prime; \vw\right)  \right] 
                =  \mathbb{E}_{p\left(\vw\right)} \left[ \sum_{i=1}^m \alpha_i \sum_{l = 1}^{L/2} \frac{2}{L} \cos(2 \pi\mathbf{w}_{l}^{(i)\top } (\x - \x^\prime)) \right] &  \nonumber \\
                & =   \sum_{i=1}^m \alpha_i \mathbb{E}_{p\left(\mathbf{w}_{1:L/2}^{(i)}\right)} \left[ \sum_{l = 1}^{L/2} \frac{2}{L} \cos(2 \pi\mathbf{w}_{l}^{(i)\top } (\x - \x^\prime))  \right] &   (\text{linearity of expectation})\\
                & =  \sum_{i=1}^m \alpha_i  \mathbb{E}_{s_i(\w)} \left[  \cos(2 \pi \mathbf{w}_{1}^{(i)\top } (\x - \x^\prime))  \right] &   (\text{i.i.d. of } \w^{(i)}_l) \\
                & =  \sum_{i=1}^m \alpha_i  \mathbb{E}_{s_i(\w)} \left[   \frac{\exp(2 \pi j \mathbf{w}_{1}^{(i)\top}(\x - \x^\prime)) + \exp(-2 \pi j \mathbf{w}_{1}^{(i)\top}(\x - \x^\prime) ) }{2} \right] & (\text{Euler’s identity}) \\
                & = \sum_{i=1}^m \alpha_i k_i(\x, \x^\prime;   \bm{\mu}_i, \bm{\sigma_i^2})  & (\text{symmetrity of } s_i(\w)) \\
                & = k_{\text{sm}}(\x, \x^\prime; \{ \alpha_i, \bm{\mu}_i, \bm{\sigma_i^2} \}_{i=1}^m) & (\text{SM kernel definition})
        \end{align}
    \end{subequations}
    Hence concludes that $\phi\left( \x; \vw \right)^\top \phi\left( \x; \vw \right)$ is an unbiased estimator of the SM kernel characterized by parameter $\{\alpha_i, \bm{\mu}_i, \bm{\sigma^2_i}\}_{i=1}^m$.
\end{proof}
\vspace{.2in}

\subsection{Proof of Theorem \ref{thm:SM_RFF_approx}}
\label{appendix:thm_SM_RFF_approx}
\begin{proof}
Similar theorem has been proven in the Gaussian process regression model; see {Proposition 3.1} in \cite{jung2022efficient}, and  {Theorem 3} in \cite{lopez2014randomized}.  For ease of reference, we follow the existing results and show the proof as follows.

$\bullet$~ To prove Theorem \ref{thm:SM_RFF_approx}, we first introduce the following Lemma for Matrix Bernstein inequality \cite{tropp2015introduction}.

    \begin{lemma}[Matrix Bernstein Inequality] \label{lemma:matrix_bernstein}
    Consider a finite sequence $\left\{\boldsymbol{X}_i\right\}$ of independent, random, Hermitian matrices with dimension $N$. Assume that
    $$
    \mathbb{E} [\boldsymbol{X}_i] =\mathbf{0} \text {  and  }  \left\|\boldsymbol{X}_i\right\|_2 \leq H \text {  for each index } i,
    $$ where $\|\cdot\|_2$ denotes the matrix spectral norm.  Introduce the random matrix
    $
    \boldsymbol{Y}=\sum_i \boldsymbol{X}_i,
    $ 
    and let $v(\boldsymbol{Y})$ be the matrix variance statistic of the sum:
    $$
    v(\boldsymbol{Y})=\left\|\mathbb{E} [\boldsymbol{Y}^2]\right\|=\left\|\sum_i \mathbb{E} [\boldsymbol{X}_i^2]\right\| .
    $$
    
    Then we have
    \begin{equation}
        \mathbb{E} \left[\|\boldsymbol{Y}\|_2 \right] \leq \sqrt{2 v(\boldsymbol{Y}) \log N}+\frac{1}{3} L \log N .
    \end{equation}
    
    Furthermore, for all $\epsilon \geq 0$.
    \begin{equation}
        {P}\left\{ \|\boldsymbol{Y}\|_2 \geq \epsilon \right\} \leq N \cdot \exp \left(\frac{-\epsilon^2 / 2}{v(\boldsymbol{Y})+H \epsilon / 3}\right).
    \end{equation}
    
    \begin{proof}
        The proof of Lemma \ref{lemma:matrix_bernstein} can be found in {Theorem 6.6.1}, {\S~6.6}, \cite{tropp2015introduction}.
    \end{proof}
    \end{lemma}
    \vspace{.1in}

{$\bullet$}~ Next, we show how to apply Lemma~\ref{lemma:matrix_bernstein} to prove Theorem.~\ref{thm:SM_RFF_approx}.

    \paragraph{1). Factorization of Approximation Error Matrix.}  
    
    With the constructed SM kernel matrix approximation, $\hat{\mathbf{K}}_{\mathrm{sm}} =  \Phi_{\mathrm{sm}}(\vx) \Phi_{\mathrm{sm}}(\vx)^\top$, where the random feature matrix $\Phi_{\mathrm{sm}}(\vx)\!=\!\left[\phi\left(\x_1\right),  \ldots , \phi\left(\x_N\right)\right]^\top \!\in \! \mathbb{R}^{N \times mL}$, we have the following approximation error matrix:
    \begin{equation}
       \mathbf{E} = \hat{\mathbf{K}}_{\mathrm{sm}} - {\mathbf{K}}_{\mathrm{sm}}.
    \end{equation}
    We are going to show that $\mathbf{E}$ can be factorized as 
    \begin{equation}
        \mathbf{E} = \sum_{i=1}^m \sum_{l=1}^{L/2} \mathbf{E}_{l}^{(i)}
    \end{equation}
    where $\mathbf{E}_{l}^{(i)}$ is a sequence of independent, random, Hermitian matrices with dimension $N$.
    
    Specifically, we define $\mathbf{Z}_{l}^{(i)}$ as
    \begin{equation}
        \mathbf{Z}_{l}^{(i)} = \left[\exp(2\pi j \w_l^{(i) \top}\x_1), \ldots,  \exp(2\pi j \w_l^{(i) \top}\x_N)\right]^\top \in \mathbb{R}^{N \times 1}, \text{  where } \w_l^{(i)} \sim s_i (\w),
    \end{equation}
    and we can show that 
    \begin{equation}
        \begin{aligned}
            \relax [\hat{\mathbf{K}}_{\mathrm{sm}}]_{h,g}  & = \sum_{i=1}^m \frac{2\alpha_i}{L}  \sum_{l=1}^{L/2} \cos(2\pi \w_l^{(i)\top} (\x_h -\x_g) ) \\
            &  = \sum_{i=1}^m   \sum_{l=1}^{L/2} \frac{2\alpha_i}{L} \operatorname{Re}\left(\exp(2\pi j \w_l^{(i)\top} (\x_h -\x_g) ) \right)\\
            & = \sum_{i=1}^m   \sum_{l=1}^{L/2} \frac{2\alpha_i}{L} \operatorname{Re}\left( \left[\mathbf{Z}_{l}^{(i)} \mathbf{Z}_{l}^{(i)*}\right]_{h,g} \right)
        \end{aligned}
    \end{equation}
     where $\mathbf{Z}_{l}^{(i)*}$ is the conjugate transpose of $\mathbf{Z}_{l}^{(i)}$. Thus, we have $\hat{\mathbf{K}}_{\mathrm{sm}} = \sum_{i=1}^m \sum_{l=1}^{L/2} \frac{2\alpha_i}{L} \operatorname{Re}(\mathbf{Z}_{l}^{(i)} \mathbf{Z}_{l}^{(i)*})$. Based on this factorization and Eq.~\eqref{eq:proof_RFF_SM} in Proposition \ref{prop:SM_RFF_approx}, we have that $${\mathbf{K}}_{\mathrm{sm}} = \sum_{i=1}^m \sum_{l=1}^{L/2} \frac{2\alpha_i}{L} \mathbb{E}[\operatorname{Re}(\mathbf{Z}_{l}^{(i)} \mathbf{Z}_{l}^{(i)*})].$$
    Therefore, the approximation error matrix $\mathbf{E}$ can be factorized as
    $
        \mathbf{E} = \sum_{i=1}^m \sum_{l=1}^{L/2} \mathbf{E}_{l}^{(i)} 
    $
    where 
    \begin{equation}
         \mathbf{E}_{l}^{(i)} = \frac{2\alpha_i}{L} \left( \operatorname{Re}(\mathbf{Z}_{l}^{(i)} \mathbf{Z}_{l}^{(i)*}) - \mathbb{E}[\operatorname{Re}(\mathbf{Z}_{l}^{(i)} \mathbf{Z}_{l}^{(i)*})]\right)
    \end{equation}
    is a sequence of independent, random, Hermitian matrices with dimension $N$ that satisfy the condition of $\mathbb{E} [ \mathbf{E}_{l}^{(i)}] = \mathbf{0}$.
    
    We next find the upper bound for $\|\mathbf{E}_{l}^{(i)}\|_2$.
    
    \paragraph{2). Upper Bound for $\|\mathbf{E}_{l}^{(i)}\|_2$.}  
        \begin{subequations}
        \label{eq:upper_bound_e_l_i}
            \begin{align}
            \|\mathbf{E}_{l, i}\|_2 & =\frac{2 \alpha_i}{L} \left\|\operatorname{Re}\left(\mathbf{Z}_l^{(i)}\mathbf{Z}_{l}^{(i)*}\right)-\mathbb{E} \left[\operatorname{Re}\left(\mathbf{Z}_l^{(i)}\mathbf{Z}_{l}^{(i)*}\right)\right]\right\|_2 \\
            & \leq \frac{2\alpha_i}{L}\left(\left\|\operatorname{Re}\left(\mathbf{Z}_l^{(i)}\mathbf{Z}_{l}^{(i)*}\right)\right\|_2+\left\|\mathbb{E}\left[\operatorname{Re}\left(\mathbf{Z}_l^{(i)}\mathbf{Z}_{l}^{(i)*}\right)\right]\right\|_2\right) \quad \quad \text{ (triangle inequality)}\\
            & \leq \frac{2\alpha_i}{L}\left(\left\|\operatorname{Re}\left(\mathbf{Z}_l^{(i)}\mathbf{Z}_{l}^{(i)*}\right)\right\|_2+ \mathbb{E} \left[ \left\| \operatorname{Re}\left(\mathbf{Z}_l^{(i)}\mathbf{Z}_{l}^{(i)*}\right) \right\|_2 \right] \right)   \quad \quad \text{  (Jensen’s inequality)}\\
            & \leq \frac{2a}{L}\left( 2N + 2N\right) \\
            & = \frac{2a }{L} 4 N  
            \end{align}
        \end{subequations}
where $a = \sqrt{\sum_{i=1}^m \alpha_i^2}$ and 
\begin{subequations}
    \begin{align}
        & \bm{c}_l^{(i)}  =\left[\cos \left(2 \pi \w_{l}^{(i)\top} \x_1\right), \ldots, \cos \left(2 \pi \w_{l}^{(i)\top} \x_N\right)\right]^\top \in \mathbb{R}^{N \times 1}, \\
        & \bm{s}_l^{(i)}  =\left[\sin \left(2 \pi \w_{l}^{(i)\top} \x_1\right) , \ldots, \sin \left(2 \pi \w_{l}^{(i)\top} \x_N\right)\right]^\top \in \mathbb{R}^{N \times 1},\\
        & \operatorname{Re}\left(\mathbf{Z}_l^{(i)}\mathbf{Z}_{l}^{(i)*}\right) = \bm{c}_l^{(i)} \bm{c}_l^{(i)\top} + \bm{s}_l^{(i)}\bm{s}_l^{(i)\top},
    \end{align}
\end{subequations}
and the last inequality in Eq.~\eqref{eq:upper_bound_e_l_i}, we use the fact that 
$$
    \left\|\operatorname{Re}\left(\mathbf{Z}_l^{(i)}\mathbf{Z}_l^{(i)*}\right)\right\|_2=\sup _{\|\bm{v}\|_2^2=1} \bm{v}^\top\left(\bm{c}_l^{(i)} \bm{c}_l^{(i)\top}+\bm{s}_l^{(i)} \bm{s}_l^{(i)\top}\right) \bm{v}
    \leq 2 N.
$$
Next, we are going to bound the variance, $\left\| \sum_{i=1}^m \sum_{l=1}^{L/2} \mathbb{E}[ (\mathbf{E}_{l}^{(i)})^2 ] \right\|_2$. 

\paragraph{3). Upper Bound for the Variance, $\left\| \sum_{i=1}^m \sum_{l=1}^{L/2} \mathbb{E}[ (\mathbf{E}_{l}^{(i)})^2 ] \right\|_2$.} 
We first have the following bound:
    \begin{subequations}
    \begin{align}
    \frac{L^2}{4\alpha_i^2} \mathbb{E}\left[\left(\mathbf{E}_l^{(i)}\right)^2\right] & =\mathbb{E}\left[\operatorname{Re}\left(\mathbf{Z}_{l}^{(i)} \mathbf{Z}_{l}^{(i)*}\right)^2\right]-\left(\mathbb{E}\left[\operatorname{Re}\left(\mathbf{Z}_{l}^{(i)} \mathbf{Z}_{l}^{(i)*} \right)\right]\right)^2 \\
    & \preccurlyeq \mathbb{E}\left[\operatorname{Re}\left(\mathbf{Z}_{l}^{(i)} \mathbf{Z}_{l}^{(i)*}\right)^2\right]  \label{eq:first_inqu_}\\
    & =\mathbb{E}\left[\left(\bm{c}_l^{(i)\top} \bm{c}_l^{(i)}\right) \bm{c}_l^{(i)} \bm{c}_l^{(i)\top}+\left(\bm{s}_l^{(i)\top} \bm{s}_l^{(i)}\right) \bm{s}_l^{(i)} \bm{s}_l^{(i)\top}+\left(\bm{s}_l^{(i)\top} \bm{c}_l^{(i)}\right)\left(\bm{s}_l^{(i)} \bm{c}_l^{(i)\top}+\bm{c}_l^{(i)} \bm{s}_l^{(i)\top}\right)\right] \\
    & \preccurlyeq N \mathbb{E}\left[\bm{c}_l^{(i)} \bm{c}_l^{(i)\top}+\bm{s}_l^{(i)} \bm{s}_l^{(i)\top}\right]+\mathbb{E}\left[\left(\bm{s}_l^{(i)\top} \bm{c}_l^{(i)}\right)\left(\bm{s}_l^{(i)} \bm{c}_l^{(i)\top}+\bm{c}_l^{(i)} \bm{s}_l^{(i)\top}\right)\right] \label{eq:second_inqu_} \\
    & =N \mathbb{E}\left[\operatorname{Re}\left(\mathbf{Z}_{l}^{(i)} \mathbf{Z}_{l}^{(i)*}\right)\right]+\mathbb{E}\left[\left(\bm{s}_l^{(i)\top} \bm{c}_l^{(i)}\right)\left(\bm{s}_l^{(i)} \bm{c}_l^{(i)\top}+\bm{c}_l^{(i)} \bm{s}_l^{(i)\top}\right)\right]
    \end{align}
    \end{subequations}
where the notation $\mathbf{A} \preccurlyeq \mathbf{B}$ denotes that $\mathbf{B} - \mathbf{A}$ is a positive semi definite (PSD) matrix, and the  inequality in Eq.~\eqref{eq:first_inqu_} holds due to the fact that  $\left(\mathbb{E}\left[\operatorname{Re}\left(\mathbf{Z}_{l}^{(i)} \mathbf{Z}_{l}^{(i)*} \right)\right]\right)^2$ is a PSD matrix. The inequality in Eq.~\eqref{eq:second_inqu_} holds because 
\begin{equation}
    \begin{aligned}
        & N \mathbb{E}\left[\bm{c}_l^{(i)} \bm{c}_l^{(i)\top}+\bm{s}_l^{(i)} \bm{s}_l^{(i)\top}\right] - \mathbb{E}\left[\left(\bm{c}_l^{(i)\top} \bm{c}_l^{(i)}\right) \bm{c}_l^{(i)} \bm{c}_l^{(i)\top}+\left(\bm{s}_l^{(i)\top} \bm{s}_l^{(i)}\right) \bm{s}_l^{(i)} \bm{s}_l^{(i)\top} \right] \\
        & = \mathbb{E}\left[\left(\bm{s}_l^{(i)\top} \bm{s}_l^{(i)}\right) \bm{c}_l^{(i)} \bm{c}_l^{(i)\top}+\left(\bm{c}_l^{(i)\top} \bm{c}_l^{(i)}\right) \bm{s}_l^{(i)} \bm{s}_l^{(i)\top} \right] \quad \ldots \quad \footnotesize{\left[\text{due to }  \left(\bm{c}_l^{(i)\top} \bm{c}_l^{(i)} + \bm{s}_l^{(i)\top} \bm{s}_l^{(i)}\right)=N \right]}
    \end{aligned}
\end{equation}
is a PSD matrix.

Then we are able to bound the variance, $\left\| \sum_{i=1}^m \sum_{l=1}^{L/2} \mathbb{E}[ (\mathbf{E}_{l}^{(i)})^2 ] \right\|_2$, as
\begin{equation}
\label{eq:variance_bound}
    \begin{aligned}
    & ~~~ \left\| \sum_{i=1}^m \sum_{l=1}^{L/2} \mathbb{E}[ (\mathbf{E}_{l}^{(i)})^2 ] \right\|_2 \\ 
    & \leq\left\|\sum_{i=1}^m \sum_{l=1}^{L/2} \frac{4 \alpha_i^2}{L^2}\left(N \mathbb{E}\left[\operatorname{Re}\left(\mathbf{Z}_l^{(i)} \mathbf{Z}_l^{(i)*}\right)\right]+\mathbb{E}\left[\left(\bm{s}_l^{(i) \top} \bm{c}_l^{(i)}\right)\left(\bm{s}_l^{(i)} \bm{c}_l^{(i) \top}+\bm{c}_l^{(i)} \bm{s}_l^{(i) \top}\right)\right]\right)\right\|_2 \\
    & \leq \frac{2a}{L}\left\|\sum_{i=1}^m \alpha_i\left(N \mathbb{E}\left[\operatorname{Re}\left(\mathbf{Z}_l^{(i)} \mathbf{Z}_l^{(i)*}\right)\right]+\mathbb{E}\left[\left(\bm{s}_l^{(i) \top} \bm{c}_l^{(i)}\right)\left(\bm{s}_l^{(i)} \bm{c}_l^{(i) \top}+\bm{c}_l^{(i)} \bm{s}_l^{(i) \top}\right)\right]\right)\right\|_2 \\
    & \leq \frac{2a}{L}\left(N\left\|\mathbf{K}_{\mathrm{sm}}\right\|_2+\sum_{i=1}^m \alpha_i\left\|\mathbb{E}\left[\left(\bm{s}_l^{(i) \top} \bm{c}_l^{(i)}\right)\left(\bm{s}_l^{(i)} \bm{c}_l^{(i) \top}+\bm{c}_l^{(i)} \bm{s}_l^{(i) \top}\right)\right]\right\|_2\right)   \quad \text{ (triangle inequality)}\\
    & \leq \frac{2a}{L}\left(N\left\|\mathbf{K}_{\mathrm{sm}}\right\|_2+\sum_{i=1}^m \alpha_i \mathbb{E}\left[\left\|\left(\bm{s}_l^{(i) \top} \bm{c}_l^{(i)}\right)\left(\bm{s}_l^{(i)} \bm{c}_l^{(i) \top}+\bm{c}_l^{(i)} \bm{s}_l^{(i) \top}\right)\right\|_2\right]\right) \quad \text{ (Jensen’s inequality)}\\
    & \leq \frac{2a}{L}\left(N\left\|\mathbf{K}_{\mathrm{sm}}\right\|_2+\frac{N}{2} \sum_{i=1}^m \alpha_i \mathbb{E}\left[\left\|\left(\bm{s}_l^{(i)} \bm{c}_l^{(i) \top}+\bm{c}_l^{(i)} \bm{s}_l^{(i) \top}\right)\right\|_2\right]\right) \qquad \qquad \quad    \left( | \bm{s}_l^{(i)\top} \bm{c}_l^{(i)} | \le \frac{N}{2} \right) \\
    & \leq \frac{2a N}{L}\left(\left\|\mathbf{K}_{\mathrm{sm}}\right\|_2+\frac{N}{2} a \sqrt{m} \right)  
    \end{aligned}
\end{equation}
where the last inequality is because that 
\begin{equation}
    \mathbb{E}\left[\left\|\left(\bm{s}_l^{(i)} \bm{c}_l^{(i) \top}+\bm{c}_l^{(i)} \bm{s}_l^{(i) \top}\right)\right\|_2\right] = \sup_{\|\bm{v}\|_2^2 =1} \mathbb{E}\left[ \left\| \bm{v}^\top \left(\bm{s}_l^{(i)} \bm{c}_l^{(i) \top}+\bm{c}_l^{(i)} \bm{s}_l^{(i) \top}\right) \bm{v} \right\|_2\right] \le N, 
\end{equation}
and $\sum_{i=1}^m \alpha_i \le a \sqrt{m}$ by the Cauchy–Schwarz inequality.
\paragraph{4). Final Result.}
We next can apply the derived upper bounds, Eqs.~\eqref{eq:upper_bound_e_l_i} and \eqref{eq:variance_bound}, to the $H$ and $v(\bm{Y})$  in Lemma \ref{lemma:matrix_bernstein}, 
\begin{equation}
\begin{aligned}
    & {P} \left(\left\|\hat{\mathbf{K}}_{\mathrm{sm}} - \mathbf{K}_{\mathrm{sm}} \right\|_2 \geq \epsilon\right) \leq N \exp \left(\frac{-3 \epsilon^2 L }{2N a \left(6\left\| \mathbf{K}_{\mathrm{sm}} \right\|_2 + 3 N a \sqrt{m}+8 \epsilon\right)}\right)
\end{aligned}
\end{equation}
which completes the proof of Theorem \ref{thm:SM_RFF_approx}
\end{proof}

\section{Extended Related Work}
\label{app:related}

\paragraph{VAEs.}




As a facet of model collapse, the posterior collapse in variational autoencoders (VAEs) occurs when the variational posterior distribution of the latent variables approaches to the prior, resulting in a failure to exploit the valuable knowledge embedded in the observed data. Numerous approaches have been proposed to tackle this issue, with the most commonly embraced heuristic solution being the annealing of the KL term in the ELBO objective \cite{bowman2015generating, sonderby2016train}. Specifically, \citet{gulrajani2016pixelvae} suggest that posterior collapse is induced by the high-capacity decoder, which can map any noise vector to the desired target $\vx$. Motivating by this hypothesis, \citet{gulrajani2016pixelvae, yang2017improved} propose reducing the capacity of the decoder for better representations, albeit at the cost of a reduction in generative capability.   
Another line of works, such as \cite{lucas2019don, wang2022posterior,wang2021posterior}, claims that posterior collapse is partially attributed to the suboptimal selection of likelihood variances, aligning with our findings in the context of the Bayesian non-parametric GPLVM. Nevertheless, despite the alignment of these works addressing posterior collapse with our findings, the primary objective in VAEs is to improve generative capacity, deviating from our emphasis, which lies in recovering compact and informative latent representations. 


\paragraph{GPLVMs.} 
This paper focuses on the GPLVMs \citep{lawrence2005probabilistic}, which apply GP for modeling the nonlinear function in LVM, obviating the need to optimize substantial neural network parameters while alleviating overfitting and generalization issues \cite{wilson2020bayesian}.  The seminal work of GPLVM was proposed by \citet{lawrence2005probabilistic}. Subsequently, \citet{titsias2010bayesian} introduced the Bayesian formulation of the GPLVM, which variationally integrated out latent variables. However, this model exhibits computational efficiency only with specific preliminary kernel functions, such as the radial basis function (RBF) kernel \citep{williams2006gaussian}, imposing significant constraints on the model capacity of the GPLVM and leading to model collapse.
Recent endeavors have focused on enhancing the scalability and flexibility of the GPLVM \citep{lalchand2022generalised,de2021learning}, as well as ensuring compatibility with various likelihoods \cite{ramchandran2021latent}. Despite the relevance of these endeavors, the inference of these models relies on inducing points-based sparse GP \citep{titsias2009variational}. This necessitates optimizing additional inducing points, leading to increased computational burden and the risk of getting stuck in suboptimal solutions.  
Consequently, despite the enhanced model capability, these models often face challenges in achieving their theoretical potential to address model collapse.




\begin{table*}[t!]
    \caption{A summary of relevant LVMs, where $N$ and $M$ denote \# observations and the observation dimension, respectively, while $U, m, L$ represent \# inducing points, \#  mixture components in SM kernel, and the dimension of random features, respectively.
        } \label{table:comparison}
    \centering
    \vspace{-1ex}
    \resizebox{0.99\textwidth}{!}{%
        \begin{tabular}{lccccccl}
        \toprule
            Model 
            & \begin{tabular}[c]{@{}c@{}} Scalable \\ model  \end{tabular} 
            & \begin{tabular}[c]{@{}c@{}} Advanced \\ kernel \end{tabular} 
            & \begin{tabular}[c]{@{}c@{}} Probabilistic \\ mapping  \end{tabular} 
            & \begin{tabular}[c]{@{}c@{}} Bayesian  inference \\ of latent variables   \end{tabular}  
            & \begin{tabular}[c]{@{}c@{}} Computational \\ complexity \end{tabular} 
            & \begin{tabular}[c]{@{}c@{}} \# parameters  \end{tabular}  
            & Reference          \\
            \midrule
            \tikzexternalenable
            \tikzexternaldisable  \textsc{vae}    
            & \cmark                       
            & -                  
            &  \xmark       
            & \cmark         
            & -                 
            & -                
            & \citet{kingma2019introduction}  \\
            \tikzexternalenable
            \tikzexternaldisable \textsc{nbvae}  
            & \cmark                  
            & -                  
            & \xmark & \cmark   
            & -              
            & -              
            & \citet{zhao2020variational}   \\
            \tikzexternalenable
            \tikzexternaldisable \textsc{dca}
            & \cmark                      
            & -                 
            & \xmark & \cmark     
            & -                    
            & -                   
            & \citet{eraslan2019single}   \\
            \tikzexternalenable
            \tikzexternaldisable  \textsc{cvq}-\textsc{vae}       
            & \cmark                   
            & -                 
            & \xmark  & \cmark         
            & -                     
            & -                  
            & \citet{zheng2023online}  \\
            \tikzexternalenable
            \tikzexternaldisable   \textsc{gplvm}   
            & \xmark  
            & \xmark    
            & \cmark     
            & \xmark     
            & $\mathcal{O}(N^3)$                         
            & $N(N+Q)+C$    
            & \citet{lawrence2005probabilistic}   \\
            \tikzexternalenable
            \tikzexternaldisable  \textsc{bgplvm}                      
            & \cmark       
            & \xmark                    
            & \cmark                    
            & \cmark                          
            & $\mathcal{O}(NU^2)$                        
            & $Q(1+U+N+NQ) + C$                                                   
            & \citet{titsias2010bayesian}     \\
            \tikzexternalenable
            \tikzexternaldisable  \textsc{gplvm}-\textsc{svi}               
            & \cmark      
            & \xmark             
            & \cmark                    
            & \cmark   
            & $\mathcal{O}(MU^3)$   
            & $U(M+MU+Q)+2NQ+C$     
            & \citet{lalchand2022generalised}   \\
            \tikzexternalenable
            \tikzexternaldisable  \textsc{rflvm}      
            & \xmark                 
            & \cmark                  
            & \cmark      
            & \xmark            
            & $\mathcal{O}(NM^2L)$    
            &   $NQ+L(Q+M+\frac{Q^2}{2})+2M+C$   
            & \citet{zhang2023bayesian}   \\
            \midrule
            \tikzexternalenable
            \tikzexternaldisable  \adrflvm        
            & \cmark                                 
            & \cmark & \cmark                
            & \cmark      
            & $\mathcal{O}(N (mL)^2)$    
            & $Q(N+NQ+2m)+m+C$              
            & \textbf{This work} \\
            \bottomrule
            \tikzexternalenable
        \end{tabular}}
\end{table*}

\section{Experiment Details} \label{Appendix: Experiment_details}



\subsection{Data Descriptions and Preprocessing}\label{app:ds_preprocessing}

We first describe the detailed parameter settings for the two synthetic $S$-shaped datasets used in \S~\ref{subsec:S_shape_demo}. The datasets are generated from a GPLVM with different kernel configurations, which are listed below:
\begin{itemize}
    \item Dataset with \textbf{RBF} kernel:
    \begin{equation}
        k_{\mathrm{rbf}}(\x, \x^\prime) = \ell_{o} \exp(-\frac{(\x- \x^\prime)^2}{2 \ell_l^2}),
    \end{equation}
    with outputscale $\ell_o = 1$ and lengthscale $\ell_l = 1$.
    \item Dataset with a hybrid (\textbf{RBF+periodic}) kernel:
    \begin{subequations}
        \begin{align}
           & k_{\mathrm{hybrid}}(\x, \x^\prime) = k_{\mathrm{rbf}}(\x, \x^\prime) + k_{\mathrm{periodic}}(\x, \x^\prime), & \\
            &k_{\mathrm{rbf}}(\x, \x^\prime) = \ell_{o} \exp(-\frac{(\x- \x^\prime)^2}{2 \ell_l^2}),  & \quad \text{ with }  \ell_o = 0.5, \ell_l = 1;\\
            & k_{\mathrm{periodic}}(\x, \x^\prime) = \ell_{o} \exp \left(-\frac{2 \sin ^2\left(\frac{\x-\x^{\prime}}{p}\right)}{ \ell_l^2}\right),  & \quad \text{ with }  \ell_o = 0.5, \ell_l = 1, p=4.5.
        \end{align}
    \end{subequations}
\end{itemize}

Next, we offer a comprehensive introduction to real-world datasets and downsample large-scale datasets to a smaller size to accommodate the high computational complexity in RFLVM \citep{gundersen2021latent}. 

\begin{itemize}
    \item{
        \textsc{bridges}:
        We recorded the daily count of bicycles crossing each of the four East River bridges in New York City\footnote{\url{https://data.cityofnewyork.us/Transportation/Bicycle-Counts-for-East-River-Bridges/gua4-p9wg}}. To assign labels, we categorized the data into weekday versus weekend, treating them as binary labels due to the absence of explicit labels in the dataset. This categorization was made based on the understanding that weekdays and weekends are inherently linked to variations in bicycle counts.
    }
    
    \item{
        \textsc{cifar}-10:
        To create a final dataset of size 2000, we subsampled 400 images from each class within [airplane, automobile, bird, cat, deer]. These images were further resized from \(32 \times 32\) pixels to \(20 \times 20\) pixels and converted to grayscale. Test performance of different models on the full dataset can be found in \S.~\ref{app:larger_datasets}.
    }
    
    \item{
        \textsc{mnist:}
        The dataset size was reduced by randomly selecting $1000$ images. Test performance of different models on the full dataset can be found in \S.~\ref{app:larger_datasets}.
        }
    
    \item{
        \textsc{montreal:} We analyze the daily count of cyclists on eight bicycle lanes in Montreal\footnote{\url{http://donnees.ville.montreal.qc.ca/dataset/f170fecc-18db-44bc-b4fe-5b0b6d2c7297/resource/64c26fd3-0bdf-45f8-92c6-715a9c852a7b}}. Given the absence of explicit labels, we employed the four seasons as labels, as seasonality is correlated with bicycle counts.
    }
    
    \item{
        \textsc{newsgroups} The 20 Newsgroups Dataset\footnote{\url{http://qwone.com/~jason/20Newsgroups/}} was employed, with classes limited to \textit{comp.sys.mac.hardware}, \textit{sci.med}, and \textit{alt.atheism}. The vocabulary was constrained to words with document frequencies falling within the range of $10-90\%$.
    }

    \item{
        \textsc{yale:}
        The Yale Faces Dataset\footnote{\url{http://vision.ucsd.edu/content/yale-face-database}} was employed in our study, with subject IDs utilized as labels.
    }

    \item{
        \textsc{Brendan:}
        This dataset comprises 2000 images, each with a size of $20 \times 28$ pixels, depicting the face of Brendan\footnote{\url{https://cs.nyu.edu/~roweis/data/frey_rawface.mat}}. 
    }
    
    
\end{itemize} 

\subsection{Benchmark Methods Descriptions}

\begin{itemize}
    \item{
        \textbf{PCA, LDA, Isomap:} PCA \citep{wold1987principal}, LDA \citep{blei2003latent}, and Isomap \citep{balasubramanian2002isomap} were implemented utilizing the \texttt{sklearn.decomposition} module within the \texttt{scikit-learn} library \citep{sklearn_api}.
    }
    
    \item{
        \textbf{HPF:} The implementation of HPF \citep{gopalan2015scalable} is based on the \texttt{hpfrec} library\footnote{\url{https://github.com/david-cortes/hpfrec}}.   
    }
    
    \item{
        \textbf{BGPLVM:} We utilized the \texttt{BayesianGPLVMMiniBatch} implementation in the \texttt{GPy} library\footnote{\url{http://github.com/SheffieldML/GPy}}, which is an inducing points-based method \cite{titsias2009variational}.
    }

    \item{
        \textbf{GPLVM-SVI:} 
        We used the official source code based on \texttt{GPyTorch}\footnote{\url{https://github.com/vr308/Generalised-GPLVM}}. We also extended the GPLVM-SVI \citep{lalchand2022generalised} to accommodate the SM kernel function, but this modification could result in further performance degradation.
    }

    \item{
        \textbf{VAE:} The implementation of the VAE \citep{kingma2013auto} was built upon the example code provided by the \texttt{pytorch} library\footnote{\url{https://github.com/pytorch/examples/blob/main/vae/main.py}}.
    }

     \item{
        \textbf{NBVAE, DCA, CVQ-VAE, RFLVM:} All implementations for those algorithms adhere to the corresponding official code libraries available online
        \footnote{\url{https://github.com/ethanhezhao/NBVAE}}
        \footnote{\url{https://github.com/theislab/dca}}
        \footnote{\url{https://github.com/lyndonzheng/CVQ-VAE}}
        \footnote{\url{https://github.com/gwgundersen/rflvm}}. 
    }    
    
\end{itemize} 

\newpage
\subsection{Default Hyperparameter Configurations}

{
    \captionof{table}{Default hyperparameter settings.}
    \label{tab:para_settings}
    \vspace{-2ex}
    \small
    \begin{center}
        \scalebox{.95}
        {
            \rowcolors{1}{}{mylightgray}
            \begin{sc}
                \begin{tabular}{lc}
                    \toprule
                    parameter & value \\                
                    \midrule
                    \midrule
                    \# mixture densities in SM kernel ($m$)  & $2$ \\
                    dim. of random feature ($L$)   & $50$ \\
                    dim. of latent space ($Q$) & $2$ \\
                    \midrule
                    optimizer & adam \cite{kingma2015adam} \\
                    learning rate & $0.005$ \\
                    beta & $(0.9, 0.99)$ \\
                    \# iterations & $10000$    \\
                    \bottomrule
                    
                \end{tabular}
            \end{sc}
        }
    \end{center}
}

Tab.~\ref{tab:para_settings} displays the default hyperparameter settings employed by \adrflvm. 
These hyperparameter settings are employed in the majority of experiments, with the exception of the experiment corresponding to the left side of Fig.~\ref{fig:sigma}. In this case, the dimensionality of the latent space is configured to $50$ to intuitively illustrate the variability in the number of zero-columns within the latent variables. 

\subsection{Additional Results}
\label{app:aditional_results}

\subsubsection{S-shaped Latent Manifold Estimation}
\label{app:S-shape Latent Manifold Estimation}

\begin{figure}[ht!]
    \centering
    \includegraphics[width=.8\linewidth]{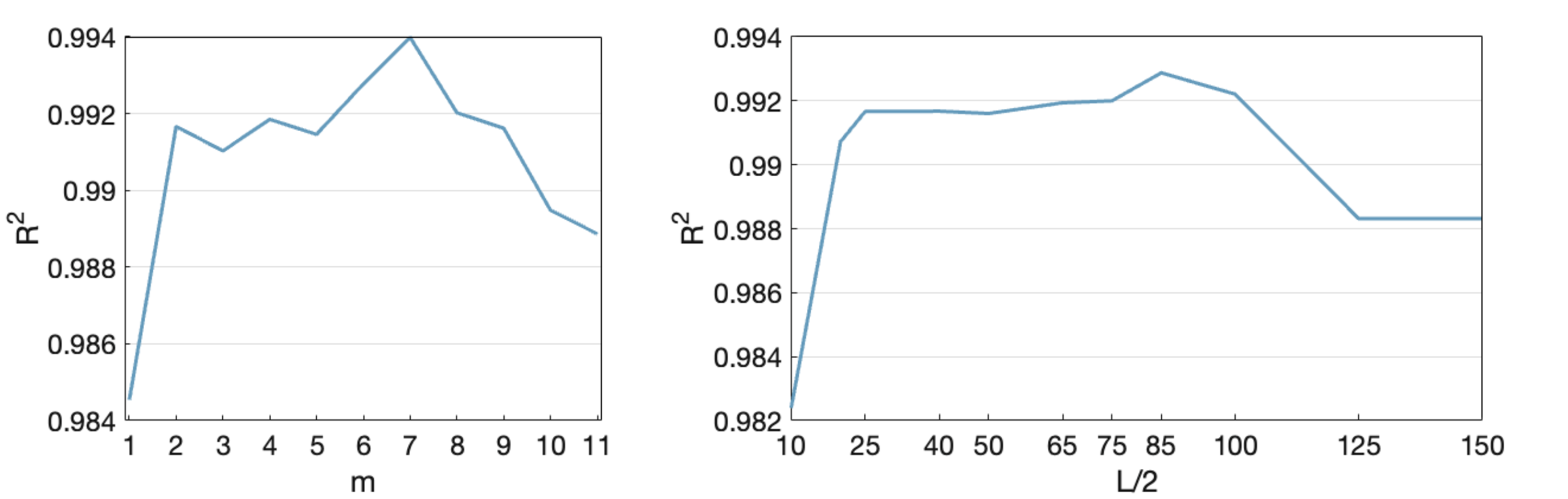}
    \caption{\textbf{(Left)} $\mathrm{R}^2$ against the number of mixture densities in SM kernel $(m)$. \textbf{(Right)}  $\mathrm{R}^2$ versus the dimensionality of the random feature ($L/2$). } 
    \label{fig:r2_m_L}
\end{figure}

To validate the rationale behind our parameter selection, this section presents an evaluation of \adrflvm, showcasing its performance in manifold visualization and $R^2$ scores across various values of $m$ and $L/2$. Fig.~\ref{fig:r2_m_L} depicts the \adrflvm performance in terms of $R^{2}$ scores. Additionally, visualizations of the latent manifold recovered by \adrflvm are provided in Fig.~\ref{fig:s-curve-m} and Fig.~\ref{fig:s-curve-L}. The results affirm that opting for $m = 2$ and $L/2 = 50$ ensures the lowest computational complexity while maintaining comparable performance. 

\subsubsection{Missing Data Imputation}
\label{app:missing_data}

To intuitively showcase the capability of \adrflvm in the task of missing data imputation, visualizations of the reconstructed observed data are presented in Fig.~\ref{appx_fig:mnist_missing_illustration} and Fig.~\ref{appx_fig:bface_missing_illustration}, underscoring its superior ability to restore missing pixels. 

\newpage
\begin{figure}[t!]
    \centering
    \subfloat[Ground-truth]{\includegraphics[width=.2\linewidth]{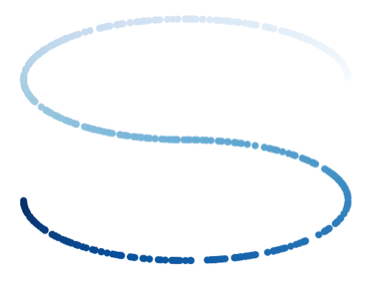}} \hspace{.1in}
    \subfloat[m=1]{\includegraphics[width=.2\linewidth]{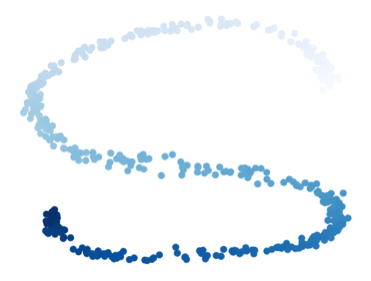}} \hspace{.1in}
    \subfloat[m=2]{\includegraphics[width=.2\linewidth]{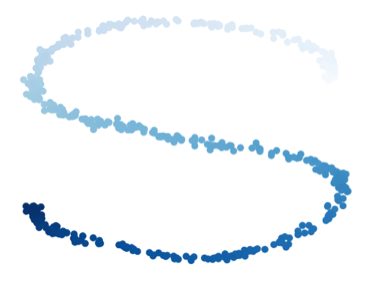}} \hspace{.1in}
    \subfloat[m=3]{\includegraphics[width=.2\linewidth]{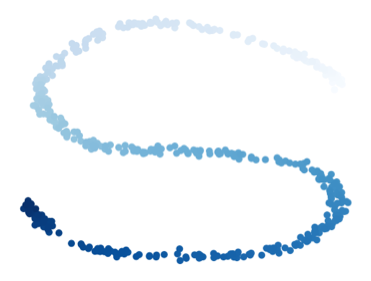}} \hspace{.1in}
    \vspace{.05in}

    \subfloat[m=4]{\includegraphics[width=.2\linewidth]{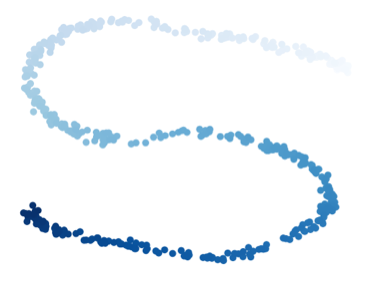}} \hspace{.1in}
    \subfloat[m=5]{\includegraphics[width=.2\linewidth]{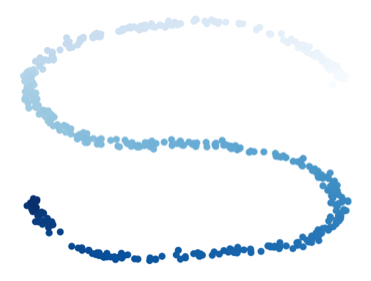}} \hspace{.1in}
    \subfloat[m=6]{\includegraphics[width=.2\linewidth]{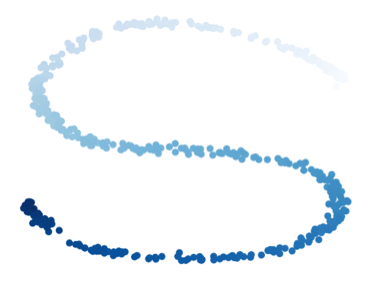}} \hspace{.1in}
    \subfloat[m=7]{\includegraphics[width=.2\linewidth]{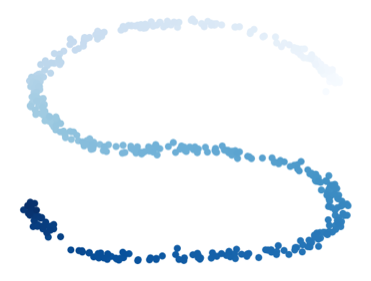}} \hspace{.1in}
    \vspace{.05in}
    
    \subfloat[m=8]{\includegraphics[width=.2\linewidth]{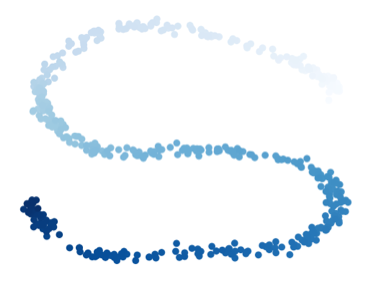}} \hspace{.1in}
    \subfloat[m=9]{\includegraphics[width=.2\linewidth]{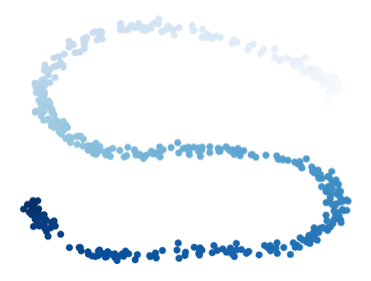}} \hspace{.1in}
    \subfloat[m=10]{\includegraphics[width=.2\linewidth]{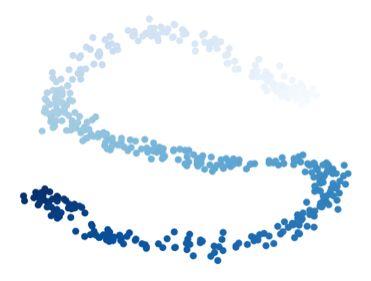}} \hspace{.1in}
    \subfloat[m=11]{\includegraphics[width=.2\linewidth]{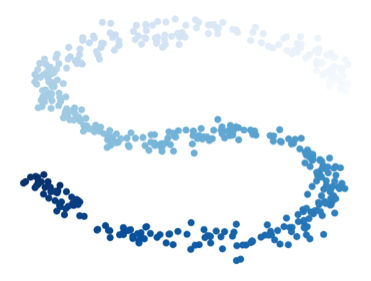}} \hspace{.1in}
    \caption{Latent manifold learning results with $L/2=25$ and different $m$}
    \label{fig:s-curve-m}
\end{figure}
\begin{figure}[t!]
    \centering
    \subfloat[Ground-truth]{\includegraphics[width=.2\linewidth]{figures/s_shape/appx/gd.png}} \hspace{.1in}
    \subfloat[L=10]{\includegraphics[width=.2\linewidth]{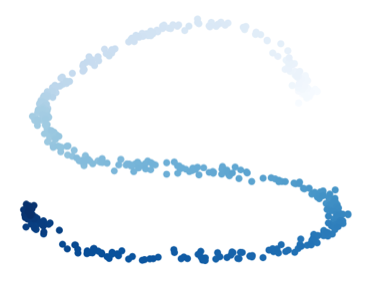}} \hspace{.1in}
    \subfloat[L=20]{\includegraphics[width=.2\linewidth]{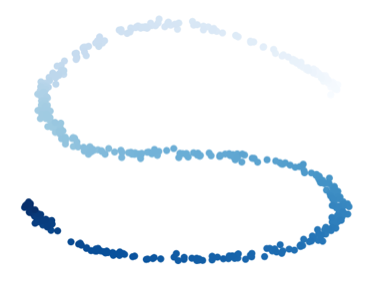}} \hspace{.1in}
    \subfloat[L=25]{\includegraphics[width=.2\linewidth]{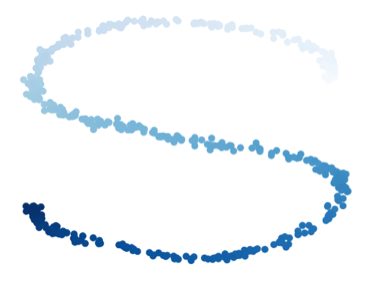}} \hspace{.1in}
    \vspace{.05in}

    \subfloat[L=40]{\includegraphics[width=.2\linewidth]{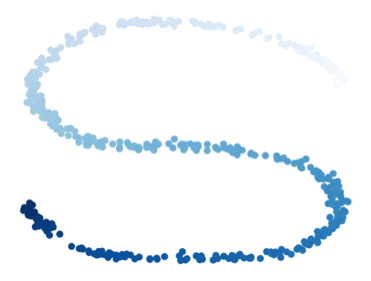}} \hspace{.1in}
    \subfloat[L=50]{\includegraphics[width=.2\linewidth]{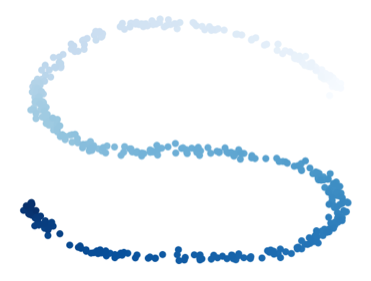}} \hspace{.1in}
    \subfloat[L=65]{\includegraphics[width=.2\linewidth]{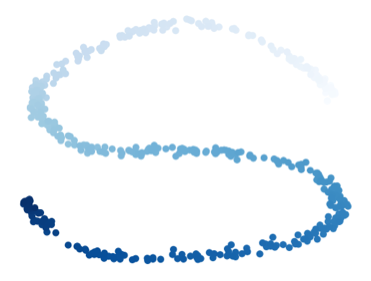}} \hspace{.1in}
    \subfloat[L=75]{\includegraphics[width=.2\linewidth]{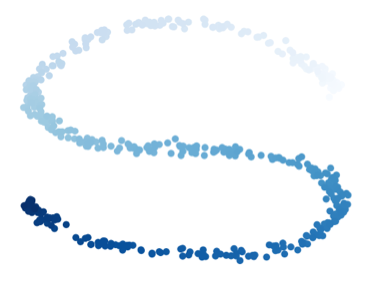}} \hspace{.1in}
    \vspace{.05in}
    
    \subfloat[L=85]{\includegraphics[width=.2\linewidth]{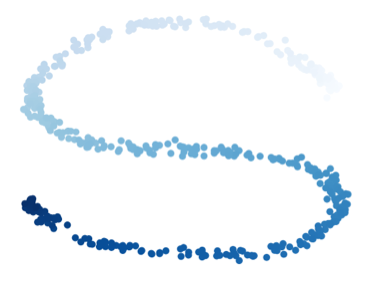}} \hspace{.1in}
    \subfloat[L=100]{\includegraphics[width=.2\linewidth]{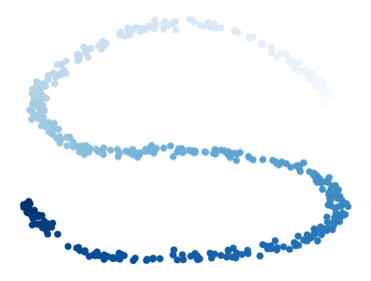}} \hspace{.1in}
    \subfloat[L=125]{\includegraphics[width=.2\linewidth]{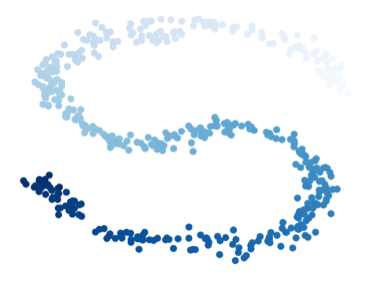}} \hspace{.1in}
    \subfloat[L=150]{\includegraphics[width=.2\linewidth]{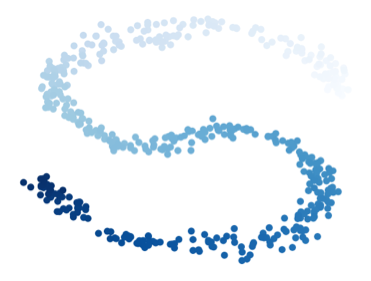}} \hspace{.1in}
    \caption{Latent manifold learning results with $m=2$ and different $L$.}
    \label{fig:s-curve-L}
\end{figure}

\begin{figure}[t!]
    \centering
    \subfloat[MNIST reconstruction task with 0\% missing pixels.]{
    \includegraphics[width=.3\linewidth]{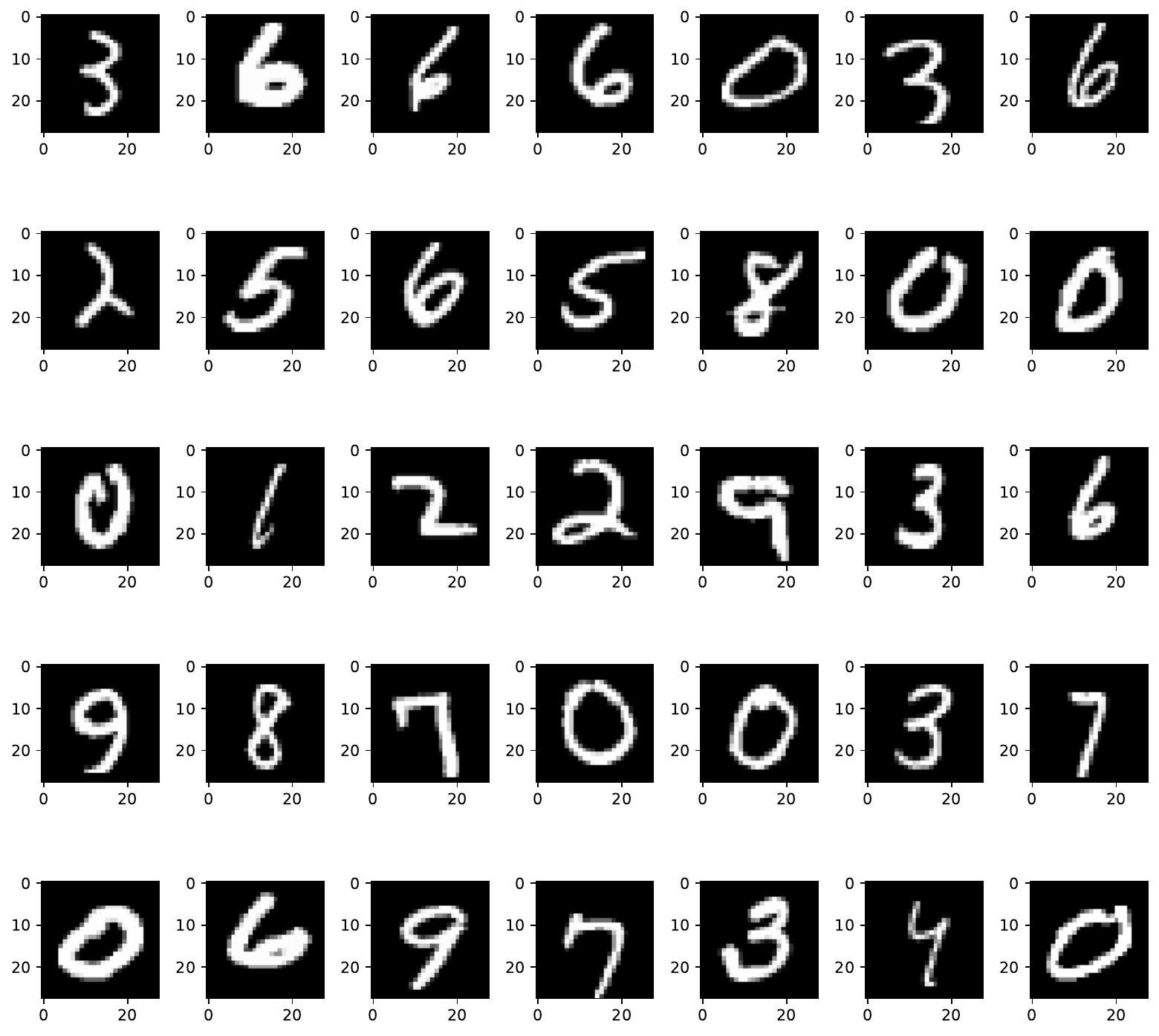} \ \ \
    \includegraphics[width=.3\linewidth]{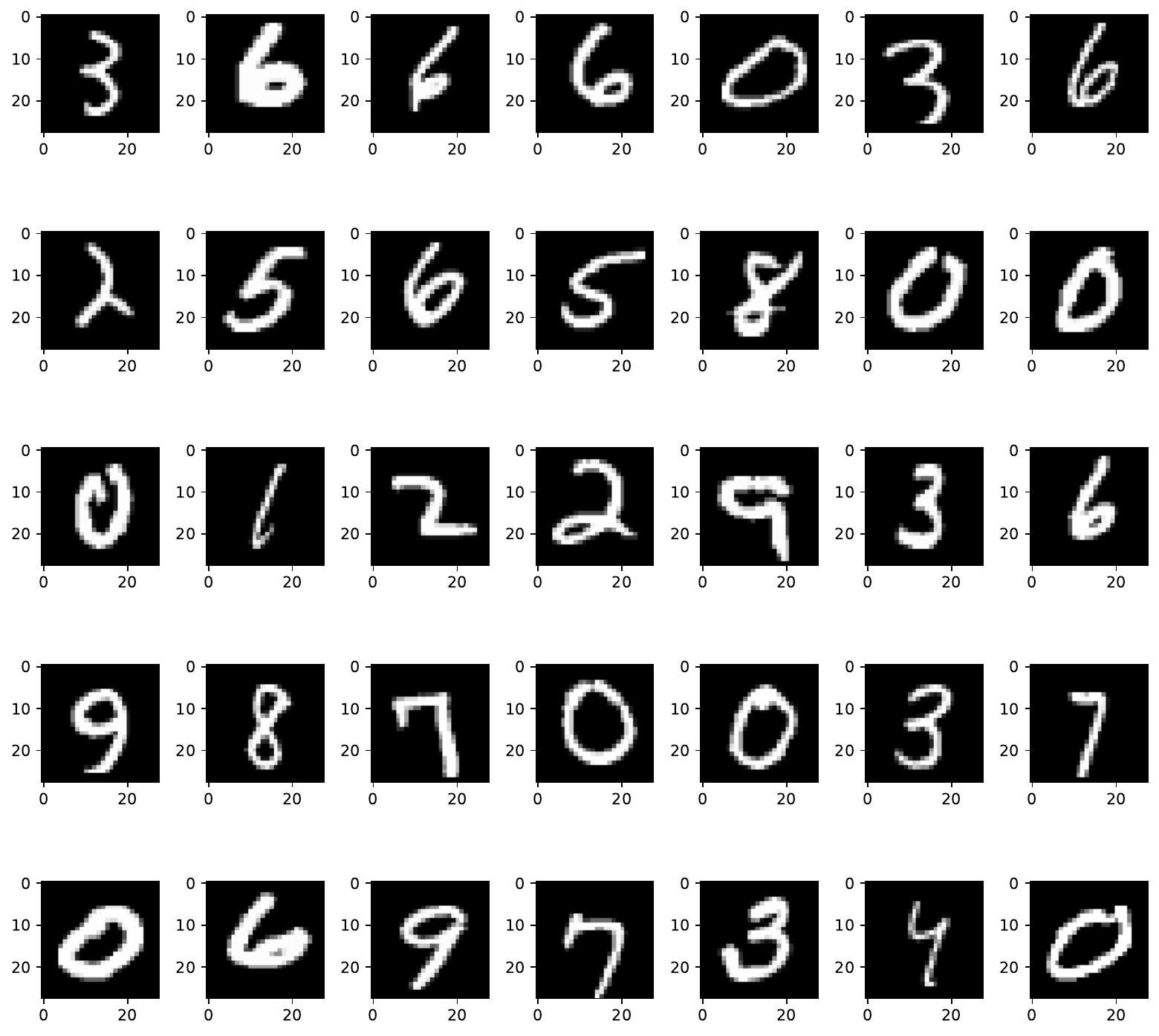} \ \ \
    \includegraphics[width=.3\linewidth]{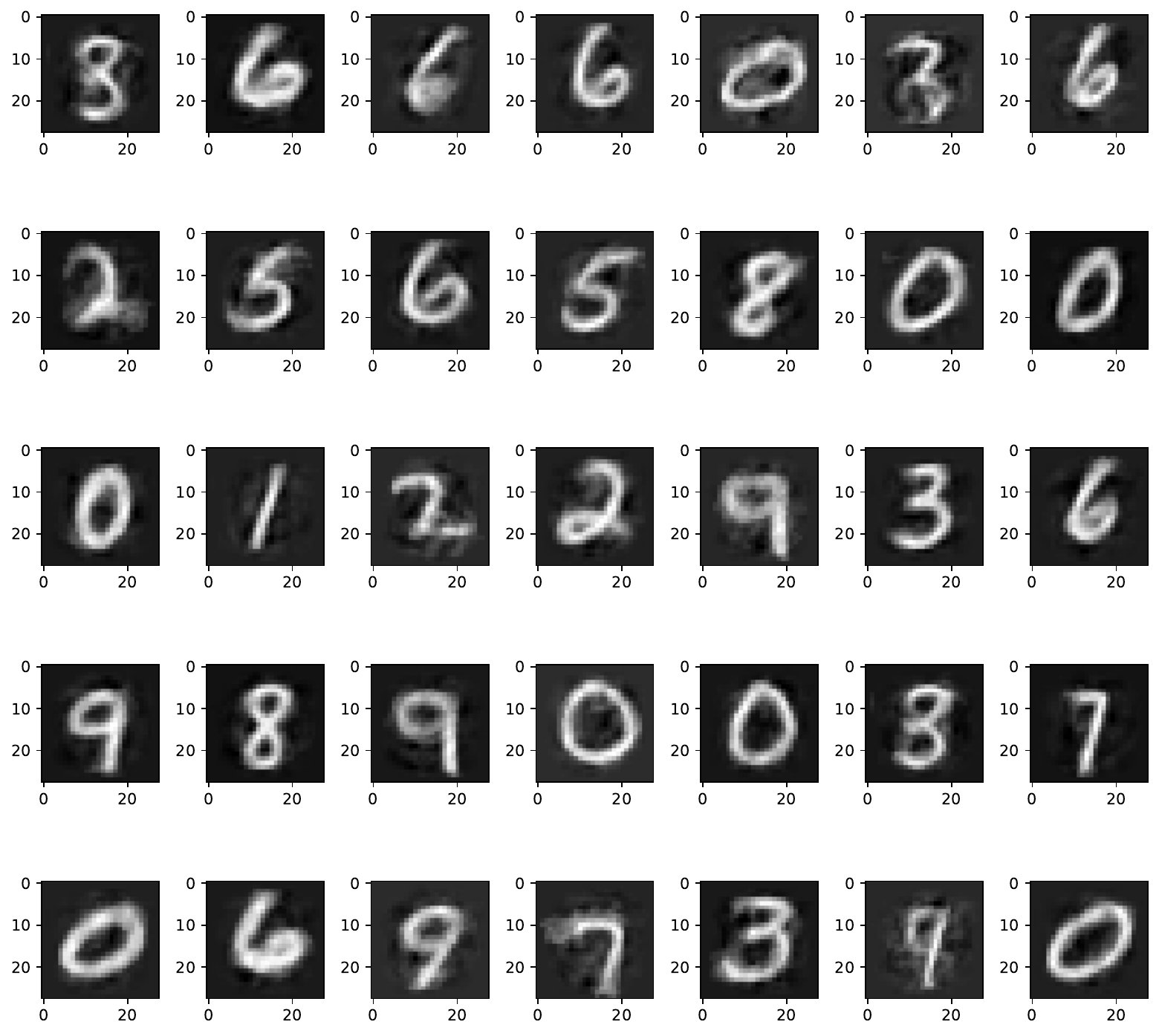}
    } \vspace{.15in}
    
    \subfloat[MNIST reconstruction task with 10\% missing pixels.]{
    \includegraphics[width=.3\linewidth]{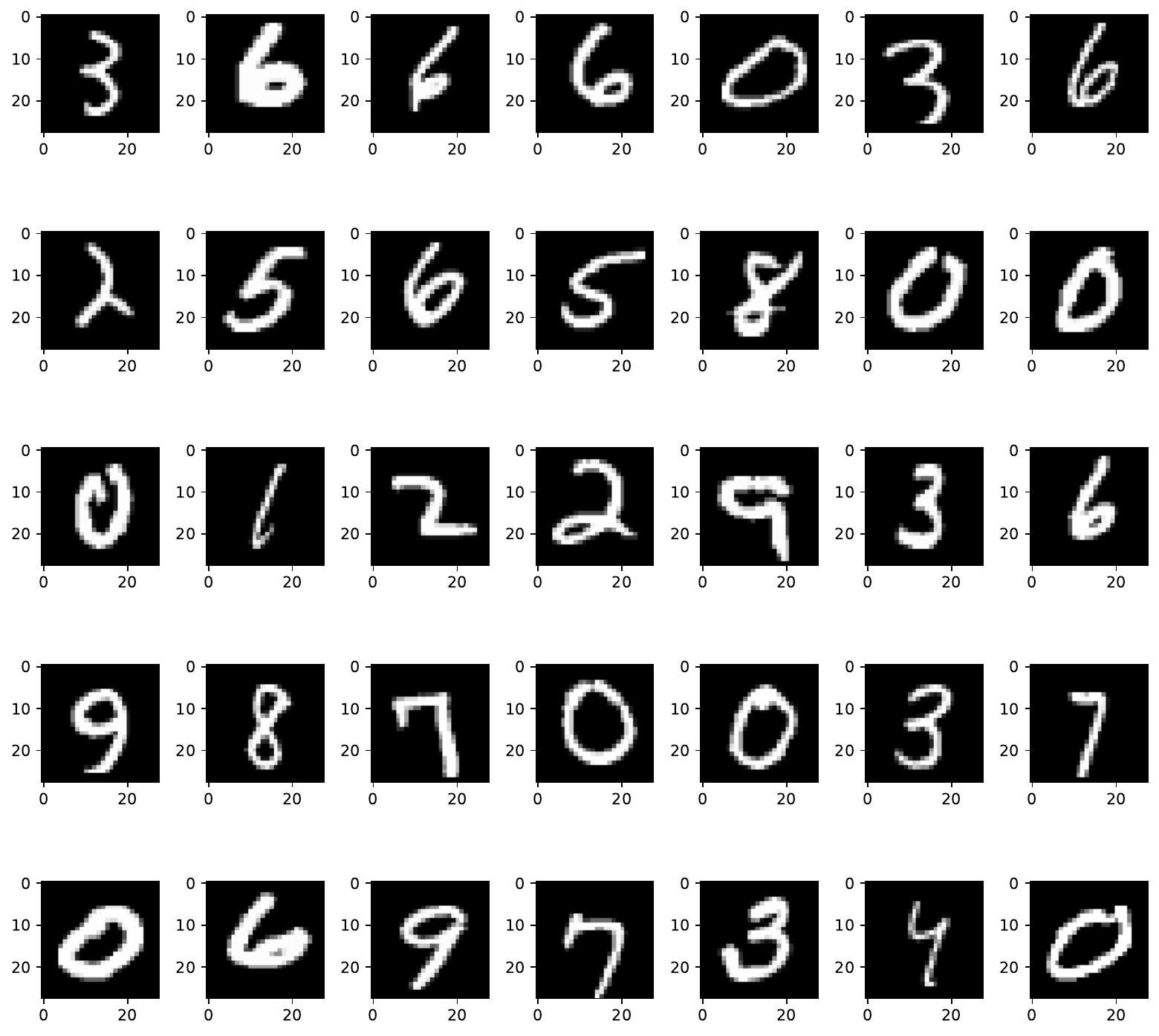} \ \ \
    \includegraphics[width=.3\linewidth]{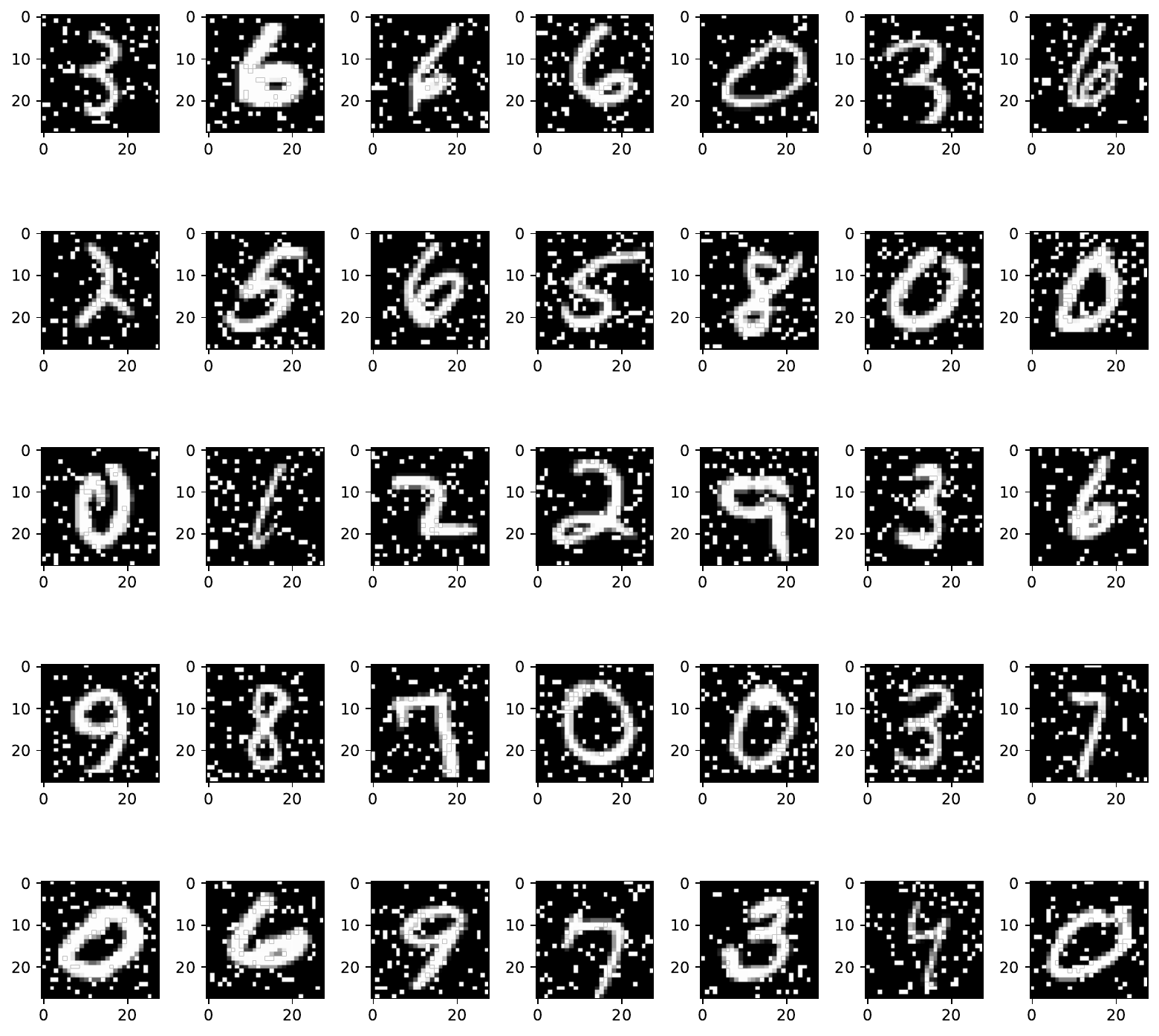} \ \ \
    \includegraphics[width=.3\linewidth]{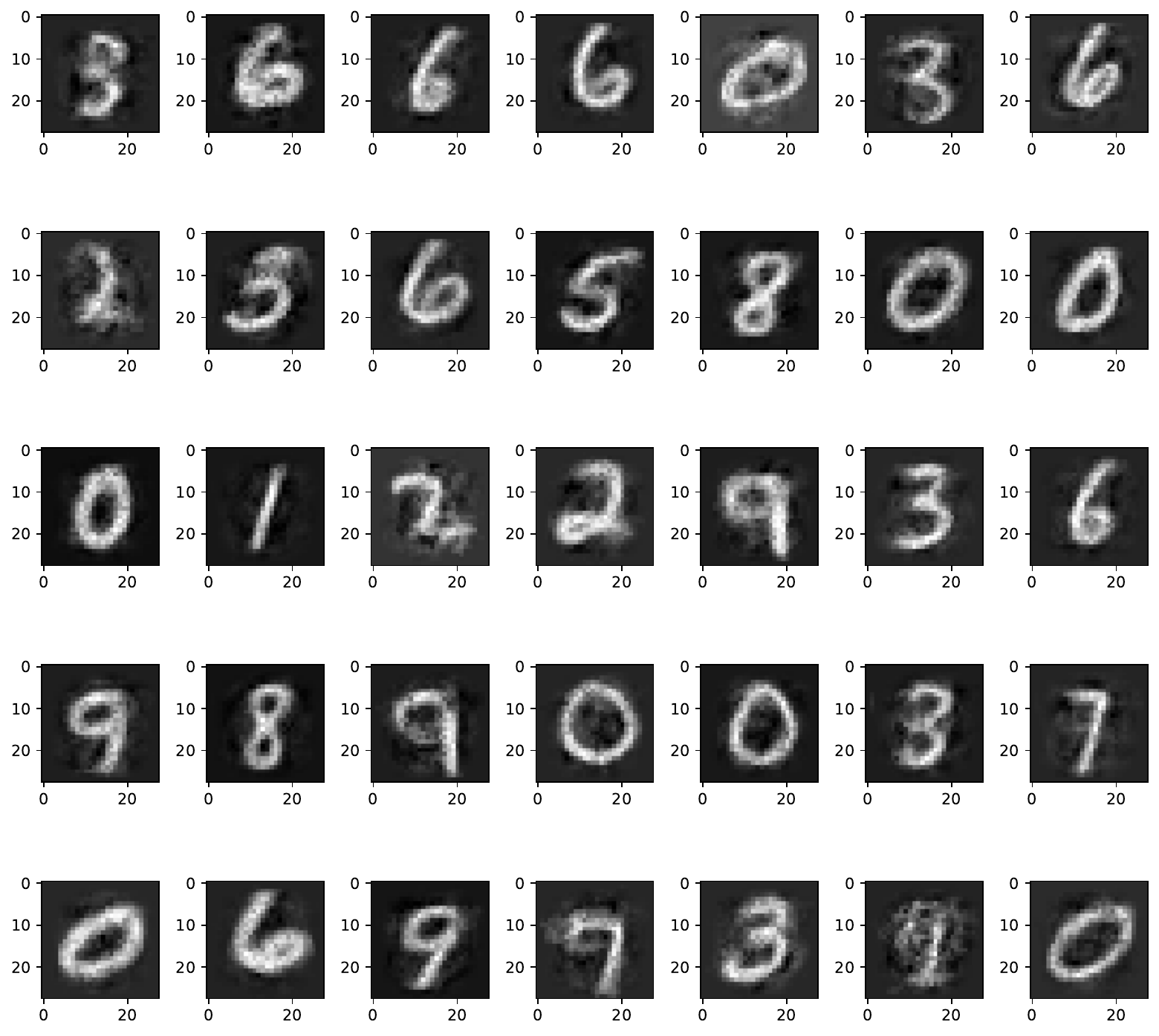}
    } \vspace{.15in}

    \subfloat[MNIST reconstruction task with 30\% missing pixels.]{
    \includegraphics[width=.3\linewidth]{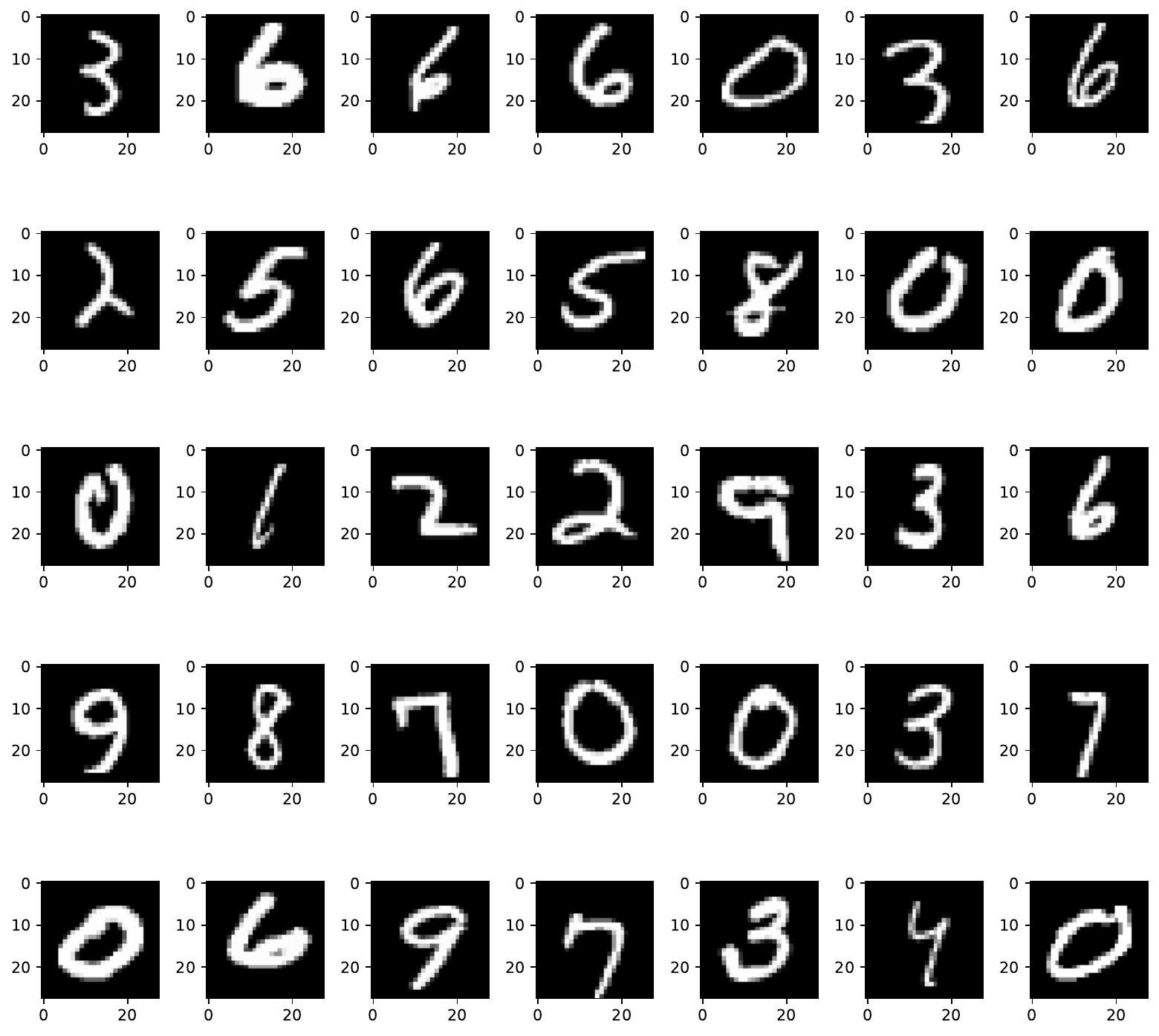} \ \ \
    \includegraphics[width=.3\linewidth]{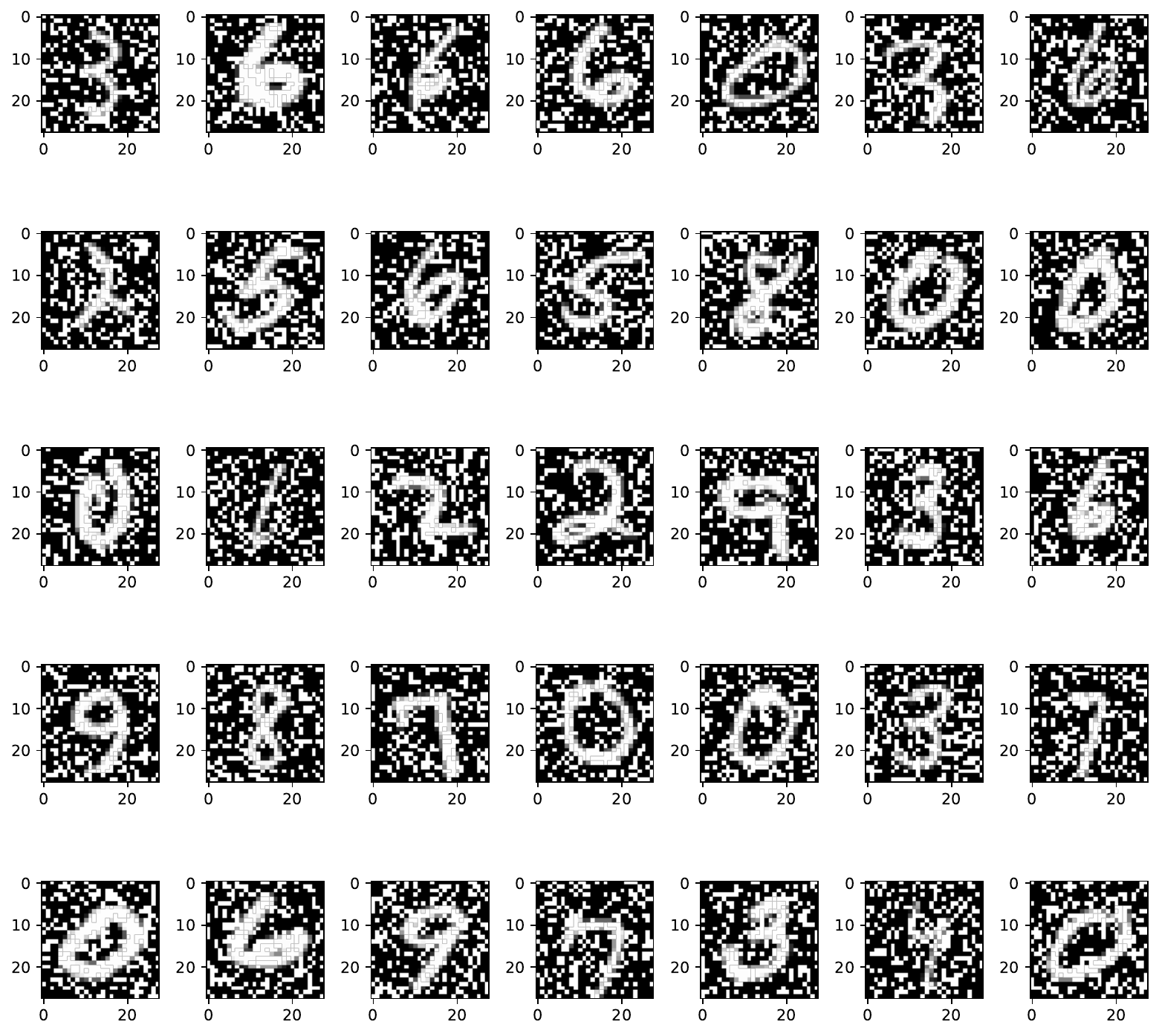} \ \ \
    \includegraphics[width=.3\linewidth]{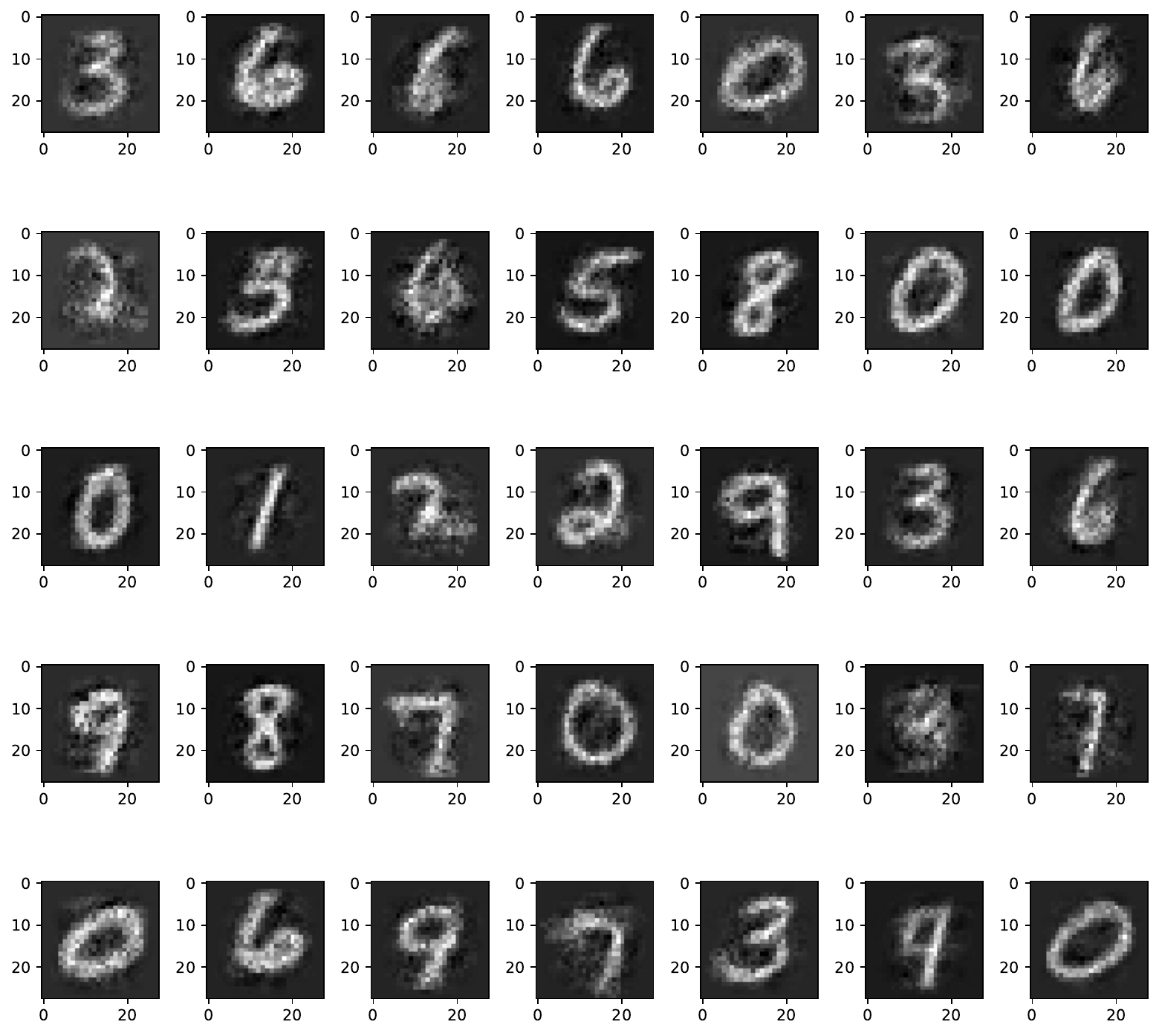}
    } \vspace{.15in}

    \subfloat[MNIST reconstruction task with 60\% missing pixels.]{
    \includegraphics[width=.3\linewidth]{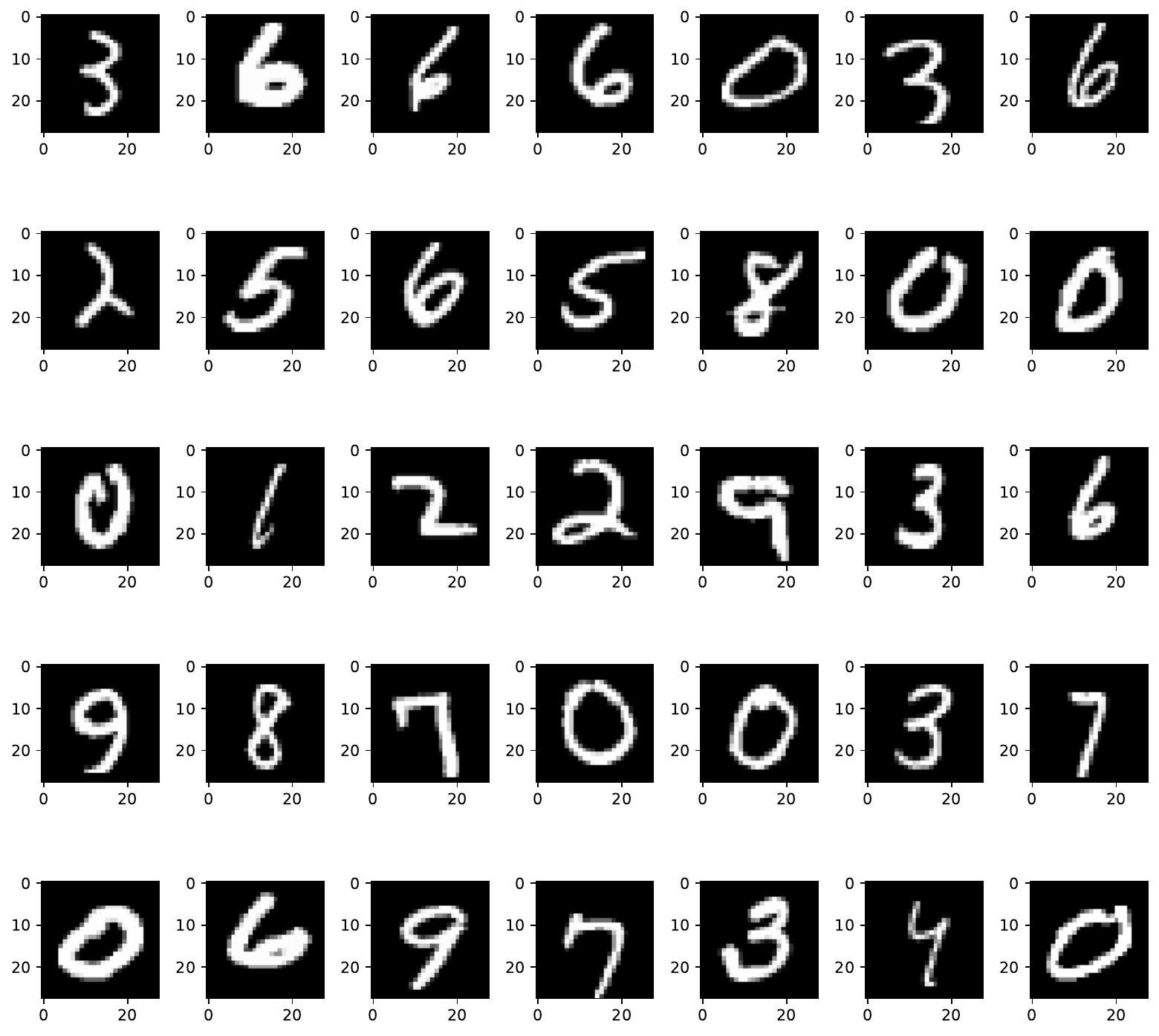} \ \ \
    \includegraphics[width=.3\linewidth]{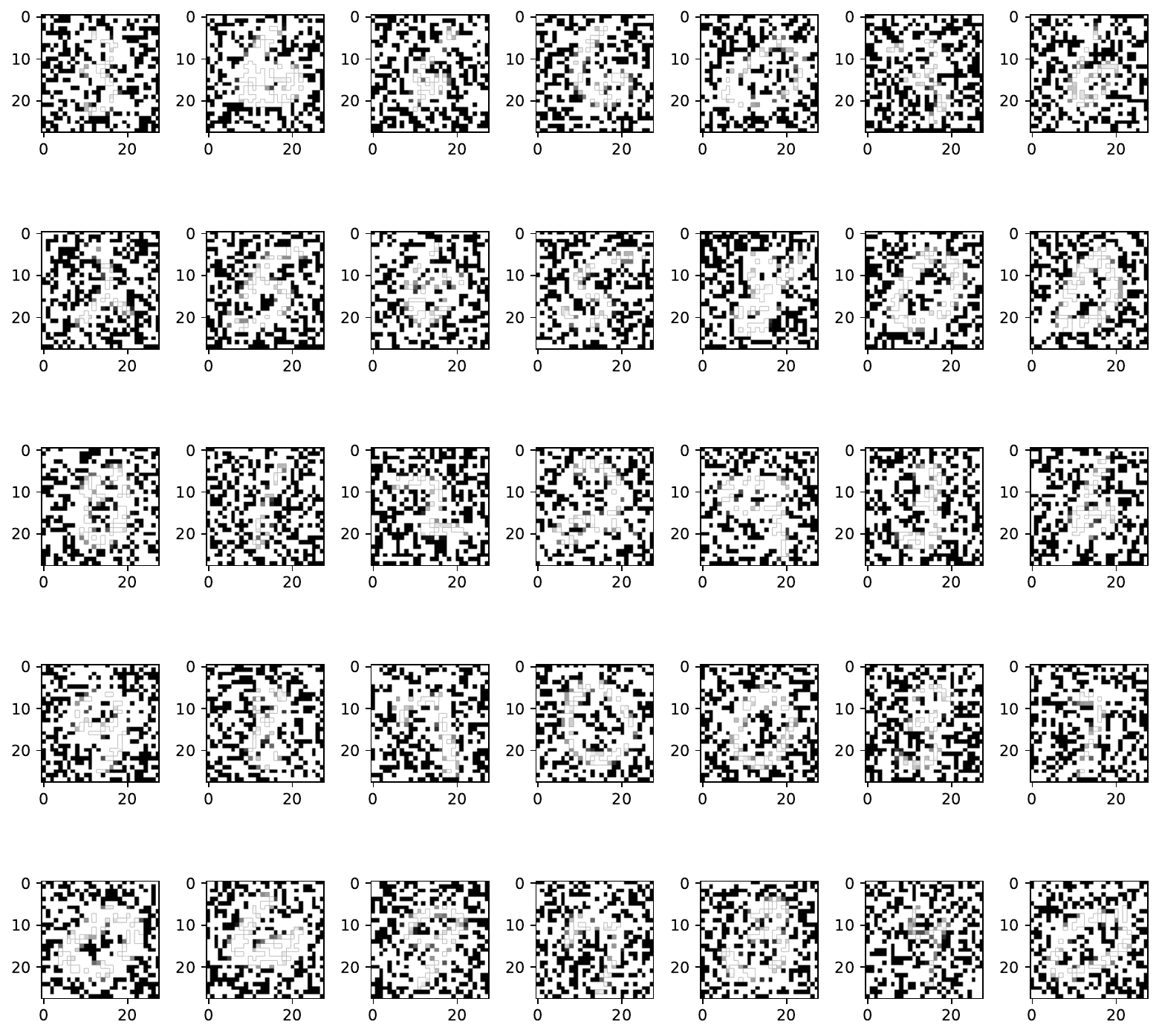} \ \ \
    \includegraphics[width=.3\linewidth]{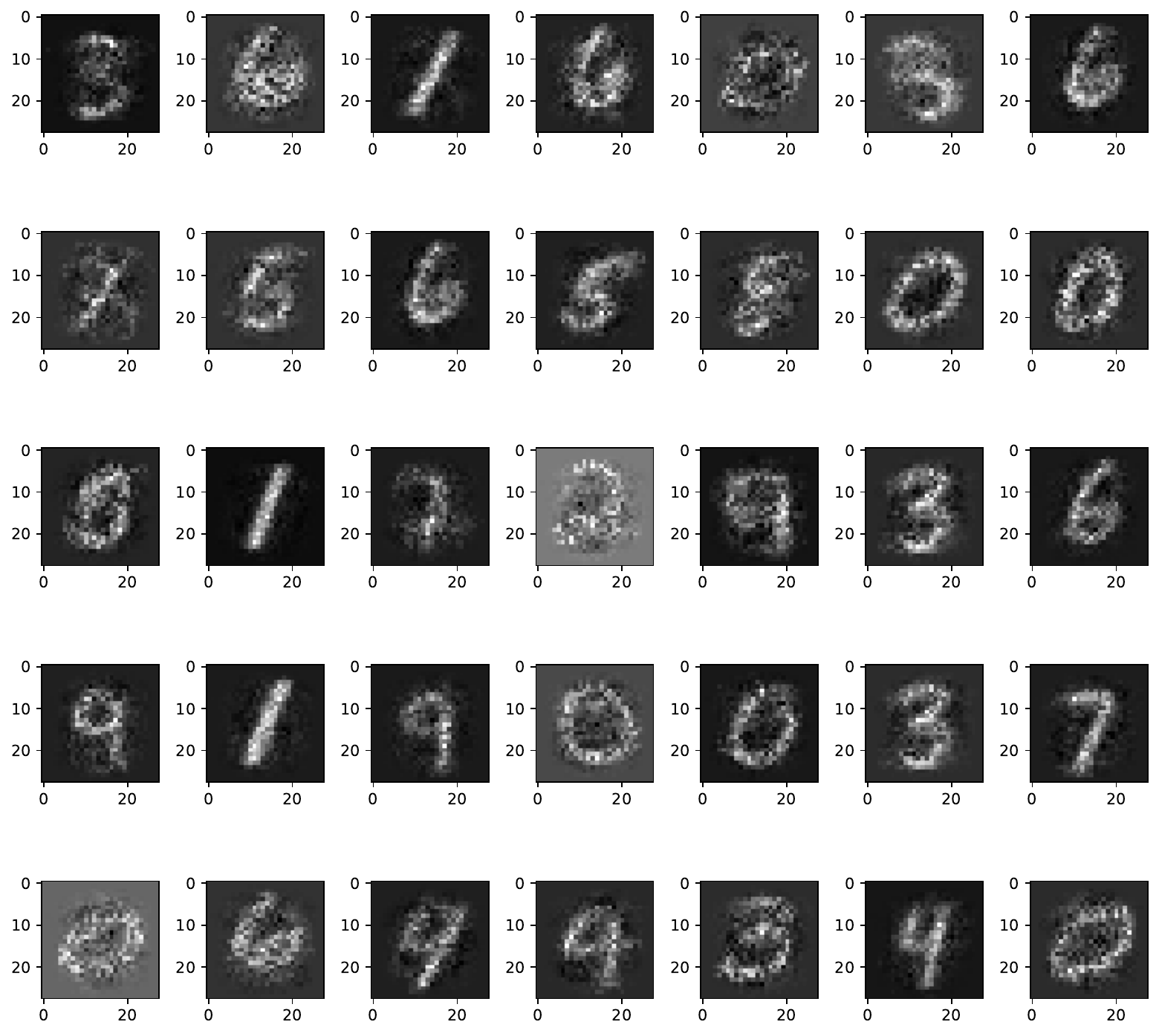}
    }  
    \caption{MNIST reconstruction task with missing pixels. From left to right: Ground truth, training images, reconstructions}
    \label{appx_fig:mnist_missing_illustration}
\end{figure}

\begin{figure}[t!]
    \centering
    \subfloat[Brendan faces reconstruction task with 0\% missing pixels.]{
    \includegraphics[width=.3\linewidth]{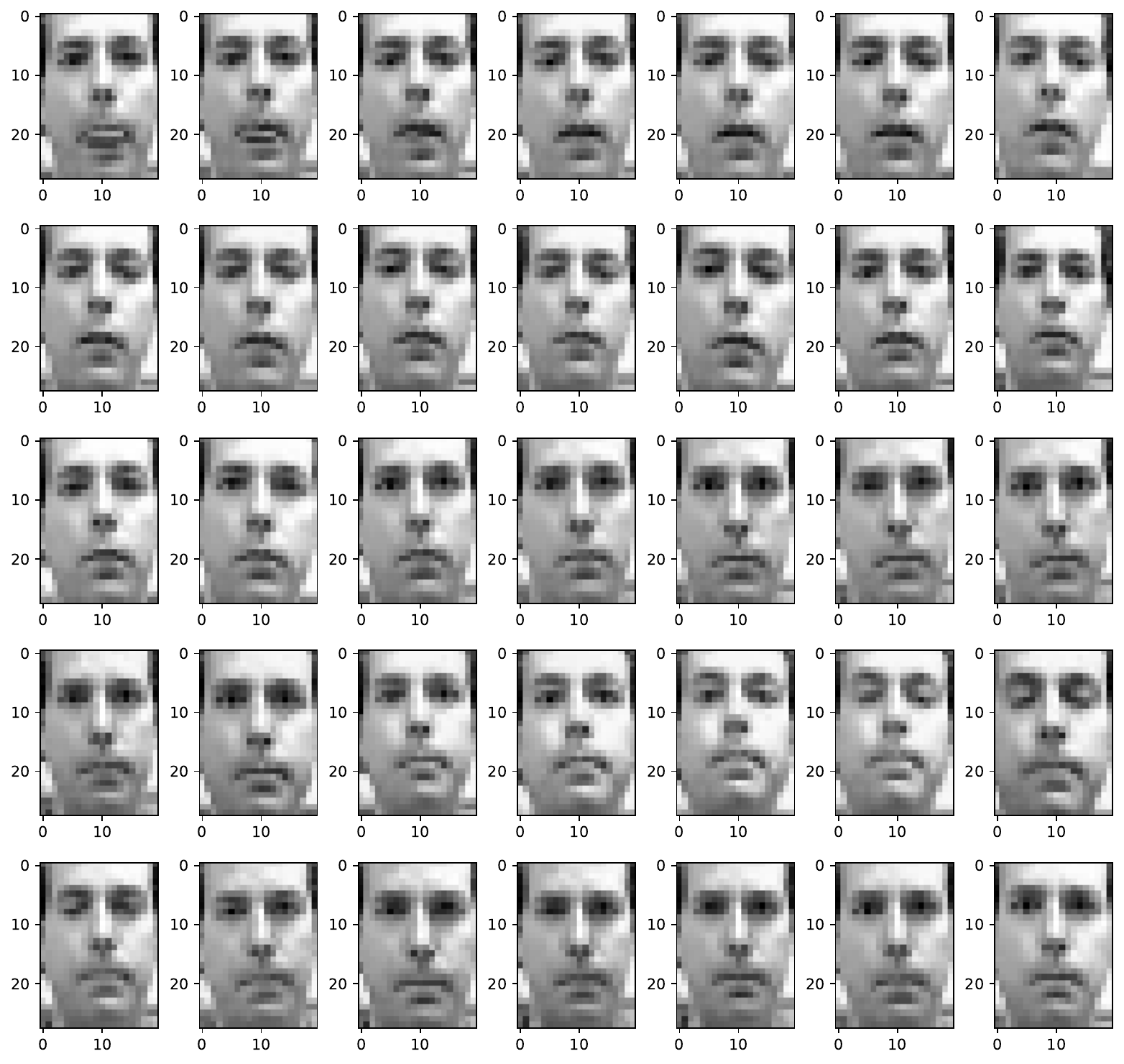} \ \ \
    \includegraphics[width=.3\linewidth]{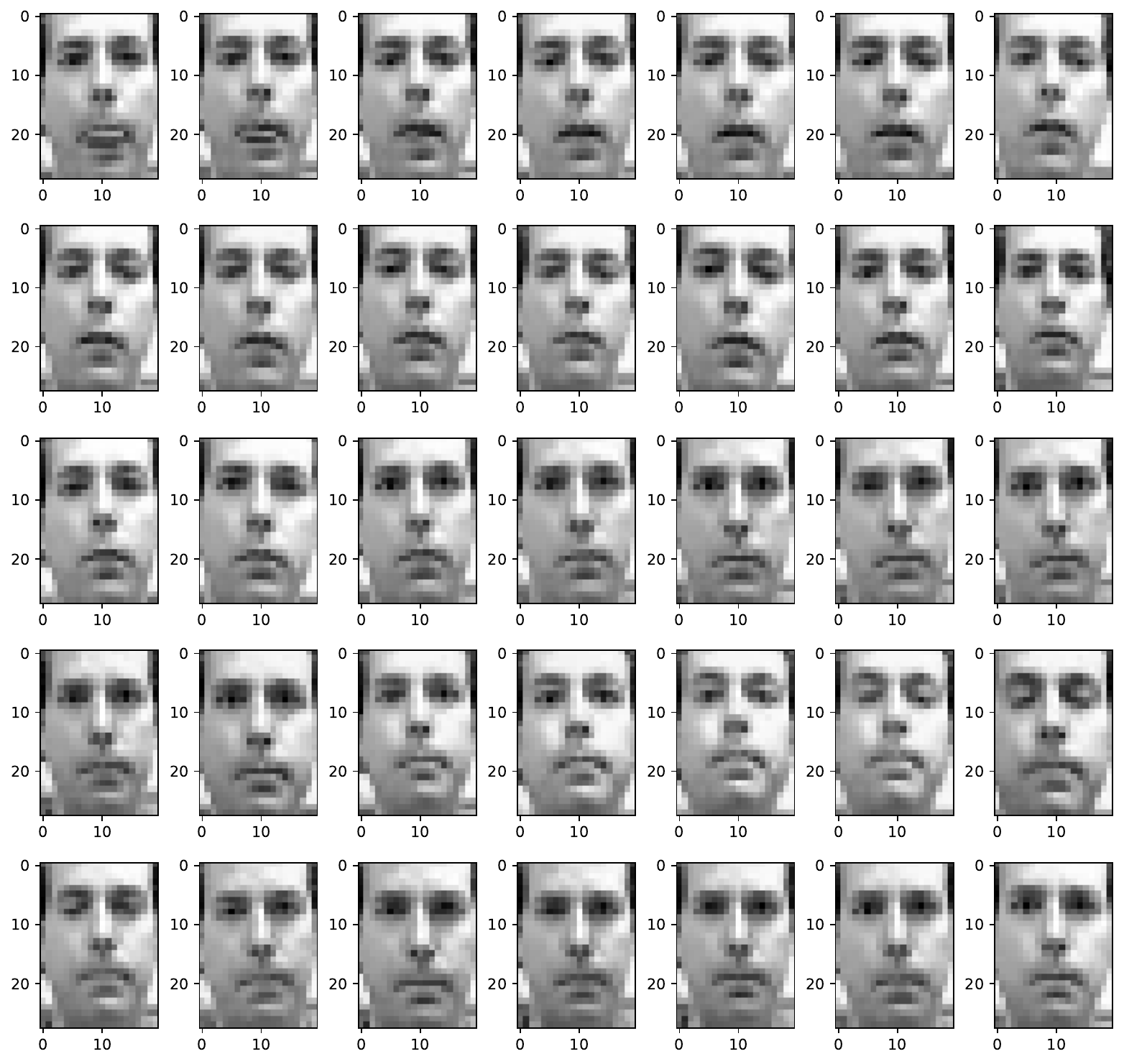} \ \ \
    \includegraphics[width=.3\linewidth]{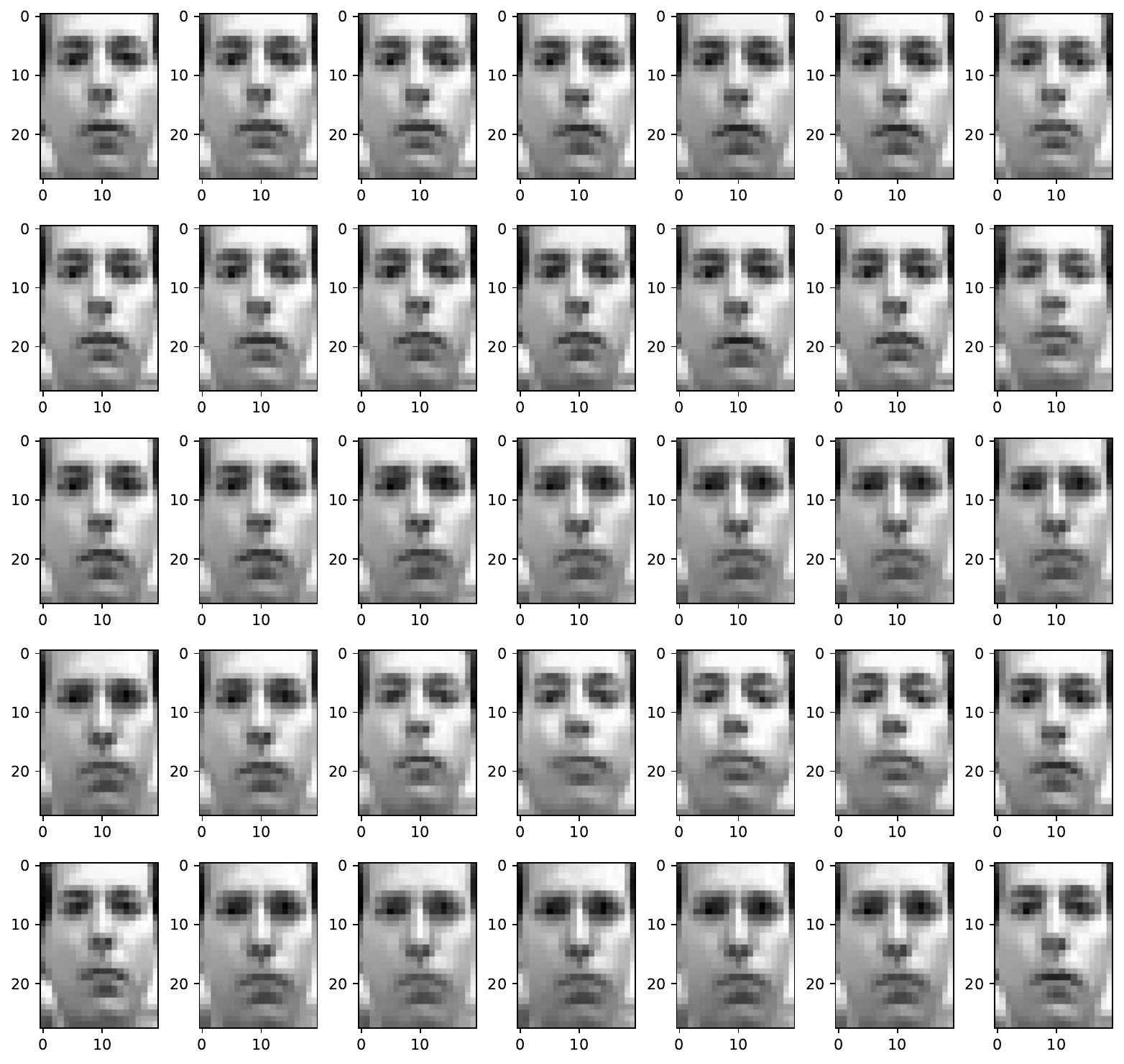}
    } \vspace{.1in}
    
    \subfloat[Brendan faces reconstruction task with 10\% missing pixels.]{
    \includegraphics[width=.3\linewidth]{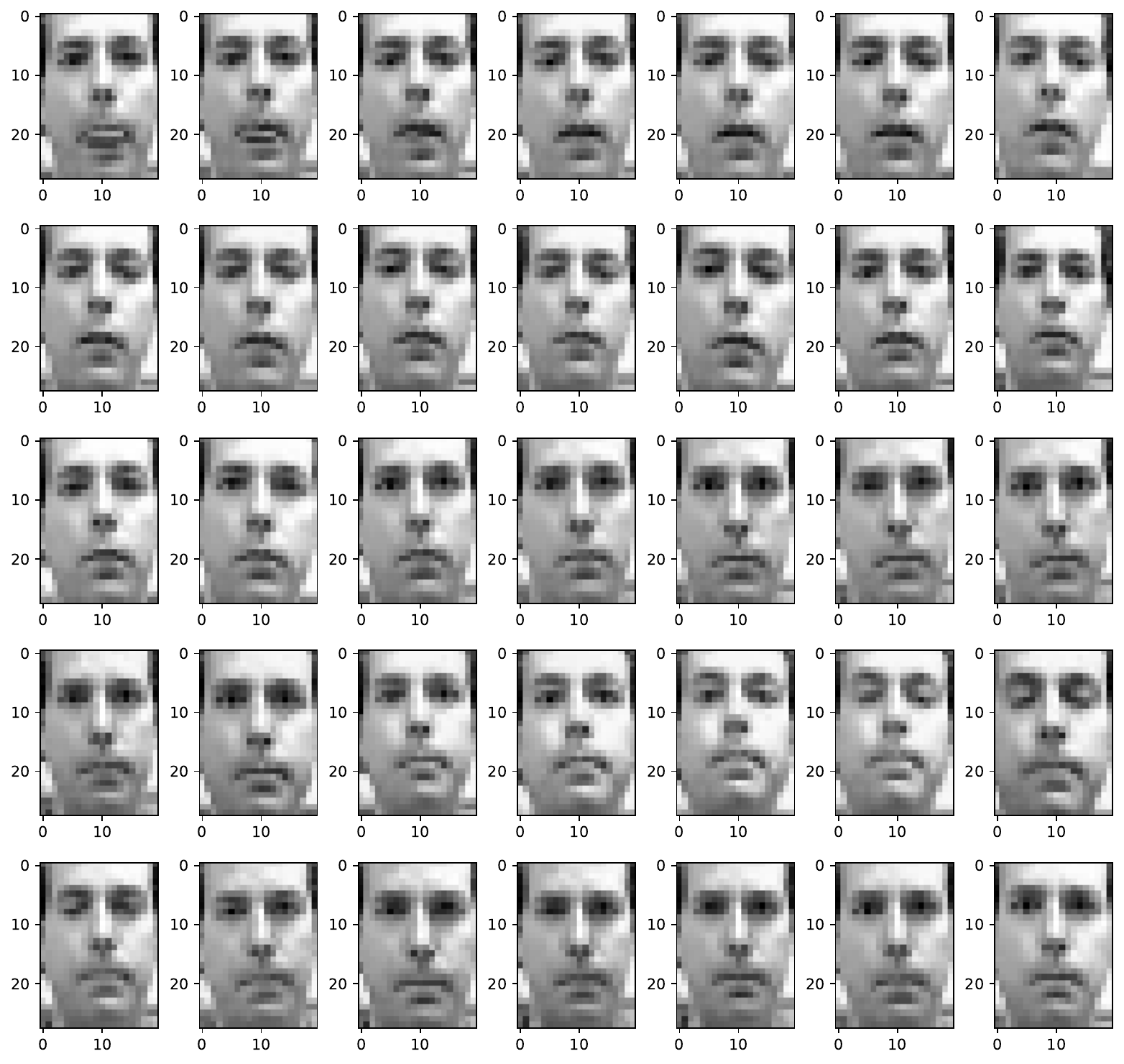} \ \ \
    \includegraphics[width=.3\linewidth]{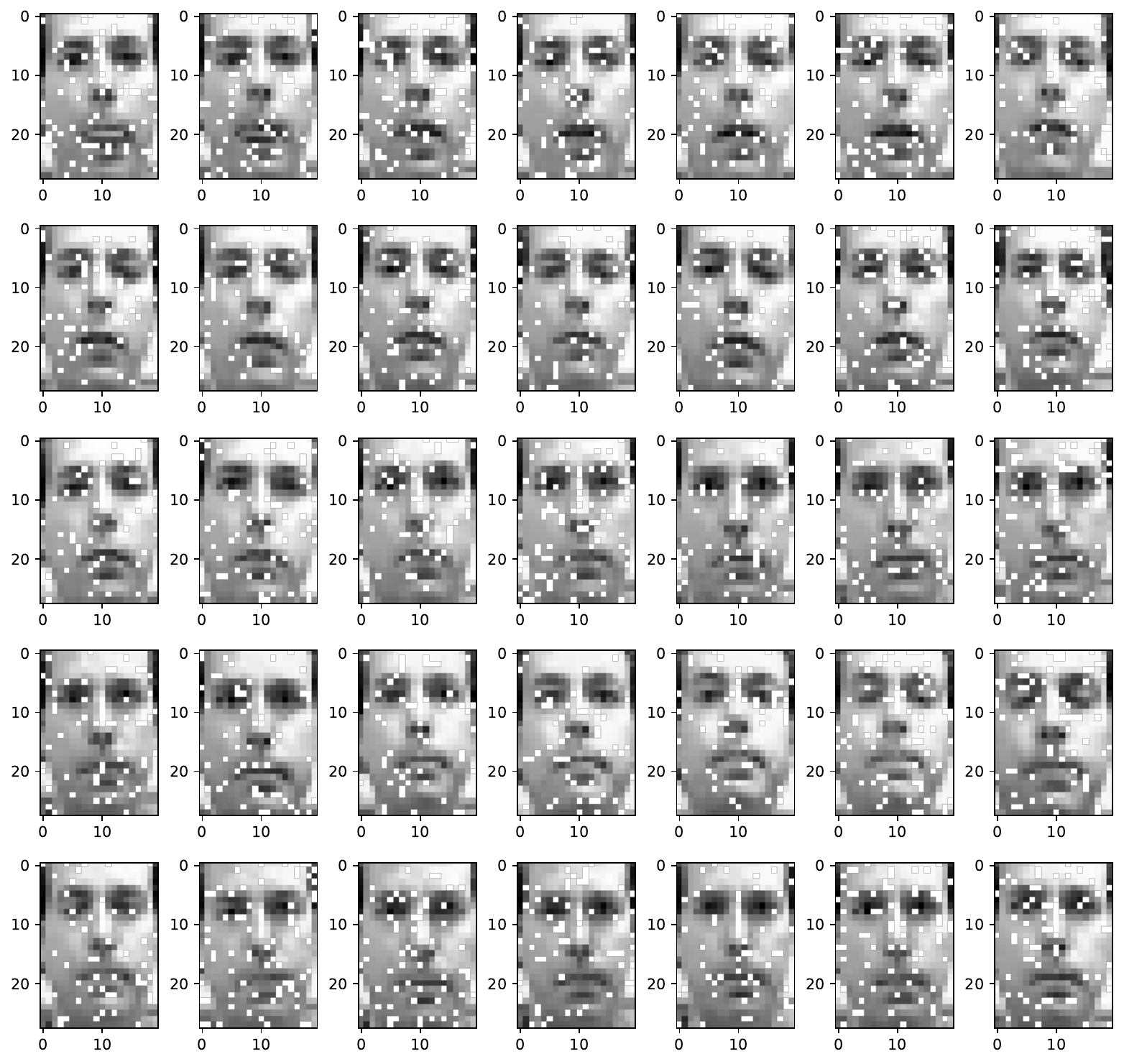} \ \ \
    \includegraphics[width=.3\linewidth]{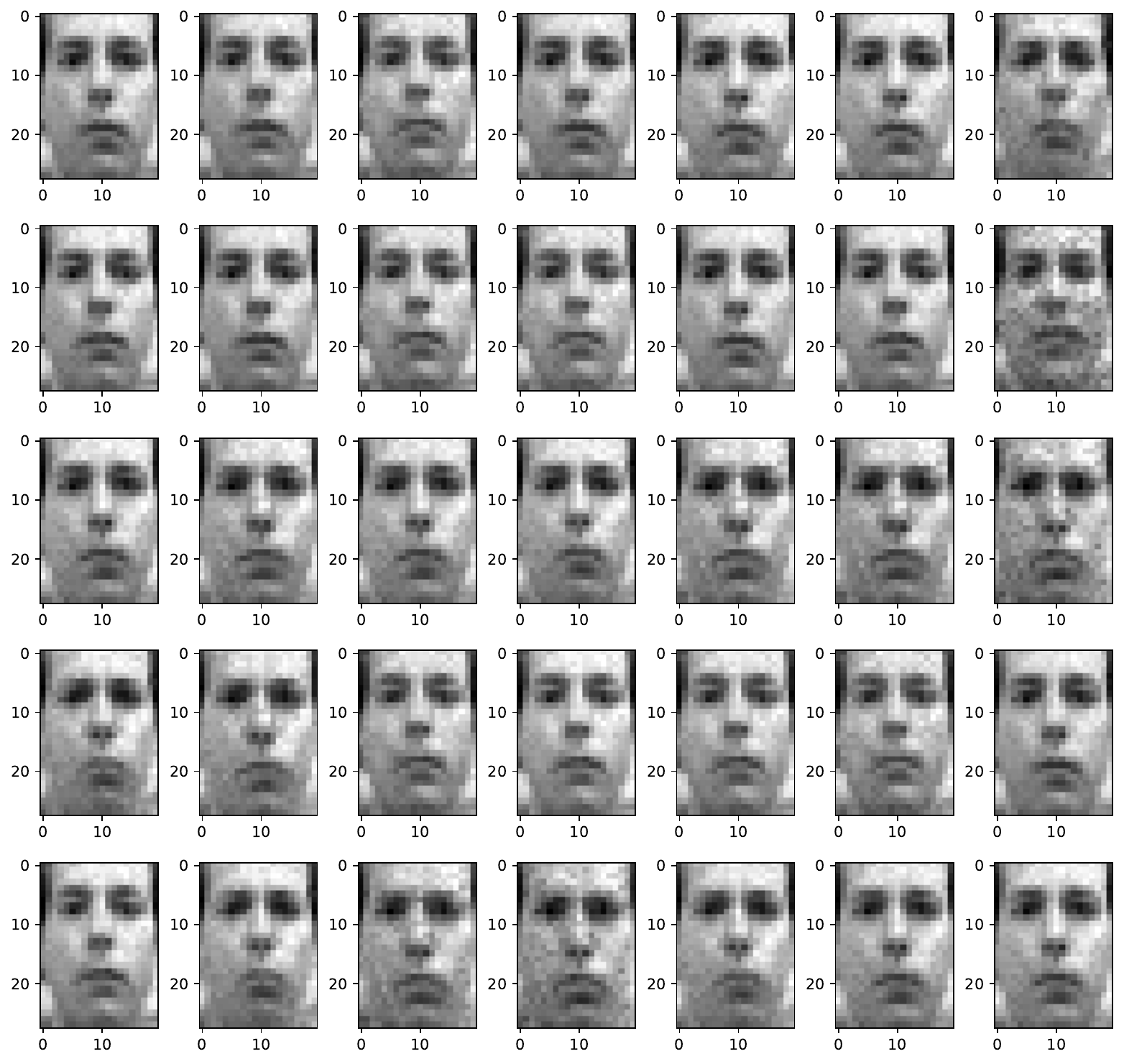}
    } \vspace{.1in}

    \subfloat[Brendan faces reconstruction task with 30\% missing pixels.]{
    \includegraphics[width=.3\linewidth]{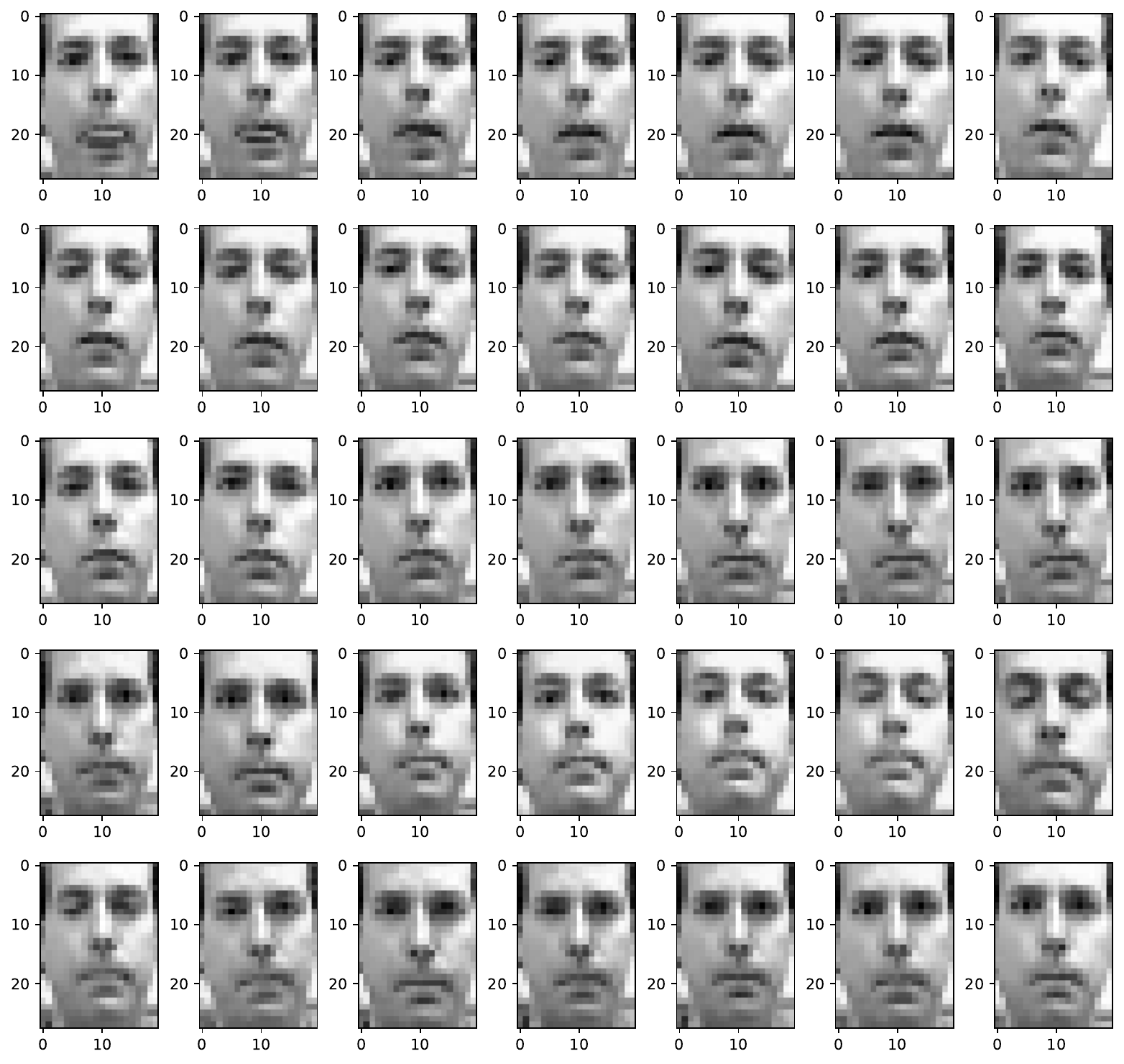} \ \ \
    \includegraphics[width=.3\linewidth]{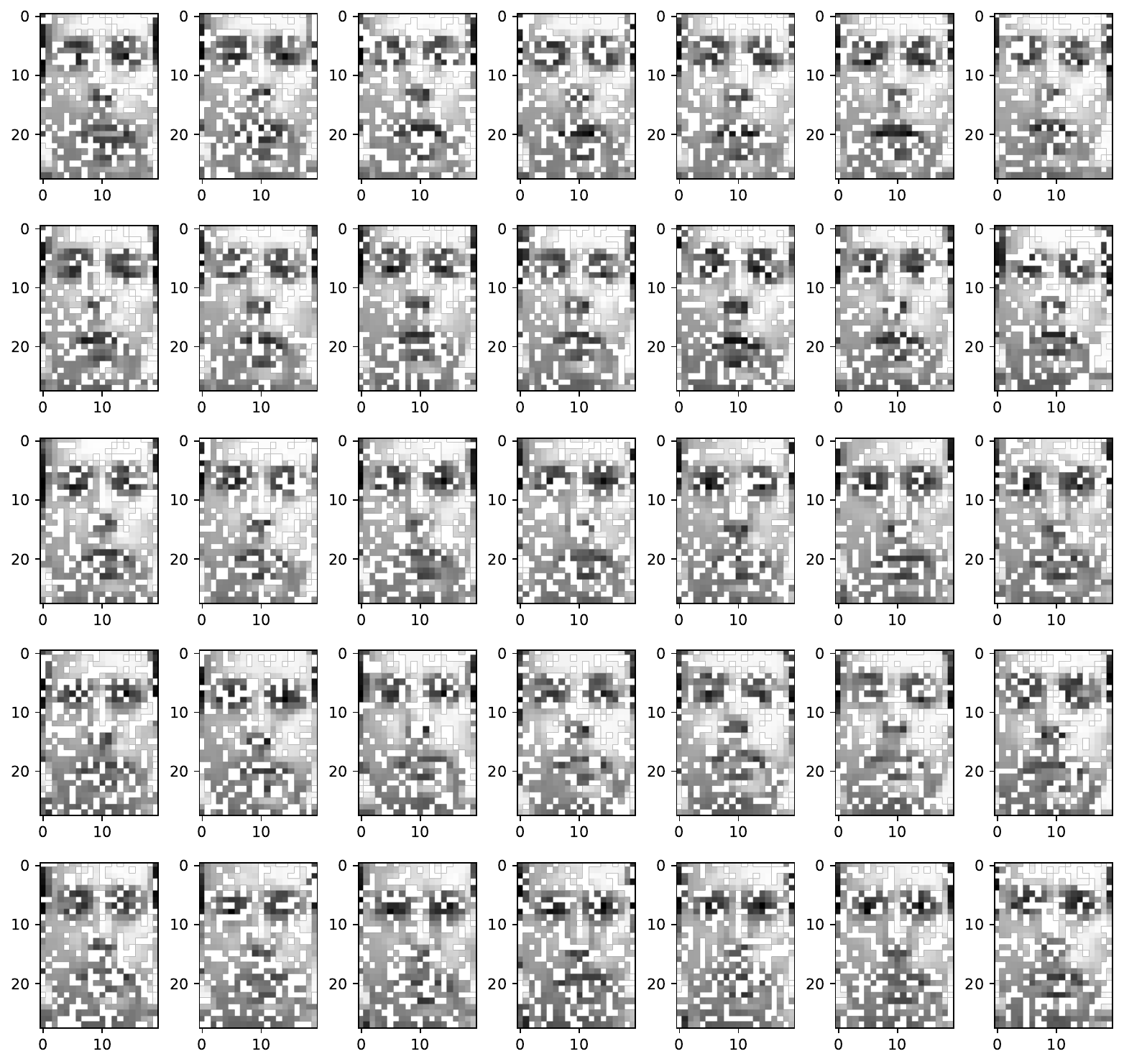} \ \ \
    \includegraphics[width=.3\linewidth]{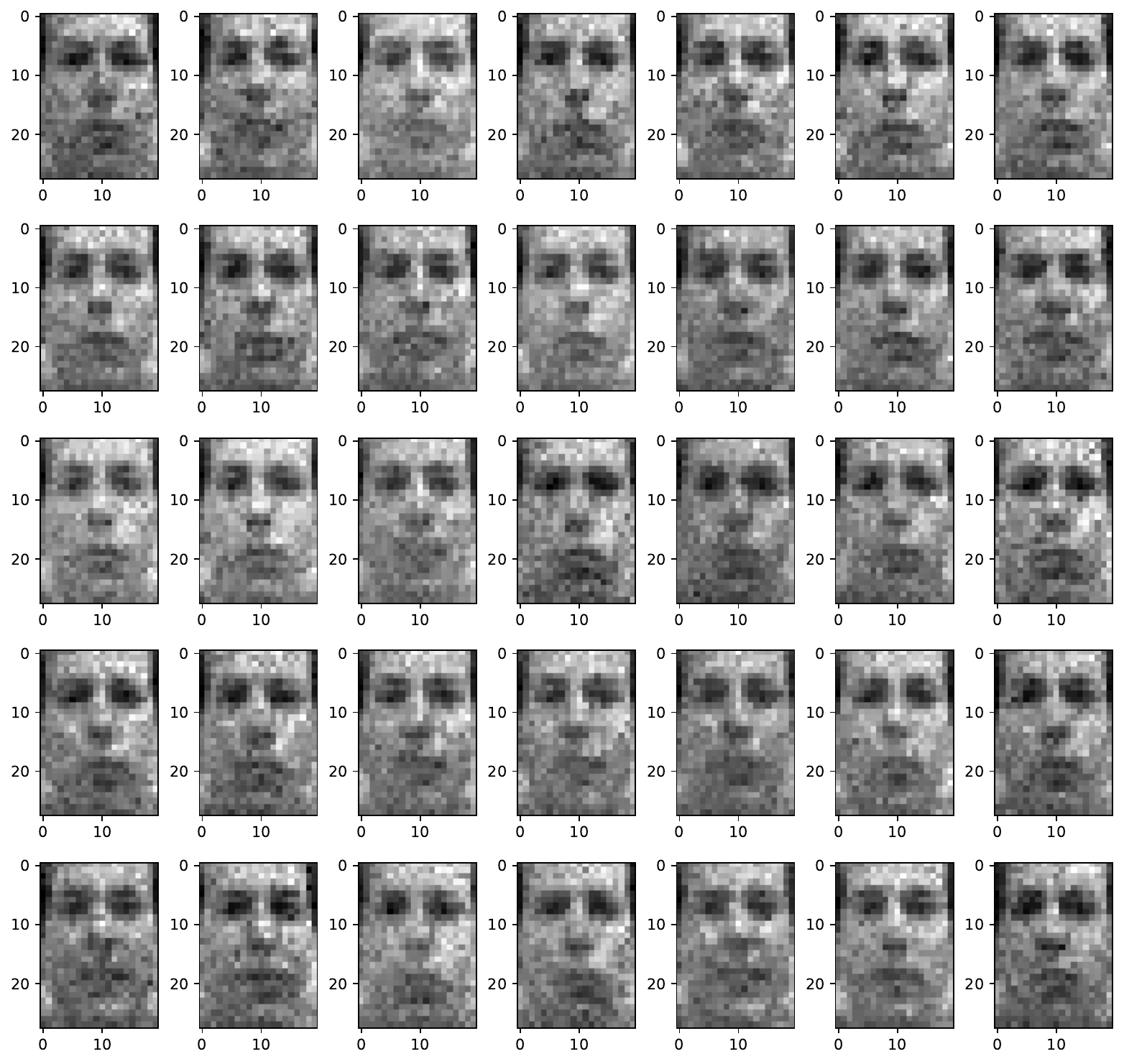}
    } \vspace{.1in}

    \subfloat[Brendan faces reconstruction task with 60\% missing pixels.]{
    \includegraphics[width=.3\linewidth]{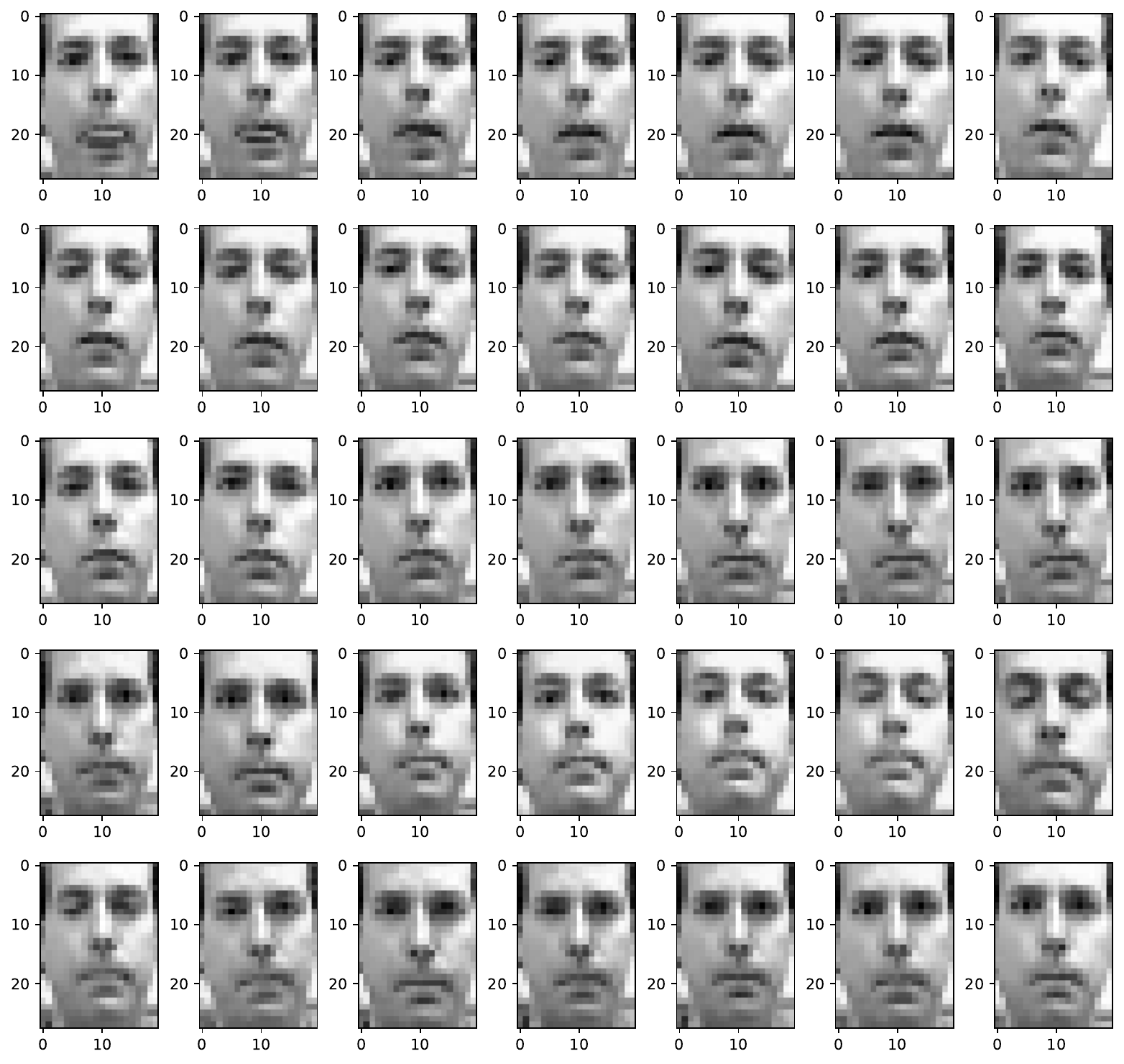} \ \ \
    \includegraphics[width=.3\linewidth]{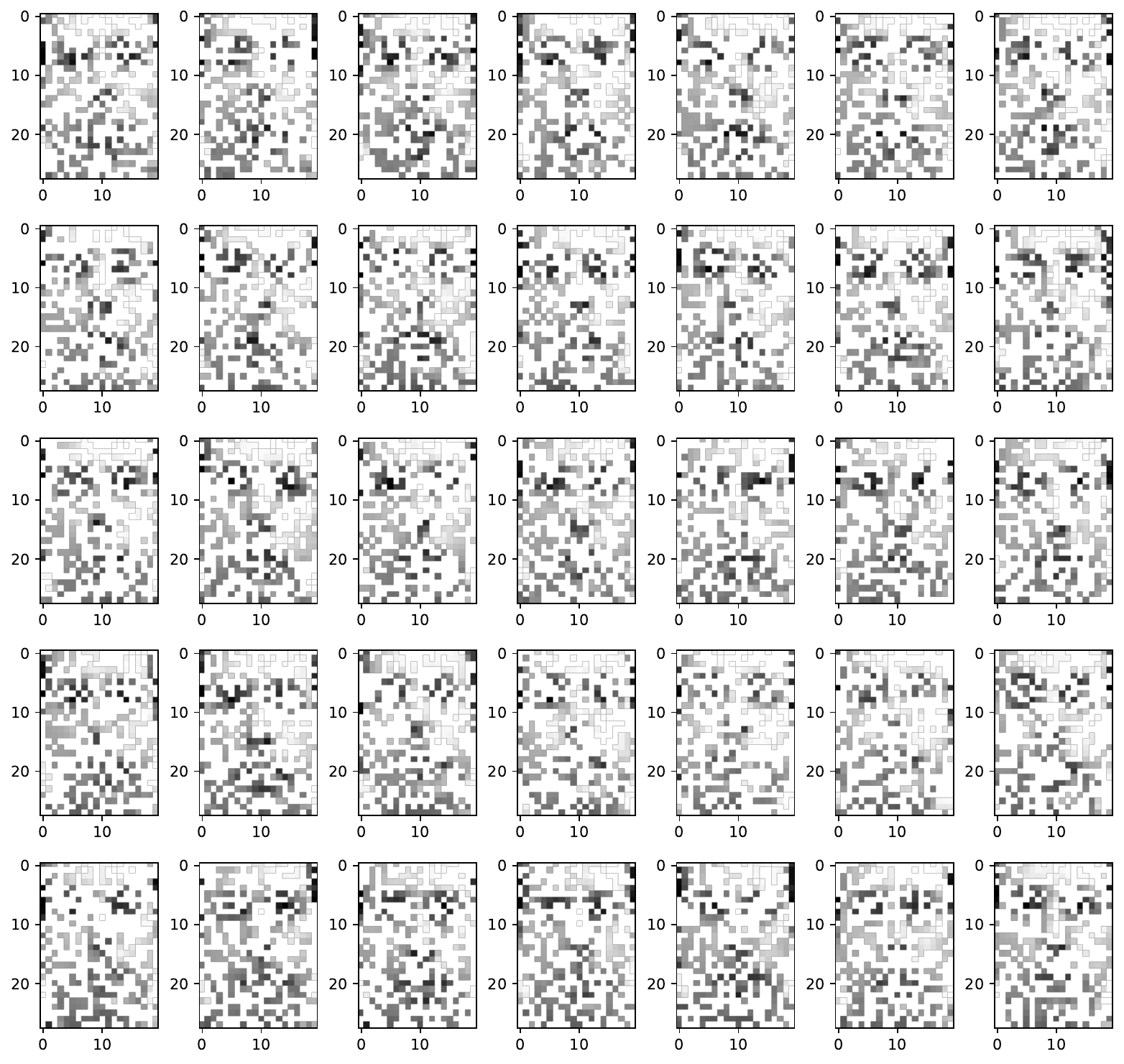} \ \ \
    \includegraphics[width=.3\linewidth]{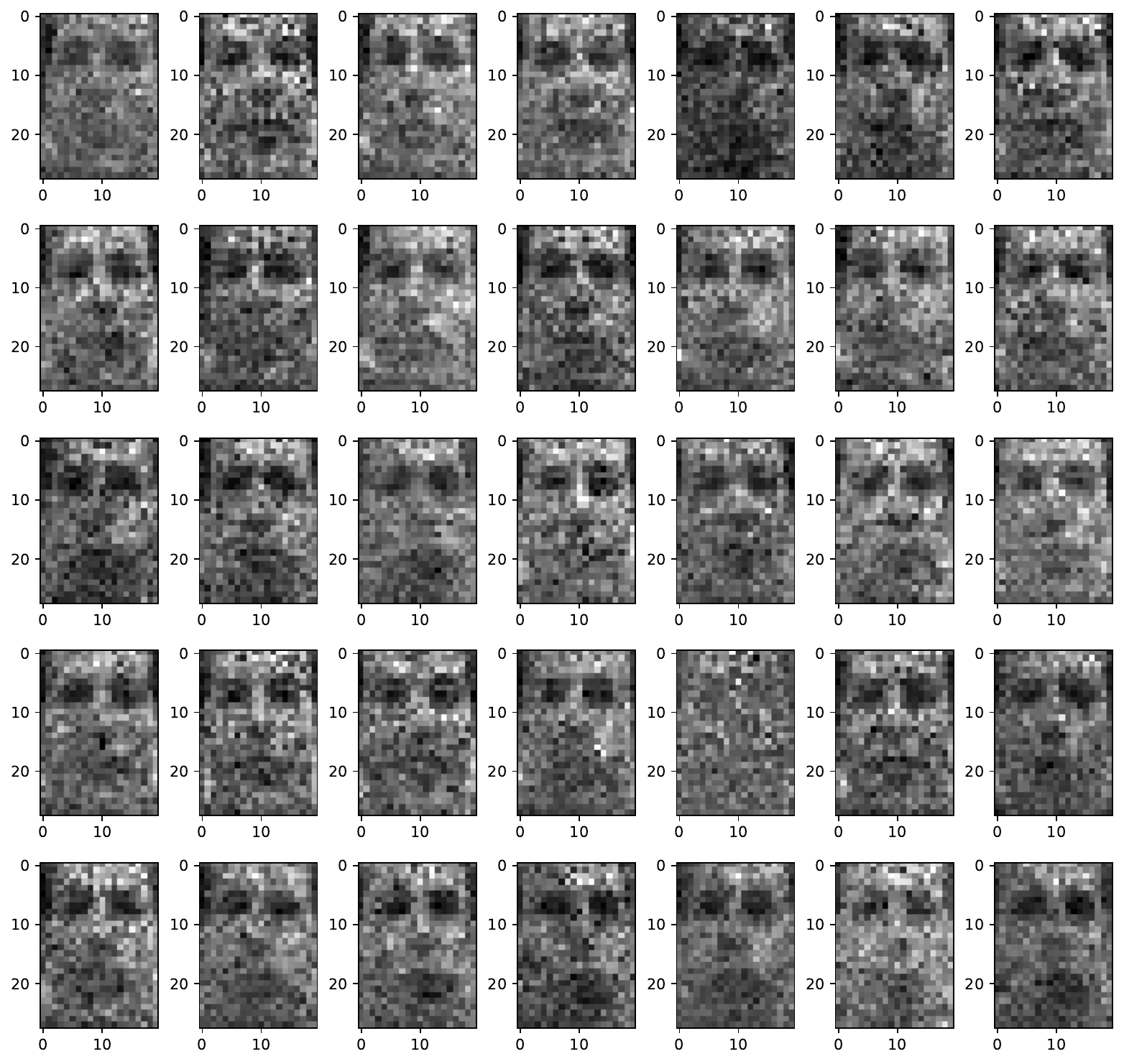}
    }  \vspace{-.1in}
    \caption{Brendan faces reconstruction task with missing pixels. From left to right: Ground truth, training images, reconstructions}
    \label{appx_fig:bface_missing_illustration}
\end{figure}

\clearpage
\subsubsection{KNN Classification Accuracy with Varying \texorpdfstring{$K$}{K}} 
\label{app:multi-K}

\begin{table}[t!]
\centering
\caption{KNN classification accuracy using different numbers of nearest neighbors ($K$ values). We ran this classification using 5-fold cross validation.} \label{tab:multi-K}
\scalebox{.63}{
\begin{tabular}{@{}l || llllllllll | llllllllll@{}}
\toprule
\multicolumn{1}{c||}{\textsc{methods}}                   & \multicolumn{10}{c|}{VAE}                                 & \multicolumn{10}{c}{\textsl{advised}RFLVM }                 \\ 
\multicolumn{1}{c||}{$K$-\textsc{value}}   & \multicolumn{1}{c}{1} & \multicolumn{1}{c}{2} & \multicolumn{1}{c}{3} & \multicolumn{1}{c}{4} & \multicolumn{1}{c}{5} & \multicolumn{1}{c}{6} & \multicolumn{1}{c}{7} & \multicolumn{1}{c}{8} & \multicolumn{1}{c}{9} & \multicolumn{1}{c|}{10} & \multicolumn{1}{c}{1} & \multicolumn{1}{c}{2} & \multicolumn{1}{c}{3} & \multicolumn{1}{c}{4} & \multicolumn{1}{c}{5} & \multicolumn{1}{c}{6} & \multicolumn{1}{c}{7} & \multicolumn{1}{c}{8} & \multicolumn{1}{c}{9} & \multicolumn{1}{c}{10} \\ \midrule \midrule  \rowcolor[HTML]{f2f2f2} 
\multicolumn{1}{c||}{\textsc{Bridges}} & 0.780 & 0.794  & 0.766  & 0.789  & 0.794  & 0.799  &  0.804  & 0.808  & 0.780  & 0.776  & 0.846                          & 0.846   & \textbf{0.902}  &  \textbf{0.902}  &  \textbf{0.907}  & \textbf{0.888}  & \textbf{0.893} & \textbf{0.898} & 0.879   & \textbf{0.903}                        \\   
  \multicolumn{1}{c||}{\textsc{Cifar}} & 0.256   &  0.260 & 0.266  & 0.274  & 0.280  & 0.282 & 0.291  & 0.296  & 0.293 & 0.300   &  \textbf{0.300} 
                                & \textbf{0.310}   & \textbf{0.309}  & \textbf{0.340}  & \textbf{0.335}  &  \textbf{0.342} & \textbf{0.350}  & \textbf{0.357}  & \textbf{0.365} & \textbf{0.358}                        \\ 
\rowcolor[HTML]{f2f2f2}  \multicolumn{1}{c||}{\textsc{Mnist}}      
& 0.631 &  0.614  & 0.657  & 0.646  & 0.677  & 0.670  & 0.674  & 0.671  & 0.673   & 0.669   & \textbf{0.801} 
&  \textbf{0.780}  & \textbf{0.819}  & \textbf{0.824}  & \textbf{0.823}  & \textbf{0.813}  & \textbf{0.812}  & \textbf{0.800}  & \textbf{0.802}  & \textbf{0.800}  \\ 
\multicolumn{1}{c||}{\textsc{montreal}}  & 0.649
& 0.655  & 0.683  & 0.662  & 0.699  & 0.705  & 0.718  & 0.718  & 0.696  & 0.712 & \textbf{0.799}                        
& 0.759  & 0.796  & 0.802  & \textbf{0.815}   &  0.787  &  0.777 & 0.755  & 0.768 & 0.759 
\\
\rowcolor[HTML]{f2f2f2} \multicolumn{1}{c||}{\textsc{yale}}  & 0.667   
&  0.667  &  0.672 & 0.667  & 0.636   & 0.630  & 0.600  & 0.576  & 0.558  & 0.552 & \textbf{0.757}   
& \textbf{0.703}  & \textbf{0.721}  & \textbf{0.745}  & \textbf{0.727}  & \textbf{0.691}  & \textbf{0.685}  & \textbf{0.673}  & \textbf{0.655}  & \textbf{0.642} 
\\ 
\multicolumn{1}{c||}{\textsc{newsgroups}} & 0.381
& 0.389  &  0.384 & 0.397  & 0.402  & 0.409  & 0.406  & 0.410  & 0.399 & 0.404    & 0.401                    
& \textbf{0.403}  & 0.399  & 0.412  & \textbf{0.419}  & 0.408  & 0.414  & 0.414  & \textbf{0.426} & 0.424                       \\ \midrule 
\multicolumn{1}{c||}{\textsc{methods}}    & \multicolumn{10}{c|}{BGPLVM}                              & \multicolumn{10}{c}{RFLVM}                           \\  
\multicolumn{1}{c||}{$K$-\textsc{value}}  & \multicolumn{1}{c}{1} & \multicolumn{1}{c}{2} & \multicolumn{1}{c}{3} & \multicolumn{1}{c}{4} & \multicolumn{1}{c}{5} & \multicolumn{1}{c}{6} & \multicolumn{1}{c}{7} & \multicolumn{1}{c}{8} & \multicolumn{1}{c}{9} & \multicolumn{1}{c|}{10} & \multicolumn{1}{c}{1} & \multicolumn{1}{c}{2} & \multicolumn{1}{c}{3} & \multicolumn{1}{c}{4} & \multicolumn{1}{c}{5} & \multicolumn{1}{c}{6} & \multicolumn{1}{c}{7} & \multicolumn{1}{c}{8} & \multicolumn{1}{c}{9} & \multicolumn{1}{c}{10}                 \\ \midrule\midrule
\rowcolor[HTML]{f2f2f2}  \multicolumn{1}{c||}{\textsc{bridges}} & 0.836   & 0.808  & 0.813  & 0.794  & 0.818  &  0.808 & 0.837  & 0.832  & 0.832  & 0.813 & \textbf{0.860}   &  \textbf{0.859}  & 0.869  & 0.892  & 0.869  & 0.869  & 0.883 & 0.888  & \textbf{0.883}  & 0.887                        \\ 
\multicolumn{1}{c||}{\textsc{cifar}}  & 0.262   & 0.278   & 0.279  & 0.293  & 0.291  & 0.295  & 0.282  & 0.290  & 0.288  & 0.294     & 0.270                  &  0.288  &  0.288  & 0.288  & 0.306  & 0.310  & 0.320  & 0.326  & 0.331  & 0.333                        \\ 
\rowcolor[HTML]{f2f2f2}  \multicolumn{1}{c||}{\textsc{mnist}}   & 0.573     & 0.585  & 0.611  & 0.622  & 0.627  &  0.636 & 0.640  & 0.645  & 0.652   & 0.648           & 0.592  &  0.567  & 0.591  & 0.603  & 0.634  & 0.624  & 0.638  & 0.634  & 0.633  & 0.633                        
\\ 
\multicolumn{1}{c||}{\textsc{montreal}}   & 0.752
&  0.771  & 0.759  &  0.771  & 0.787  & 0.778  & 0.800  & 0.800  & 0.793  & 0.787                       & 0.778
&  \textbf{0.818}  & \textbf{0.819}  & \textbf{0.844}  & 0.809  & \textbf{0.831}  & \textbf{0.806}  & \textbf{0.815}  & \textbf{0.806}  & \textbf{0.790}                        \\ 
\rowcolor[HTML]{f2f2f2}  \multicolumn{1}{c||}{\textsc{yale}} & 0.558       & 0.515  & 0.545  & 0.558  & 0.527  & 0.503  & 0.503  & 0.467  &  0.485 & 0.461                       & 0.576 &  0.515  & 0.612  & 0.564  & 0.564  & 0.576  & 0.576  & 0.558  & 0.582  &  0.576                        \\
\multicolumn{1}{c||}{\textsc{newsgroups}}       & 0.388               &  0.374  & 0.406  & 0.397   & 0.392 & 0.395   & 0.404  & 0.397   & 0.396  & 0.403               & \textbf{0.404}             
&  0.394   & \textbf{0.411}   & \textbf{0.425}  & 0.412  & \textbf{0.413}  & \textbf{0.417}  & \textbf{0.424}  & 0.416 & \textbf{0.426}                        \\ 
\bottomrule
\end{tabular}}
\end{table}

\begin{table}[t!]
\centering
\caption{KNN classification accuracy using different numbers of nearest neighbors ($K$ values) on larger datasets. We ran this classification using 5-fold cross validation.} \label{tab:large_datasets}
\scalebox{.65}{
\begin{tabular}{@{}l || llllllllll | llllllllll@{}}
\toprule
\multicolumn{1}{c||}{\textsc{methods}}                   & \multicolumn{10}{c|}{VAE}                                 & \multicolumn{10}{c}{\textsl{advised}RFLVM }                 \\ 
\multicolumn{1}{c||}{$K$-\textsc{value}}   & \multicolumn{1}{c}{1} & \multicolumn{1}{c}{2} & \multicolumn{1}{c}{3} & \multicolumn{1}{c}{4} & \multicolumn{1}{c}{5} & \multicolumn{1}{c}{6} & \multicolumn{1}{c}{7} & \multicolumn{1}{c}{8} & \multicolumn{1}{c}{9} & \multicolumn{1}{c|}{10} & \multicolumn{1}{c}{1} & \multicolumn{1}{c}{2} & \multicolumn{1}{c}{3} & \multicolumn{1}{c}{4} & \multicolumn{1}{c}{5} & \multicolumn{1}{c}{6} & \multicolumn{1}{c}{7} & \multicolumn{1}{c}{8} & \multicolumn{1}{c}{9} & \multicolumn{1}{c}{10} 
\\ \midrule\midrule 
 \rowcolor[HTML]{f2f2f2} 
 \multicolumn{1}{c||}{\textsc{f-Cifar}}   & 0.157 &	0.151 &	0.157 &	0.166	& 0.174 &	0.178 &	0.180 &	0.187 &	0.188 &	0.190 &     \textbf{0.172} & 	\textbf{0.161}	& \textbf{0.177} &	\textbf{0.181} &	\textbf{0.194}	& \textbf{0.199 }&	\textbf{0.203} &	\textbf{0.209}	& \textbf{0.213} &	\textbf{0.214}     
  \\ 
 \multicolumn{1}{c||}{\textsc{fd-Cifar}}     &  0.263	& 0.266 &	0.279 &	0.285 &	0.293	& 0.297	& 0.302	& 0.304	& 0.309 & 	0.312 &    \textbf{0.321}	& \textbf{0.323} &	\textbf{0.344} & 	\textbf{0.359} &	\textbf{0.368} &	\textbf{0.370}	& \textbf{0.377} &\textbf{0.384} &	\textbf{0.390} &	\textbf{0.391}      
  \\ 
 \rowcolor[HTML]{f2f2f2}  \multicolumn{1}{c||}{\textsc{f-Mnist}}   & 0.728  & 0.728  & 0.756  & 0.766  & 0.774  & 0.775  & 0.778  & 0.782   & 0.782  &  0.783  &  \textbf{0.794} & \textbf{0.796} & \textbf{0.831} & \textbf{0.838} & \textbf{0.845} & \textbf{0.847} & \textbf{0.850} & \textbf{0.852} & \textbf{0.851} & \textbf{0.852}             
\\ \midrule
\multicolumn{1}{c||}{\textsc{methods}}    & \multicolumn{10}{c|}{BGPLVM}                              & \multicolumn{10}{c}{Isomap}                               
\\  
\multicolumn{1}{c||}{$K$-\textsc{value}}   & \multicolumn{1}{c}{1} & \multicolumn{1}{c}{2} & \multicolumn{1}{c}{3} & \multicolumn{1}{c}{4} & \multicolumn{1}{c}{5} & \multicolumn{1}{c}{6} & \multicolumn{1}{c}{7} & \multicolumn{1}{c}{8} & \multicolumn{1}{c}{9} & \multicolumn{1}{c|}{10} & \multicolumn{1}{c}{1} & \multicolumn{1}{c}{2} & \multicolumn{1}{c}{3} & \multicolumn{1}{c}{4} & \multicolumn{1}{c}{5} & \multicolumn{1}{c}{6} & \multicolumn{1}{c}{7} & \multicolumn{1}{c}{8} & \multicolumn{1}{c}{9} & \multicolumn{1}{c}{10}                     
\\ \midrule\midrule
\rowcolor[HTML]{f2f2f2} \multicolumn{1}{c||}{\textsc{f-cifar}}     
 & 0.138 & 0.132 & 0.140 &	0.145 &	0.154 &	0.156 &	0.159 &	0.161 &	0.162	& 0.163 &  0.144 &	0.142 &	0.147 &	0.157 	& 0.159 &	0.163 & 	0.165 &	0.168	& 0.170 &	0.173  
 \\ 
 \multicolumn{1}{c||}{\textsc{fd-Cifar}}     
&  0.250	& 0.260	& 0.258 & 	0.264 &	0.279 &	0.277 &	0.281	& 0.285	& 0.287 &	0.287 &  0.264 &	0.270 &	0.279 &	0.287 &	0.291 &	0.292 &	0.296	& 0.301 &	0.305	& 0.305 
  \\ 
 \rowcolor[HTML]{f2f2f2}  \multicolumn{1}{c||}{\textsc{f-Mnist}}     
& 0.414 & 0.420	& 0.433 &	0.449 &	0.455 &	0.464	& 0.466 &	0.470	& 0.473 &	0.474 &  0.456 &	0.468 &	0.493	& 0.504 &	0.514 &	0.524 &	0.529	& 0.534	& 0.535 &	0.540
\\ \midrule 
\multicolumn{1}{c||}{\textsc{methods}}    & \multicolumn{10}{c|}{NBVAE}                              & \multicolumn{10}{c}{CVQ-VAE}                                
\\  
\multicolumn{1}{c||}{$K$-\textsc{value}}   & \multicolumn{1}{c}{1} & \multicolumn{1}{c}{2} & \multicolumn{1}{c}{3} & \multicolumn{1}{c}{4} & \multicolumn{1}{c}{5} & \multicolumn{1}{c}{6} & \multicolumn{1}{c}{7} & \multicolumn{1}{c}{8} & \multicolumn{1}{c}{9} & \multicolumn{1}{c|}{10} & \multicolumn{1}{c}{1} & \multicolumn{1}{c}{2} & \multicolumn{1}{c}{3} & \multicolumn{1}{c}{4} & \multicolumn{1}{c}{5} & \multicolumn{1}{c}{6} & \multicolumn{1}{c}{7} & \multicolumn{1}{c}{8} & \multicolumn{1}{c}{9} & \multicolumn{1}{c}{10}                     
\\ \midrule\midrule
\rowcolor[HTML]{f2f2f2} \multicolumn{1}{c||}{\textsc{f-cifar}}                        & 0.134 & 0.137 & 0.140 & 0.147 & 0.152 & 0.157 & 0.156  & 0.155 & 0.161 & 0.162 &   0.101   &   0.098   &   0.099   &   0.101   &     0.102    &   0.102   &   0.101   &   0.100   &   0.098   &   0.096   
\\  
\multicolumn{1}{c||}{\textsc{fd-Cifar}}    & 0.252 & 0.248 & 0.255 & 0.264 & 0.273 & 0.277 & 0.282  & 0.287 & 0.291 & 0.292 &   0.203   &   0.201   &   0.200   &   0.199   &     0.201    &   0.200   &   0.201   &   0.199   &   0.197   &   0.200                
  \\ 
 \rowcolor[HTML]{f2f2f2}  \multicolumn{1}{c||}{\textsc{f-Mnist}}         & 0.502 & 0.502 & 0.533 & 0.548 & 0.557 & 0.566 & 0.571  & 0.577 & 0.579 & 0.582 &   0.104   &   0.107   &   0.107   &   0.104   &     0.105    &   0.106   &   0.102   &   0.103   &   0.103   &   0.103 
 \\
\bottomrule
\end{tabular}
}
\end{table}

We have presented the KNN results with ten different choices of $K$ in Tab.~\ref{tab:multi-K}, wherein the setting of $K=1$ aligns with the configuration employed in \citep{gundersen2021latent}.  The simulation results consistently demonstrate the superiority of our method over the benchmarks regardless of the $K$ values across most datasets. In those exception cases, \adrflvm still achieves very comparable performance with RFLVM on some relatively simple datasets, e.g., \textsc{Bridges, Montreal}, and \textsc{Newsgroup} datasets.

\subsubsection{Larger Datasets Extension}
\label{app:larger_datasets}

To ensure equitable evaluation of deep learning methods, such as various VAE variants, we conducted comprehensive comparisons on larger datasets, including the full \textsc{mnist} and \textsc{cifar}  datasets. The results are summarized in Table \ref{tab:large_datasets},  where \textsc{f-cifar} and \textsc{f-mnist} represent the full \textsc{cifar} and \textsc{mnist} datasets, respectively, and \textsc{fd-cifar} denotes the full \textsc{cifar} dataset with each image downsampled to $20 \times 20$ pixels. Our empirical results demonstrate significant performance improvement for both VAE and our \adrflvm when applied to larger datasets. Notably, \adrflvm consistently outperforms the other benchmarks across datasets of varying sizes, highlighting its superiority over state-of-the-art variants irrespective of the dataset size.





\end{document}